\documentclass{article} % For LaTeX2e
\usepackage{iclr2023_conference,times}

\usepackage{amsmath,amsfonts,bm}

\def\Secref#1{Section~\ref{#1}}
\def\eqref#1{equation~\ref{#1}}
\def\eq#1{~(\ref{#1})}
\def\Appref#1{Appendix~\ref{#1}}
\def\Asmpref#1{Assumption~\ref{#1}}

\def\1{\bm{1}}

\def\rb{{\textnormal{b}}}

\def\vb{{\bm{b}}}

\DeclareMathAlphabet{\mathsfit}{\encodingdefault}{\sfdefault}{m}{sl}
\SetMathAlphabet{\mathsfit}{bold}{\encodingdefault}{\sfdefault}{bx}{n}

\def\gO{{\mathcal{O}}}

\def\gS{{\mathcal{S}}}

\def\gZ{{\mathcal{Z}}}

\def\sR{{\mathbb{R}}}

\newcommand{\E}{\mathbb{E}}

\newcommand{\KL}{D_{\mathrm{KL}}}

\let\ab\allowbreak

\newcommand{\eg}{\emph{e.g.}}
\newcommand{\ie}{\emph{i.e.}}

\newcommand{\inner}[2]{\left\langle #1,#2 \right\rangle}
\newcommand{\rbr}[1]{\left(#1\right)}
\newcommand{\sbr}[1]{\left[#1\right]}

\newcommand{\nbr}[1]{\left\|#1\right\|}
\newcommand{\abr}[1]{\left|#1\right|}

\usepackage{amsfonts, amsmath, amssymb, amsthm}

\usepackage{hyperref}
\usepackage{url}
\usepackage{microtype}
\usepackage{graphicx}
\usepackage{subfigure}
\usepackage{booktabs} % for professional tables
\usepackage{natbib}
\hypersetup{
    colorlinks=true,
    linkcolor=blue,
    filecolor=magenta,      
    urlcolor=cyan,
}

\usepackage{multirow}
\usepackage{xspace}

\usepackage{algorithm}
\usepackage{algpseudocode}

\usepackage{xcolor}
\usepackage{enumitem}

\newcommand{\Tongzheng}[1]{{\color{red} [Tongzheng: #1]}}
\newcommand{\lina}[1]{{\color{orange} [Lina: #1]}}
\newcommand{\revise}[1]{{\color{black}#1}}

\newcommand{\AlgName}{{Latent Variable Representation}\xspace}
\newcommand{\algabb}{{LV-Rep}\xspace}

\newtheorem{theorem}{Theorem}
\newtheorem{lemma}[theorem]{Lemma}

\newtheorem{definition}{Definition}
\newtheorem{assumption}{Assumption}
\theoremstyle{remark}
\newtheorem{remark}{Remark}

\title{Latent Variable Representation for Reinforcement Learning}

\author{Tongzheng Ren\textsuperscript{1, 2, \thanks{Equal Contribution. Project Website: \href{https://rlrep.github.io/lvrep/}{https://rlrep.github.io/lvrep/}}} \quad Chenjun Xiao\textsuperscript{3, \footnotemark[1]} \quad Tianjun Zhang\textsuperscript{1, 4} 
\quad Na Li\textsuperscript{1, 5} \quad Zhaoran Wang\textsuperscript{6} \\
\bf{Sujay Sanghavi\textsuperscript{2} \quad 
 Dale Schuurmans\textsuperscript{1, 3} \quad Bo Dai\textsuperscript{1, 7}}
\\
 \textsuperscript{1}Google Research, Brain Team \quad \textsuperscript{2}UT Austin \quad \textsuperscript{3}University of Alberta \quad \textsuperscript{4}UC Berkeley \\ \textsuperscript{5}Harvard University \quad \textsuperscript{6}Northwestern University \quad \textsuperscript{7}Georgia Tech\\
 \texttt{tongzheng@utexas.edu, chenjun@ualberta.ca, bodai@google.com}\\
}

\allowdisplaybreaks

\iclrfinalcopy % Uncomment for camera-ready version, but NOT for submission.
\begin{document}

\maketitle

\begin{abstract}
%Deep probabilistic latent variable models have been successfully applied in model-based reinforcement learning~(RL) due to their powerful modeling ability for complex transition dynamics. However, on the other hand, the flexibility of the latent variable models also introduce the challenges in exploration for RL, especially with rigorous guarantees. 
%In this paper, we reformulate the latent variable models as low-rank Markov decision processes~(MDPs), which naturally provides us the representation for the state-action value functions. We then implement the optimism/pessimism in the face of uncertainty principle in a computational efficient way upon the kernel embeddings of latent variable models learned with variational methods. 
%A PAC analysis in both online and offline setting has been established, justifying the statistical efficiency of the proposed exploration. In addition, an experimental investigation demonstrates the superior empirical performances over the current state-of-the-art algorithms across several benchmarks.

Deep latent variable models have achieved significant empirical successes in model-based reinforcement learning (RL) due to their expressiveness in modeling complex transition dynamics. On the other hand, it remains unclear theoretically and empirically how latent variable models may facilitate learning, planning, and exploration to improve the sample efficiency of RL. In this paper, we provide a representation view of the latent variable models for state-action value functions, which allows both tractable variational learning algorithm and effective implementation of the optimism/pessimism principle in the face of uncertainty for exploration. In particular, we propose a computationally efficient planning algorithm with UCB exploration by incorporating kernel embeddings of latent variable models. Theoretically, we establish the sample complexity of the proposed approach in the online and offline settings. Empirically, we demonstrate superior performance over current state-of-the-art algorithms across various benchmarks.
\end{abstract}

\setlength{\abovedisplayskip}{2pt}
\setlength{\abovedisplayshortskip}{2pt}
\setlength{\belowdisplayskip}{2pt}
\setlength{\belowdisplayshortskip}{2pt}
\setlength{\jot}{2pt}

\setlength{\floatsep}{2ex}
\setlength{\textfloatsep}{2ex}

%%%%%%%%%%%%%%%%%%%%%%%%%%%%%%%%%%%%%%%%%%%%%%%%%%%%%%%
\vspace{-2mm}
\section{Introduction}\label{sec:intro}
\vspace{-2mm}
%%%%%%%%%%%%%%%%%%%%%%%%%%%%%%%%%%%%%%%%%%%%%%%%%%%%%%%

% \Bo{
% \begin{itemize}
%     \item the success of model-based RL
%     \item incoherence in model-based RL: learning, planning, and exploration.  
%     \item representation view of latent variable models. 
% \end{itemize}
% }

% Deep latent variable model, also known as variational auto-encoder~\citep{kingma2013auto}, is a powerful class for modeling complicated distributions via tractable distributions over the expanded variable space parametrized with deep neural networks. 

% \Bo{RL and the properties of model-free and model-based RL}
Reinforcement learning~(RL) seeks an optimal policy that maximizes the expected accumulated rewards by interacting with an unknown environment sequentially. Most  research  in RL is based on the framework of Markov decision processes~(MDPs)~\citep{puterman2014markov}. For MDPs with finite states and actions, there is already a clear understanding with sample and computationally efficient algorithms~\citep{auer2008near, dann2015sample,osband2014model,azar2017minimax,jin2018q}. However, the cost of these RL algorithms quickly becomes unacceptable for large or infinite state problems. Therefore, function approximation or parameterization is a major tool to tackle the curse of dimensionality. Based on the parametrized component to be learned,  RL algorithms can roughly be classified into two categories: model-free and model-based RL, where the algorithms in the former class directly learn a value function or policy to maximize the cumulative rewards, while algorithms in the latter class  learn a model to mimic the environment and the optimal policy is obtained by planning with the learned simulator. 

% \Bo{ sample inefficiency/exploration difficulty in model-free RL. }
Model-free RL algorithms exploit an end-to-end learning paradigm for policy and value function training, and have achieved  empirical success in robotics~\citep{peng2018sim}, video-games~\citep{mnih2013playing}, and dialogue systems~\citep{jiang2021towards}, to name a few, thanks to  flexible deep neural network parameterizations. The flexibility of such parameterizations, however, also comes with a cost in optimization and exploration. Specifically, it is well-known that  temporal-difference methods become unstable or even divergent with general nonlinear function approximation~\citep{boyan1994generalization,tsitsiklis1996analysis}. Uncertainty quantization for general nonlinear function approximators is also underdeveloped.
Although there are several theoretically interesting model-free exploration algorithms with general nonlinear function approximators~\citep{wang2020reinforcement,kong2021online,jiang2017contextual}, a computationally-friend exploration method for model-free RL is still missing. 

% Empirically, these difficulties lead to the massive sample requirements for model-free algorithms. 

% \Bo{incoherence in model-based RL: learning, planning, and exploration. }
Model-based RL algorithms, on the other hand, exploit more information from the environment during learning, and are therefore considered to be more promising in terms of sample efficiency~\citep{wang2019benchmarking}. Equipped with powerful deep models, model-based RL can successfully reduce approximation error, and have demonstrated strong performance in practice~\citep{hafner2019dream,hafner2019learning,wu2022daydreamer}, following with some theoretical justifications~\citep{osband2014model,foster2021statistical}. However, the reduction of approximation error brings new challenges in planning and exploration, which have not been treated seriously from the empirical and theoretical aspects. Specifically, with general nonlinear models, the planning problem itself is already no longer tractable, and the problem becomes more difficult with an exploration mechanism introduced. While  theoretical analysis typically assumes a planning oracle providing an optimal policy, some approximations are necessary in practice, including dyna-style planning~\citep{chua2018deep,luo2018algorithmic}, random shooting~\citep{kurutach2018model,hafner2019dream}, and policy search with backpropagation through time~\citep{deisenroth2011pilco,heess2015learning}. These may lead to sub-optimal policies, even with perfect models, wasting potential modeling power. 

% \Bo{emphasize theoretically solid and empirical tractable with function approximation.}
In sum, for both model-free and model-based algorithms, there has been insufficient work considering both statistical and computation tractability and efficiency in terms of learning, planning and exploration in a unified and coherent perspective for algorithm design. This raises the question:
% \vspace{-1mm}
\begin{center}
    \emph{Is there a way to design a {\bf provable} and {\bf practical} algorithm to remedy both the statistical and computational difficulties of RL?}
\end{center}
% \vspace{-1mm}
Here, by ``provable'' we mean the statistical complexity of the algorithm can be rigorously characterized without explicit dependence on the number of states but instead the fundamental complexity of the parameterized representation space; while by ``practical'' we mean the learning, planning and exploration components in the algorithm are computationally tractable and can be implemented in real-world scenarios. 

% \Bo{ representation view of latent variable models. contribution. }
This work provides an affirmative answer to the question above by establishing the representation view of  latent variable dynamics models through a connection to linear MDPs. Such a connection immediately provides a computationally tractable approach to planning and exploration in the \emph{linear space} constructed by the flexible deep latent variable model. %while on the other hand, the 
Such a latent variable model view also provides a \emph{variational} learning method that remedies the intractbility of MLE for general linear MDPs~\citep{agarwal2020flambe,uehara2021representation}. 
%Our contribution lies in several folds:
Our main contributions consist of the following:
% \vspace{-2mm}
\begin{itemize}[leftmargin=20pt, parsep=0pt, partopsep=0pt]
    \item We establish the representation view of latent variable dynamics models in RL, which naturally induces ~\emph{\AlgName~(\algabb)} for linearly representing the state-action value function, and paves the way for a practical variational method for representation learning (\Secref{sec:method});
    \item We provide computation efficient algorithms to implement the principle of optimistm and pessimism in the face of uncertainty with the learned~\algabb for online and offline RL (\Secref{sec:rl_rep}); 
    \item We theoretically analyze the sample complexity of~\algabb in both online and offline settings, which reveals the essential complexity beyond the cardinality of the latent variable (\Secref{sec:analysis}); 
    \item We empirically demonstrate~\algabb outperforms the state-of-the-art model-based and model-free RL algorithms on several RL benchmarks~(\Secref{sec:exp})
\end{itemize}
%%%%%%%%%%%%%%%%%%%%%%%%%%%%%%%%%%%%%%%%%%%%%%%%%%%%%%%
\vspace{-2mm}
\section{Preliminaries}\label{sec:prelim}
\vspace{-2mm}
%%%%%%%%%%%%%%%%%%%%%%%%%%%%%%%%%%%%%%%%%%%%%%%%%%%%%%%

% \Bo{
% \begin{itemize}
%     \item linear/low-rank MDP: 1) emphasize the difficulty of the learning of linear MDPs; 2), also discuss the representation ability from~\citet{agarwal2020flambe}. 
%     \item RKHS embeddings
% \end{itemize}
% }

% \Tongzheng{Make some connections in the following sections.}

In this section, we provide brief introduction to MDPs and linear MDP, which play important roles in the algorithm design and theoretical analysis. We also provide the required background knowledge on functional analysis in Appendix~\ref{appendix:tech_back}. 

\vspace{-2mm}
\subsection{Markov Decision Processes}
We consider the infinite horizon discounted Markov decision process (MDP) specified by the tuple $\mathcal{M} = \langle \mathcal{S}, \mathcal{A}, T^*, r, \gamma, d_0 \rangle$, where $\mathcal{S}$ is the state space, $\mathcal{A}$ is a discrete action space, $T^*:\mathcal{S} \times \mathcal{A} \to \Delta(\mathcal{S})$ is the transition, $r:\mathcal{S} \times \mathcal{A} \to [0, 1]$ is the reward, $\gamma \in (0, 1)$ is the discount factor and $d_0 \in \Delta(\mathcal{S})$ is the initial state distribution. Following the standard convention \citep[e.g.][]{jin2020provably}, we assume $r(s, a)$ and $d_0$ are known to the agent. We aim to find the policy $\pi:\mathcal{S} \to \Delta(\mathcal{A})$, that maximizes the following discounted cumulative reward:
\begin{align*}
\textstyle
    V_{T^*, r}^{\pi} := \mathbb{E}_{T^*, \pi}\left[\sum_{i=0}^{\infty} \gamma^i r(s_i, a_i)\bigg|s_0 \sim d_0\right]. 
\end{align*}
We define the state value function $V:\mathcal{S} \to \left[0, \frac{1}{1-\gamma}\right]$ and state-action value function $Q:\mathcal{S}\times \mathcal{A} \to \left[0, \frac{1}{1-\gamma}\right]$ following the standard notation:
\begin{align*}
\textstyle
    % V_{P^*, r}^\pi(s) =  \mathbb{E}_{P^*, \pi}\left[\sum_{i=0}^{\infty} \gamma^i r(s_i, a_i)\bigg|s_0 = s\right],\quad 
    Q_{T^*, r}^\pi(s, a) =  \mathbb{E}_{T^*, \pi} \left[\sum_{i=0}^{\infty} \gamma^i r(s_i, a_i)\bigg|s_0 = s, a_0 = a\right], \quad  
    V_{T^*, r}^\pi(s) = \mathbb{E}_{a\sim\pi(\cdot|s)}\left[Q_{T^*, r}^\pi(s, a)\right],
\end{align*}
It is straightforward to see that $V_{T^*, r}^\pi = \mathbb{E}_{s \sim d_0} \left[V_{T^*, r}^\pi(s)\right]$, as well as the following Bellman equation:
\begin{align*}
       Q_{T^*, r}^\pi(s, a) =  r(s, a) + \gamma\mathbb{E}_{s^\prime\sim T^*(\cdot|s, a)}\left[V_{T^*, r}^\pi(s^\prime)\right].
\end{align*}
We also define the discounted occupancy measure $d_{T^*}^\pi$ of policy $\pi$ as follows:
\begin{align*}
\textstyle
    % d_{P^*}^\pi(s) = & \mathbb{E}_{P^*, \pi}\left[\sum_{i=0}^{\infty} \gamma^i \mathbf{1}_{s_i = s}\bigg|s_0 \sim d_0\right],\\
    d_{T^*}^\pi(s, a) =  \mathbb{E}_{T^*, \pi}\left[\sum_{i=0}^{\infty} \gamma^i \mathbf{1}_{s_i = s, a_i = a}\bigg|s_0 \sim d_0\right].
\end{align*}
By the definition of the discounted occupancy measure, we can see $V_{T^*, r}^\pi = \mathbb{E}_{(s, a) \sim d_{T^*}^\pi}\left[r(s, a)\right]$. Furthermore, with the property of the Markov chain, we can obtain
\begin{align*}
    d_{T^*}^\pi(s, a) = (1-\gamma) d_0 \cdot \pi(a|s) + \gamma \mathbb{E}_{(\widetilde{s}, \widetilde{a}) \sim d_{T^*}^\pi(s, a)} \left[T^*(s|\widetilde{s}, \widetilde{a}) \times \pi(a|s)\right].
\end{align*}

\vspace{-2mm}
\subsection{Linear MDP}
In the tabular MDP, where the state space $|\mathcal{S}|$ is finite, there exist lots of work on sample- and computation-efficient RL algorithms \citep[e.g.][]{azar2017minimax, jin2018q}. However, such methods can still be expensive when $|\mathcal{S}|$ becomes large or even infinite, which is quite common for in real-world applications.
To address this issue, we would like to introduce function approximations into RL algorithms to alleviate the statistical and computational bottleneck. The linear MDP~\citep{jin2020provably,agarwal2020flambe} is a promising subclass admits special structure for such purposes.
% One of the most commonly studied MDP classes that allows for sample-efficient and computation-efficient algorithms is the low-rank MDP, which is defined as the follows:
\begin{definition}[Linear MDP~\citep{jin2020provably,agarwal2020flambe}]\label{def:low-rank-MDP}
    An MDP is called a linear MDP if there exists $\phi^*:\mathcal{S}\times \mathcal{A} \to \mathcal{H}$ and $\mu^*:\mathcal{S} \to \mathcal{H}$ for some proper Hilbert space $\mathcal{H}$, such that $T^*(s^\prime|s, a) = \left\langle \phi^*(s, a), \mu^*(s^\prime)\right\rangle_{\mathcal{H}}$.
\end{definition}
% \lina{motivated by a comment I put on page 4 regarding ``low rank'', I would prefer to use a different name rather than ``low-rank'' because without additional structure/assumptions, the ``linear factorization'' does not imply ``low rank'', plus it is actually unclear what kind of ``rank'' we are talking about here.}\Bo{This is a widely accepted model defined by previous literature. }
The complete definition of linear MDPs require $\phi^*$ and $\mu^*$ satisfy certain normalization conditions, which we defer to Section~\ref{sec:analysis} for the ease of presentation. The most significant benefit for linear MDP is that, for any policy $\pi:\mathcal{S} \to \mathcal{A}$, $Q_{T^*, r}^\pi(s, a)$ is linear with respect to $[r(s, a), \phi^*(s, a)]$, thanks to the following observation:
\begin{equation}\label{eq:linear_q}
    Q_{T^*, r}^\pi(s, a) = r(s, a) + \gamma\mathbb{E}_{s^\prime\sim T^*(\cdot|s, a)}\left[V_{T^*, r}^\pi(s^\prime)\right] = r(s, a) + \left\langle \phi^*(s, a), \int_{\mathcal{S}} \mu^*(s^\prime) V_{T^*, r}^\pi(s^\prime) d s^\prime\right\rangle_{\mathcal{H}}.
\end{equation}
Plenty of sample-efficient algorithms have been developed based on the linear MDP structure with known $\phi^*$~\citep[\eg][]{yang2020reinforcement, jin2020provably, yang2020provably}. This requirement limits their practical applications. In fact, in most cases, we do not have access to $\phi^*$ and we need to perform representation learning to obtain an estimate of $\phi^*$. However, the learning of $\phi$ relies on efficient exploration for the full-coverage data, while the design of exploration strategy relies on the accurate estimation of $\phi$. The coupling between exploration and learning induces extra difficulty. 

Recently,~\citet{uehara2021representation} designed UCB-style exploration for iterative finite-dimension representation updates with theoretical guarantees. The algorithm requires the computaiton oracle for the maximum likelihood estimation (MLE) to the conditional density estimation, 
\begin{align}\label{eq:mle_unnormalized}
    \max_{\phi, \mu} \sum_{i=1}^n \log \langle \phi(s_i, a_i), \mu(s_i^\prime)\rangle_{\mathcal{H}}, \quad \mathrm{s.t.} \quad  \forall (s, a), \quad \left\langle \phi(s, a), \int_{\mathcal{S}}\mu(s^\prime) ds^\prime\right\rangle_{\mathcal{H}} = 1,
\end{align}
which is difficult as we generally do not have specific realization of $(\phi, \mu)$ pairs to make the constraints hold for arbitrary $(s, a)$ pairs, and therefore, impractical for real-world applications. 

\vspace{-2mm}
\section{Latent Variable Models as Linear MDPs}\label{sec:method}
\vspace{-2mm}
%%%%%%%%%%%%%%%%%%%%%%%%%%%%%%%%%%%%%%%%%%%%%%%%%%%%%%%
% \Tongzheng{Some connections before introducing the concept} \Bo{Let's use linear MDP instead of lowrank MDP, as suggested by Lina. }
In this section, we first reveal the linear representation view of the transitions with a latent variable structure. This essential connection brings several benefits for learning, planning and exploration/exploitation. More specifically, the latent variable model view provides us a tractable variational learning scheme, while the linear representation view inspires computational-efficient planning and exploration/exploitation mechanism. 
%
% We first define the latent variable representation as following,
% \begin{definition}[Latent Variable Representation (\algabb) \citep{agarwal2020flambe}]
%     The transition operator $P^*:\mathcal{S}\times \mathcal{A} \to \Delta(\mathcal{S})$ is said to have a latent variable representation (\algabb) if there exist latent space $\mathcal{Z}$ and two conditional probability measure $p^*(z|s, a)$ and $P^*(s^\prime|z)$, such that
%     \begin{align}
%         \label{eq:latent_variable_model}
%         P^*(s^\prime|s, a) = \int_{\mathcal{Z}} p^*(z|s, a) P^*(s^\prime|z) d \mu,
%     \end{align}
%     where $\mu$ is the Lebesgue measure on $\mathcal{Z}$ when $\mathcal{Z}$ is continuous and $\mu$ is the counting measure on $\mathcal{Z}$ when $\mathcal{Z}$ is discrete.
% \end{definition}

We focus on the transition operator $T^*:\mathcal{S}\times \mathcal{A} \to \Delta(\mathcal{S})$ with a latent variable structure, \ie, there exist latent space $\mathcal{Z}$ and two conditional probability measure $p^*(z|s, a)$ and $p^*(s^\prime|z)$, such that
\begin{align}
    \label{eq:latent_variable_model}
    T^*(s^\prime|s, a) = \int_{\mathcal{Z}} p^*(z|s, a) p^*(s^\prime|z) d \mu,
\end{align}
where $\mu$ is the Lebesgue measure on $\mathcal{Z}$ when $\mathcal{Z}$ is continuous and $\mu$ is the counting measure on $\mathcal{Z}$ when $\mathcal{Z}$ is discrete.

Assume that $p^*(\cdot|s, a) \in L_2(\mu)$, $p^*(s^\prime|\cdot) \in L_2(\mu)$, we have the equivalent formulation of~\eq{eq:latent_variable_model} as
% \lina{problem of notations $P^*$} then we can write \eqref{eq:latent_variable_model} in the following equivalent formulation:
\begin{align*}
    T^*(s^\prime|s, a) = \left\langle p^*(\cdot|s, a), p^*(s^\prime|\cdot)\right\rangle_{L_2(\mu)}, 
\end{align*}
which obviously demonstrates the linear MDP structure following Defintion~\ref{def:low-rank-MDP}, and immediately implies $\phi^*(s, a) = p^*_z\rbr{\cdot|s, a}$, and $\mu^*\rbr{s^\prime} = p^*(s^\prime|\cdot)$. We call $p^*_z\rbr{\cdot|s, a}$ as \emph{\AlgName~(\algabb)}. 
% \lina{the notation is not good here. $P^*$ means three different distributions: $P^*(s^\prime|s,a)$, $P^*(z|s,a)$, $P^*(s^\prime|z)$. It causes confusions later. Can we introduce indices or other notations?}

\vspace{-2mm}
\paragraph{Connection to~\citet{ren2022free}.} 
\revise{To provide a concrete example of LV-rep, we consider the stochastic nonlinear control model with Gaussian noise \citep{ren2022free}, which is widely used in most of model-based RL algorithms.} Such a model can be understood as a special case of~\algabb. In~\citet{ren2022free}, the transition operator is defined as
% \vspace{-1mm}
\begin{equation}
\textstyle
    T^*\rbr{s^\prime|s, a} = \rbr{2\pi\sigma^2}^{-d/2}\exp\rbr{-\nbr{s^\prime - f^*\rbr{s, a}}^2/(2\sigma^2)} = \inner{p^*(\cdot|s, a)}{p^*\rbr{s^\prime|\cdot}}_{L_2(\mu)}, 
    % \rbr{2\pi\sigma^2}^{-\frac{d}{2}}\inner{\exp\rbr{-\frac{\nbr{z' - f^*\rbr{s, a}}}{\sigma^2}}}{\exp\rbr{-\frac{\nbr{z' - s^\prime}}{\sigma^2}}}_{L_2(\mu)},
\end{equation}
where $p^*\rbr{z|s, a} \propto \exp\rbr{-2\nbr{z - f^*\rbr{s, a}}^2/\sigma^2}$ and $p^*\rbr{s^\prime|z} \propto \exp\rbr{-2\nbr{z - s^\prime}^2/\sigma^2}$, both following the Gaussian distributions. The proposed~\algabb can exploit more general distributions beyond Gaussian for $p^*\rbr{\cdot|s, a}$ and $p^*\rbr{s^\prime|z}$, that introduces more flexibility in transition modeling.

Our definition of \algabb is more general than the original definition \revise{(Definition 2)} in \citet{agarwal2020flambe}, which assumes $|\mathcal{Z}|$ is finite. 
% \paragraph{On the Expressive Power of~\algabb.} 
As shown by \citet{agarwal2020flambe}, block MDPs \citep{du2019provably, misra2020kinematic} with finite latent state space $\mathcal{Z}$ have a latent variable representation where $\mathcal{S}$ corresponds to the set of observation, $\mathcal{Z}$ corresponds to the set of latent state, and \revise{$p^*(z^\prime|s, a)$ is a composition of deterministic $p(z|s)$ and a transition $p(z^\prime|z, a)$.} \citet{agarwal2020flambe} also remarks that, compared with the latent variable representation, the original low-rank representation relaxes the simplex constraint on the $p^*(z|s, a)$, and thus, can be more compact with fewer dimensions. However, the ambient dimension may not be a proper measure of the representation complexity. As we will show in~\Secref{sec:analysis}, even we work on the infinite $\mathcal{Z}$, as long as $p(z|s, a) \in \mathcal{H}_k$ and $k$ satisfies standard regularity conditions, we can still perform sample-efficient learning. A proper measure of the representation complexity is still an open problem to the whole community.
% However, we expand the representation ability of~\algabb by generalizing the feature to infinite dimension. As we will explained in~\Secref{sec:analysis}, the number of dimensions is not the essential measure of the complexity, 
% but the decay rate of the eigensystem of the kernel for $p^*\rbr{z|s, a}$. 
% \Bo{@tongzheng, check this argument. }
% the original low-rank representation can be hard to learn and sample. 
% \Tongzheng{@Bo: shall we keep this paragraph? I think we can answer this if some reviewers ask it.}
% \lina{and I have questions on the word of ``low-rank'' as I pointed out earlier.}
% \lina{In general, I would like to see more intuitive discussion provided on why we consider latent variables and the potential benefit of it. Discussing the drawback is also helpful for readers to understand why we are doing it. Also if people go ahead to keep adding ``latent'' variables, pros and cons. }
 
The \algabb with \revise{$p^*(\cdot|s, a)$ and} $p^*\rbr{s^\prime|\cdot}$ naturally satisfies the distribution requirements, which brings the benefits of efficient sampling and learning. 

\vspace{-2mm}
\paragraph{Efficient Simulation from~\algabb.} Specifically, 
% One of the benefits of the \algabb is that, 
we can easily draw samples from the learned model $\hat{T}(s^\prime|s, a) = \int_{\mathcal{Z}} \hat{p}(z|s, a)\hat{p}(s^\prime|z)d\mu$ by first sampling $z_i\sim \hat{p}(z|s, a)$, then sampling $s_i^\prime \sim p(s^\prime|z_i)$, without the need to call other complicated samplers, \eg, MCMC, for the general unnormalized transition operator in linear MDPs. Such a property is important for computation-efficient planning on the learned model. 
% Meanwhile, it also induces a tractable variational method for representation learning. 
% , which is required by all existing model-based reinforcement learning methods. 

\vspace{-2mm}
\paragraph{Variational Learning of~\algabb.}
% \lina{as I pointed out earlier, we need an overview paragraph before ``parameterization of Q function'' to connect all these components.} 
Another significant benefit of the \algabb is that, we can leverage the variational method to obtain a tractable surrogate objective of MLE, which is also known as the evidence lower bound (ELBO) \citep{kingma2019introduction}, that can be derived as follows:
\begin{align}
    \log T(s^\prime|s, a) = & \log \int p^*\rbr{z|s, a}p^*(s^\prime|z)dz = \log \int \frac{p^*\rbr{z|s, a}p^*(s^\prime|z)}{q(z|s, a, s^\prime)} q(z|s, a, s^\prime)dz\nonumber \\
    % \geq & \log T(s^\prime|s, a) - \mathrm{KL}\left(q(z|s, a, s^\prime)\|p(z|s, a, s^\prime)\right)\nonumber\\
    = & \max_{q\in \Delta\rbr{\gZ}} \E_{z\sim q(\cdot|s, a, s^\prime)} \left[\log p^*(s^\prime|z) \right] - \KL\left(q(z|s, a, s^\prime)\| p^*(z|s, a)\right),
    \label{eq:ELBO}
\end{align}
% \lina{$Q(z|s, a, s^\prime)$ is firstly used in the equation but there is no place introducing or discussing it. Plus it is the same notation as $Q$ function. It definitely causes confusion. $Q(z|s, a, s^\prime)$ is used in the alg so it is important to make it clear.} 
where $q(z|s, a, s^\prime)$ is an auxiliary distribution. The last equality comes from Jensen's inequality, and the equality only holds when $q\rbr{z|s, a, s^\prime} = p\rbr{z|s, a, s^\prime} \propto p\rbr{z|s, a}p\rbr{s^\prime|z}$. 

Compared with the standard MLE used in \citep{agarwal2020flambe, uehara2021representation}, maximizing the ELBO is more computation-efficient, as it avoids the computation of integration at any time. 
% \Tongzheng{Benefits: we never require to compute or regularize the normalizing constant, at the cost of introducing additional bias if we don't have sufficiently flexible posterior family.} 
Meanwhile, if the family of variational distribution $q$ is sufficient flexible that contains the optimal $p(z|s, a, s^\prime)$ for any possible $(p(z|s, a), p(s^\prime|z))$ pair,
% for $P\in\mathcal{P}$ where $\mathcal{P}$ is the model class we consider, 
then maximizing the ELBO is equivalent to perform MLE, \ie, they share the same solution, 
% \lina{can we add a phrase to explain ``equivalent''? we first saying that this method is more computational efficient and then ``equivalent'', it is not very clear. Plus as it is always the case that there is no free lunch: the method is more comptuational efficient, then what is the limit of the method?} \Bo{the limit is we have one auxiliary distribution to be optimized.}

% \lina{I think we should add a paragraph outlining what we are going to do next provide a flow. Right now, the flow of the rest of the section is unclear.  }

%%%%%%%%%-----------------------------------------------------------
\vspace{-2mm}
\subsection{Reinforcement Learning with~\algabb}\label{sec:rl_rep}
\vspace{-2mm}
%%%%%%%%%-----------------------------------------------------------

As the transition operator is linear with respect to~\algabb, the state-action function for arbitrary policy can be linearly represented by~\algabb. Once the~\algabb is learned, we can execute planning and exploration in the linear space formed by~\algabb. Due to the space limit, we mainly consider online exploration setting, and the offline policy optimization is explained in~\Appref{sec:offline}. 

\begin{algorithm}[t] 
\caption{Online Exploration with \algabb} \label{alg:online_algorithm-main}
\begin{algorithmic}[1]
  \State \textbf{Input:} Model class $\mathcal{P}=\{(p(z|s, a), p(s^\prime|z))\}, \mathcal{Q} = \{q(z|s, a, s^\prime)\}$, Iteration $N$.
  \State \textbf{Initialize} $\pi_0(s) = \mathcal{U}(\mathcal{A})$ where $\mathcal{U}(\mathcal{A})$ denotes the uniform distribution on $\mathcal{A}$; $\mathcal{D}_0 = \emptyset$; $\mathcal{D}_0^\prime = \emptyset$.
  \For{episode $n=1,\cdots,N$ } 
  \State Collect the transition $(s,a,s^\prime, a^\prime, \Tilde{s})$ where $s\sim d_{T^*}^{\pi_{n-1}}$, $a\sim \mathcal{U}(\mathcal{A})$, $s^\prime\sim T^*(\cdot | s,a)$,$ a^\prime \sim \mathcal{U}(\mathcal{A})$, $\tilde{s} \sim T^*(\cdot|s^\prime, a^\prime)$. $\mathcal{D}_n = \mathcal{D}_{n-1} \cup \{s,a,s^\prime\}$, $\mathcal{D}_n^\prime = \mathcal{D}_{n-1}^\prime \cup \{s^\prime, a^\prime, \tilde{s}\}$. \label{line:data_collection}
  % \State Add $(s,a,s^\prime)$ to the dataset: 
  \State Learn the latent variable model $\hat{p}_n(z|s, a)$ with $\mathcal{D}_n \cup \mathcal{D}_n^\prime$ via maximizing the ELBO in \eq{eq:ELBO}, and obtain the learned model $\hat{T}_n$.\label{line:representation_online} 

  \State Set the exploration bonus $\hat{b}_n(s,a)$ as \eq{eq:empirical_bonus}.\label{line:bonus}
  \State Update policy \label{line:online_plan}
    $ \pi_n=\mathop{\arg\max}_{\pi}V^{\pi}_{\hat{T}_n,r+\hat {b}_n}$.
  \EndFor
  \State \textbf{Return } $\pi_1,\cdots,\pi_N$.
\end{algorithmic}
\end{algorithm}

\vspace{-2mm}
\paragraph{Practical Parameterization of $Q$ function.} 
% Hence, if $P^*$ has a \algabb, the corresponding MDP instance can be viewed as a linear MDP.
% \lina{rather than say ``low rank'', can we just say ``linear factorized MDP" something like this? depending on the dimension of the hidden space, it is not necessarily to be low rank.} 
% We generalize the~\algabb to include both finite and infinite latent variables. 
With the linear factorization of dynamics through latent variable models~\eq{eq:linear_q}, we have
\begin{equation}\label{eq:exp_q}
    Q_{T^*, r}^\pi(s, a) = r(s, a) + \gamma\E_{p^*(z|s, a)}\sbr{w^\pi(z)},
\end{equation}
\revise{where $w^\pi(z) = \int_{\mathcal{S}} p^*(z|s^\prime) V_{T^*, r}^\pi(s^\prime) ds^\prime$ can be viewed as a value function of the latent state.} When the latent variable is in finite dimension, \ie, $\abr{\gZ}$ is finite, we have $w = [w(z)]_{z\in\gZ}\in \sR^{\abr{\gZ}}$, and the expectation $\E_{p^*(z|s, a)}\sbr{w^\pi(z)}$ can be computed exactly by enumerating over $\mathcal{Z}$. 

However, when $\mathcal{Z}$ is not a finite set, generally we can not exactly compute the expectation, which makes the representation of $Q$ function through $p^*(z|s, a)$ hard.
% If $P^*(\cdot|s, a) \in L_2(\mu)$ and $\int_{\mathcal{S}} P^*(s^\prime|z) V_{P^*, r}^\pi(s^\prime) d s^\prime \in L_2(\mu)$, we can write
% \begin{align*}
%     Q_{P^*, r}^\pi(s, a) = r(s, a) + \langle P^*(z|s, a), w^\pi(z)\rangle_{L_2(\mu)},
% \end{align*}
% where $w^\pi\in L_2(\mu)$. 
% When $|\mathcal{Z}|$ is finite, we can directly obtain an estimate of $w^\pi$ with the standard approximate dynamic programming style algorithm \citep[e.g.][]{munos2008finite}. 
% However, when $\mathcal{Z}$ is continuous, evaluating the $L_2$ inner product can be computation-intractable. \Tongzheng{@Bo: Please check the following part.} 
\iffalse
Fortunately, as $p^*(z|s, a)$ is a probability measure, we can perform Monte Carlo approximation, which leads to:
\begin{align*}
    Q_{T^*, r}^\pi(s, a) \approx r(s, a) + \frac{1}{m}\sum_{i\in [m]} w^\pi(z_i),
\end{align*}
and we can parameterize $w^\pi\rbr{z}$ instead. 
\fi
% and estimate $w^\pi(z)$ with the approximate dynamic programming style algorithm. 
Particularly, under our normalization condition~\Asmpref{assumption:normalization} shown later, we have $w^\pi \in \mathcal{H}_k$ where $\mathcal{H}_k$ is a reproducing kernel Hilbert space with kernel $k$. When $k$ admits a random feature representation (see Definition~\ref{def:random_feature_representation}), we can then express $w^\pi$ as:
\begin{align*}
    w^\pi(z) = \int_{\Xi} \widetilde{w}^\pi(\xi)\psi(z; \xi) d P(\xi),
\end{align*}
where the concrete $P(\xi)$ depends on the kernel $k$. Plug this representation of $w^\pi(z)$ into~\eq{eq:exp_q}, we obtain the approximated representation of $Q_{T^*, r}^\pi(s, a)$ as:
% \vspace{-1mm}
\begin{align*}
Q_{T^*, r}^\pi(s,a) & = r(s, a) + \gamma\int_{\mathcal{Z}}  w^\pi(z) p^*(z|s, a) d\mu \\
& = r(s, a) + \gamma\int_\mathcal{Z} \int_\Xi \widetilde{w}(\xi) \psi(z; \xi) d P(\xi)\cdot  p^*(z|s,a) d\mu\\
& \approx r(s, a) + \frac{\gamma}{m} \sum_{i\in[m]} \widetilde{w}(\xi_i) \psi(z_i; \xi_i),
\end{align*}
which shows that we can approximate $Q_{T^* r}^\pi(s, a)$ with a linear function on top of the random feature $\varphi(s, a) = [\psi(z_i; \xi_i)]_{i\in [m]}$ where $z_i \sim p^*(z|s, a)$ and $\xi_i \sim P(\xi)$. % \lina{As a reader, question in my head is that why we need to further introduce another layer of $xi$? what's the benefit? Moreover, since $\tilde{omega}$ and $P(\xi)$ are very abstract, it is unclear why we are doing it. If we can explain it in a way, it would be great. This would also helps the last sentence, ``in practice, we can use deeper neural network ... performance''. }
This can be viewed as a two-layer neural network with fixed first layer weight $\xi_i$ and activation $\psi$ and trainable second layer weight $\widetilde{w} = \sbr{\widetilde{w}(\xi_i)}_{i=1}^m\in\sR^m$. 
% In practice, we can use deeper neural network and make the $\xi$ trainable to further boost the performance.

\vspace{-2mm}
\paragraph{Planning and Exploration with \algabb.} 
Following the idea of REP-UCB~\citep{uehara2021representation}, we introduce an additional bonus to implement the principle of optimism in the face of uncertainty.
We use the standard elliptical potential for the upper confidence bound, which can be computed efficiently as below,
\begin{align}
    \hat{\varphi}_n(s, a) = & \left[\psi(z_i; \xi_i)\right]_{i\in [m]}, \quad \text{where} \quad \{z_i\}_{i\in [m]} \sim \hat{p}_n(z|s, a), \quad \{\xi_i\}_{i\in [m]} \sim P(\xi), \nonumber\\
    \hat{b}_n(s, a) = & \alpha_n \hat{\varphi}_n(s, a) \hat{\Sigma}_n^{-1} \hat{\varphi}_n(s, a),\quad \text{with}\quad \hat{\Sigma}_n = \sum_{(s_i, a_i) \in \mathcal{D}_n} \hat{\varphi}_n(s_i, a_i) \hat{\varphi}_n(s_i, a_i)^\top + \lambda I, \label{eq:empirical_bonus}
\end{align}
where $\alpha_n$ and $\lambda$ are some constants, and $\mathcal{D}_n$ is the collected dataset. 
% \Tongzheng{try to build some connection with the theory}\lina{is the bonus $\hat{b}_n(s,a)$? We should revise the order of \eqref{eq:empirical_bonus}: start with $\hat{b}(s,a):=...$ where $\hat{\Sigma}_n:=...$, and $\hat{\phi}(s,a):=....$ }

The planning can be then completed by Bellman recursion with bonus, \ie, 
\begin{equation}
    Q^\pi\rbr{s, a} = r\rbr{s, a} + \hat{b}_n\rbr{s, a} + \gamma \E_{T}\sbr{V^\pi\rbr{s'}}.
\end{equation}
% However, due to the nonlinear $\hat{b}_n$, the $Q^\pi$ may not represent
We can exploit the augmented feature $[r(s, a), \varphi(s, a), \hat{b}_n\rbr{s, a}]$ to linearly represented $Q^\pi$ after bonus introduced. However, there will be an extra $\gO\rbr{m^2}$ due to the bonus in feature. Therefore, we consider a two-layer MLP upon $\varphi$ to parametrize $Q(s, a) = w_0r(s, a) +\widetilde{w}_1^\top\varphi(s, a) + \widetilde{w}_2^\top\sigma\rbr{\widetilde{w}_3^\top\varphi(s, a)}$, where $\sigma\rbr{\cdot}$ is a nonlinear activation function, used for complement the effect of the nonlinear $\hat{b}_n$.
We finally conduct approximate dynamic programming style algorithm~\citep[\eg][]{munos2008finite} with the $Q$ parameterization.

\vspace{-2mm}
\paragraph{The Complete Algorithm.} We show the complete algorithm for the online exploration with \algabb in Algorithm~\ref{alg:online_algorithm-main}. Our algorithm follows the standard protocol for sequential decision making. In each episode, the agent first executes the exploratory policy obtained from the last episode and collects the data (Line~\ref{line:data_collection}). The data are later used for training the latent variable model by maximizing the ELBO defined in \eqref{eq:ELBO} (Line~\ref{line:representation_online}). With the newly learned $\hat{p}_n(z|s, a)$, we add the exploration bonus defined in \eqref{eq:empirical_bonus} to the reward (Line~\ref{line:bonus}), and obtain the new exploratory policy by planning on the learned model with the exploration bonus (Line~\ref{line:online_plan}), that will be used in the next episode. Note that, in Line~\ref{line:data_collection}, we requires to sample $s \sim d_{T^*}^{\pi_{n-1}}$, which can be obtained by starting from $s_0\sim d_t$, executing $\pi_{n-1}$, stopping with probability $1-\gamma$ at each time step $t\geq 0$ and returning $s_t$. \algabb can also be used for offline exploitation, and we defer the corresponding algorithm to Appendix~\ref{sec:offline}.
\vspace{-2mm}
\section{Theoretical Analysis}\label{sec:analysis}
\vspace{-2mm}
%%%%%%%%%%%%%%%%%%%%%%%%%%%%%%%%%%%%%%%%%%%%%%%%%%%%%%%
In this section, we provide the theoretical analysis of representation learning with \algabb. Before we start, we introduce the following two assumptions, that are widely used in the community~\citep[e.g.][]{agarwal2020flambe, uehara2021representation}.
\begin{assumption}[Finite Candidate Class with Realizability]
    \label{assumption:function_class}
    $|\mathcal{P}| < \infty$ and $(p^*(z|s, a), p^*(s^\prime|z)) \in \mathcal{P}$. Meanwhile, for all $(p(z|s, a), p(s^\prime|z)) \in \mathcal{P}$, $p(z|s, a, s^\prime) \in \mathcal{Q}$. 
    % \lina{Should it be $Q(z|s, a, s^\prime)$?}
\end{assumption}
\begin{remark}
% \Bo{please consider to change the font in remark, instead of using italic.}
% \paragraph{Remark:}
    The assumption on $\mathcal{P}$ is widely used in the community \citep[e.g.][]{agarwal2020flambe, uehara2021representation}, while the assumption on $\mathcal{Q}$ is to guarantee the estimator obtained by maximizing the ELBO defined in \eqref{eq:ELBO} is identical to the estimator obtained by MLE. We would like to remark that, the extension to other data-independent function class complexity (e.g. Rademacher complexity \citep{bartlett2002rademacher}) can be straightforward with a refined non-asymptotic generalization bound of MLE.
\end{remark}

\begin{assumption}[Normalization Conditions]
    \label{assumption:normalization}
    $\forall P\in \mathcal{P}, (s, a) \in \mathcal{S}\times\mathcal{A}, \|p(\cdot|s, a)\|_{\mathcal{H}_k} \leq 1$. Furthermore, $\forall g:\mathcal{S} \to \mathbb{R}$ such that $\|g\|_{\infty} \leq 1$, we have $\left\|\int_{\mathcal{S}} p(s^\prime|\cdot) g(s^\prime) d s^\prime\right\|_{\mathcal{H}_k} \leq C$. 
\end{assumption}

\begin{remark}
    Our assumptions on normalization conditions is substantially different from standard linear MDPs.
    % \lina{I actually like the idea of we change the Definition~\ref{def:linear-MDP} to a term like ``linear factorizable MDP'' to differential from linear MDP literature. So everytime when we discuss the difference between this paper and linear MDP, like in this remark, it is easier to refer what MDP we are talking about.} 
    Specifically, standard linear MDPs assume that the representation $\phi(s, a)$ and $\mu(s^\prime)$ are of finite dimension $d$, with $\|\phi(s, a)\|_2 \leq 1$ and $\forall \|g\|_{\infty}\leq 1$, $\left\|\int_{\mathcal{S}} \mu(s^\prime) g(s^\prime) d s^\prime\right\|_{2} \leq d$. When $|\mathcal{Z}|$ is finite, as $\|f\|_{L_2(\mu)} \leq \|f\|_{\mathcal{H}_k}$, our normalization conditions are more general than 
    % \lina{more general than? or you would like to say ``as general as''} 
    the counterparts of the standard linear MDPs and we can use the identical normalization conditions as the standard linear MDPs. However, when $|\mathcal{Z}|$ is infinite, if we assume $\|p(z|s, a)\|_{L_2(\mu)} \leq 1$, we cannot provide a sample complexity bound without polynomial dependency on $|\mathcal{P}|$, which can be unsatisfactory. Furthermore, we would like to note that, the assumption $\int_{\mathcal{S}} p(s^\prime|z) g(s^\prime) d s^\prime \in \mathcal{H}_k$ is mild, which is necessary for justifying the estimation from the approximate dynamic programming algorithm. 
    % as in practice we generally use the approximate dynamic programming algorithm to estimate such term, which can be impossible if such assumption does not hold.
    % \lina{need to revise this sentence ``Furthermore, we would like to ... not hold''. ``is mild'' means that most systems satisfy it; but ``as in practice we generally ...'' means that we don't apply ADP if the condition doesn't satisfy. Or we honestly say that ''the assumption is to allow the use of ADP to estimate ..., which is often assumed in literature''. }.
\end{remark}

% \Tongzheng{TODO: Connections to the normalzation conditions in linear MDP.}

% \Tongzheng{TODO: Connection to the effective dimension/maximum information gain. We don't directly depend on them. Also remarks after the main theorem.} 

% \Tongzheng{We need to know the number of episode in prior. We can admit this.}

\begin{theorem}[PAC Guarantee for Online Exploration, Informal]
    \label{thm:pac_online_informal}
    Assume the reproducing kernel $k$ satisfies the regularity conditions in Appendix~\ref{sec:technical_conditions}. If we properly choose the exploration bonus $\hat{b}_n(s, a)$, we can obtain an $\varepsilon$-optimal policy with probability at least $1-\delta$ after we interact with the environments for $N = \mathrm{poly}\left(C, |\mathcal{A}|,(1-\gamma)^{-1},\varepsilon, \log(|\mathcal{P}|/\delta)\right)$ episodes.
\end{theorem}
\begin{remark}
    Although $|\mathcal{Z}|$ may not be finite, we can still obtain a sample complexity independent w.r.t. $\abr{\gS}$, while has polynomial dependency on $C$, $|\mathcal{A}|$, $(1-\gamma)^{-1}$ and $\varepsilon$ and $\log |\mathcal{P}|$ with the assumption that $p(\cdot|s, a) \in \mathcal{H}_k$ and some standard regularity conditions for the kernel $k$. This means that we do not really need to assume a discrete $\mathcal{Z}$ with finite cardinality, but only need to properly control the complexity of the representation class, by either the ambient dimension $|\mathcal{Z}|$, or some ``effective dimension'' that can be derived from  the eigendecay of the kernel $k$ (see Appendix~\ref{sec:technical_conditions} for the details). The formal statement for Theorem~\ref{thm:pac_online_informal} and the proof is deferred to Appendix~\ref{sec:online_proof}. We also provide the PAC guarantee for offline exploitation with \algabb in Appendix~\ref{sec:offline_proof}.
\end{remark}

\iffalse
\begin{remark}
    The bonus we use in the theoretical analysis is different from the one we use in practice. \Tongzheng{TODO: try to make some connections}
\end{remark}
\fi

\begin{remark}
    Our proof strategy is based on the analysis of REP-UCB \citep{uehara2021representation}. However, there are substantial differences between our analysis and the analysis of REP-UCB, as the representation we consider can be infinite-dimensional, and hence the analysis of REP-UCB, which assumes that the feature is finite-dimensional, cannot be directly applied in our case. As we mentioned, to address the infinite dimension issue, we assume the representation $p(z|s, a)\in\mathcal{H}_k$ and prove two novel lemmas, one for the concentration of bonus (Lemma~\ref{lem:bonus_concentration}) and one for the ellipsoid potential bound (Lemma~\ref{lem:potential_function_RKHS}) when the representation lies in the RKHS. 
    We further note that, different from the work on the kernelized bandit and kernelized MDP \citep{srinivas2010gaussian, valko2013finite, yang2020provably} that assume the reward function and $Q$ function lies in some RKHS, we assume the condition density of the latent random variable lies in the RKHS and the $Q$ function is the $L_2(\mu)$ inner product of two functions in RKHS. As a result, the techniques used in their work cannot be directly adapted to our setting, and their regret bounds depend on the alternative notions of maximum information gain and effective dimension of the specific kernel, which can be implied by the eigendecay conditions we assume in Appendix~\ref{sec:technical_conditions} (see \citet{yang2020provably} for the details). 
\end{remark}
%%%%%%%%%%%%%%%%%%%%%%%%%%%%%%%%%%%%%%%%%%%%%%%%%%%%%%%
\vspace{-2mm}
\section{Related Work}\label{sec:related_work}
\vspace{-2mm}
%%%%%%%%%%%%%%%%%%%%%%%%%%%%%%%%%%%%%%%%%%%%%%%%%%%%%%%
% \Bo{TODO: refine the related work, add empirical representation learning algorithms.}

% \paragraph{Representation Learning in RL.} 
There are several other theoretically grounded representation learning methods under the assumption of linear MDP. However, most of these work either consider more restricted model or totally ignore the computation issue. 
\citet{du2019provably, misra2020kinematic} focused on the representation learning in block MDPs, which is a special case of linear MDP~\citep{agarwal2020flambe}, and proposed to learn the representation via the regression. However, both of them used policy-cover based exploration that need to maintain large amounts of policies in the training phase, which induces a significant computation bottleneck. \citet{uehara2021representation} and \citet{zhang2022efficient} exploit UCB upon learned representation to resolve this issue. However, their algorithms depend on some computational oracles, \ie, MLE for unnormalized conditional distribution in~\eq{eq:mle_unnormalized} or a $\max-\min-\max$ optimization solver motivated from \citet{modi2021model}, respectively, 
% model-free representation learning oracle motivated from \citet{modi2021model} requires solving a complicated $\max-\min-\max$ optimization problem, 
that can be hard to implement in practice. 

A variety of recent work have been proposed to replace the computational oracle with more tractable estimators. For example, \citet{ren2022free} exploited representation with the structure of Gaussian noise in nonlinear stochastic control problem with arbitrary dynamics, which restricts the flexibility. \citet{zhang2022making, qiu2022contrastive} proposed to use a contrastive learning approach as an alternative.
% , that can avoid estimating $\int_{\mathcal{S}}\mu(s^\prime) d s^\prime$ in the representation learning. 
However, similar to other contrastive learning approach, both of their methods require the access to a negative distribution supported on the whole state space, and their performance highly depends on the quality of the negative distribution.
% , which can be unsatisfactory in practical application. 
\citet{ren2022spectral} designed a new objective based on the idea of the spectral decomposition.
%, which do not require either estimating $\int_{\mathcal{S}}\mu(s^\prime) d s^\prime$ or designing good negative distribution. 
But the solution for their objective is not necessarily to be a valid distribution, and the generalization bound is worse than the MLE when the state space is finite. 

Algorithmically, many representation learning methods have been developed for different purposes, such as state extractor from vision-based features~\citep{laskin2020curl,laskin2020reinforcement,kostrikov2020image}, bi-simulation~\citep{ferns2004metrics,gelada2019deepmdp,zhang2020learning}, successor feature~\citep{dayan1993improving,barreto2017successor,kulkarni2016deep}, spectral representation from transition operator decomposition~\citep{mahadevan2007proto,wu2018laplacian,duan2019state}, contrastive learning~\citep{oord2018representation,nachum2021provable,yang2020provably}, and so on. 
However, most of these methods are designed for state-only feature, ignoring the action dependency, and learning from pre-collected datasets, without taking the planning and exploration in to account and ignoring the coupling between representation learning and exploratin. Therefore, there is no rigorously theoretical characterization provided. 

We would like to emphasize that the proposed~\algabb is the algorithm which achieves both statistical efficiency theoretically and computational tractability empirically. For more related work on model-based RL, please refer to~\Appref{appendix:more_related}. 

%%%%%%%%%%%%%%%%%%%%%%%%%%%%%%%%%%%%%%%%%%%%%%%%%%%%%%%
\vspace{-2mm}
\section{Experiments}\label{sec:exp}
\vspace{-2mm}
%%%%%%%%%%%%%%%%%%%%%%%%%%%%%%%%%%%%%%%%%%%%%%%%%%%%%%%

% We conduct a series of experiments on large continuous control problems to evaluate \algabb. 
We extensively test our algorithm on the Mujoco \citep{todorov2012mujoco} and DeepMind Control Suite \citep{tassa2018deepmind}. 
Before presenting the experiment results, 
we first discuss some details towards a practical implementation of \algabb.

\newcommand{\algabbc}{{LV-Rep-C}\xspace}
\newcommand{\algabbd}{{LV-Rep-D}\xspace}

\begin{table*}[t!]
% \vspace{-1.5em}
\caption{\footnotesize Performance on various MuJoCo control tasks. All the results are averaged across 4 random seeds and a window size of 10K. Results marked with $^*$ is adopted from MBBL.
\revise{\algabb-C and \algabb-D use continuous and discrete latent variable model respectively.}
\algabb achieves strong performance compared with baselines.
 }
\scriptsize
\setlength\tabcolsep{1.5pt}
\label{tab:MuJoCo_results}
\centering
\begin{tabular}{p{2cm}p{2.2cm}p{2cm}p{1.8cm}p{1.8cm}p{1.8cm}p{1.8cm}p{1.8cm}}
% {lcccccccccccc}
\toprule
& & HalfCheetah & Reacher & Humanoid-ET & Pendulum & I-Pendulum \\ 
\midrule  
\multirow{4}{*}{Model-Based RL} & ME-TRPO$^*$ & 2283.7$\pm$900.4 & -13.4$\pm$5.2 & 72.9$\pm$8.9 & \textbf{177.3$\pm$1.9} & -126.2$\pm$86.6\\
& PETS-RS$^*$  & 966.9$\pm$471.6 & -40.1$\pm$6.9 & 109.6$\pm$102.6 & 167.9$\pm$35.8 & -12.1$\pm$25.1\\
& PETS-CEM$^*$  & 2795.3$\pm$879.9 & -12.3$\pm$5.2 & 110.8$\pm$90.1 & 167.4$\pm$53.0 & -20.5$\pm$28.9\\
& Best MBBL & 3639.0$\pm$1135.8 & \textbf{-4.1$\pm$0.1} & 1377.0$\pm$150.4 & \textbf{177.3$\pm$1.9} & \textbf{0.0$\pm$0.0}\\
\midrule
\multirow{3}{*}{Model-Free RL} & PPO$^*$ & 17.2$\pm$84.4 & -17.2$\pm$0.9 & 451.4$\pm$39.1 & 163.4$\pm$8.0 & -40.8$\pm$21.0 \\
& TRPO$^*$ & -12.0$\pm$85.5 & -10.1$\pm$0.6 & 289.8$\pm$5.2 & 166.7$\pm$7.3 & -27.6$\pm$15.8 \\
% TD3$^*$ & 40.4$\pm$8.3 & -14.0$\pm$0.9 & -60.0$\pm$1.2 & 161.4$\pm$14.4 & -224.5$\pm$0.4 \\
& SAC$^*$ (3-layer)  & 4000.7$\pm$202.1 & -6.4$\pm$0.5 & \textbf{1794.4$\pm$458.3} & 168.2$\pm$9.5 & -0.2$\pm$0.1\\
% SPEDE-REG \\
% {\bf \algabb-REG} & 40.0$\pm$3.8 & \textbf{-5.8$\pm$0.6} & 40.0$\pm$3.8 & \textbf{168.5$\pm$4.3} & \textbf{0.0$\pm$0.1}\\
\midrule
\multirow{5}{*}{Representation RL} & DeepSF & 4180.4$\pm$113.8 & -16.8$\pm$3.6 & 168.6$\pm$5.1 & 168.6$\pm$5.1 & -0.2$\pm$0.3\\
& SPEDE & 4210.3$\pm$92.6 & -7.2$\pm$1.1 & 886.9$\pm$95.2 & 169.5$\pm$0.6 & 0.0$\pm$0.0\\
% & {\bf \ucblinear} & \textbf{4667.4$\pm$341.5} & -7.3$\pm$0.7 & 1165.6$\pm$144.7 & {167.8$\pm$1.3} & \textbf{0.0$\pm$0.0} \\
& {\bf \algabb-C} & \textbf{5557.6$\pm$439.5} & \textbf{-5.8$\pm$0.3} & {1086$\pm$278.2} & 167.1$\pm$3.1 & \textbf{0.0$\pm$0.0} \\
& {\bf \algabb-D} & \textbf{4616.5$\pm$261.5} & \textbf{-6.0$\pm$0.2} & 1359.2 $\pm$198.6  & 170.2 $\pm$ 4.2 & \textbf{0.0$\pm$0.0} \\
\bottomrule 
\end{tabular}
\centering
\begin{tabular}{p{2cm}p{2.2cm}p{2cm}p{1.8cm}p{1.8cm}p{1.8cm}p{1.8cm}p{1.8cm}}
% {lcccccccccccc}
& & Ant-ET & Hopper-ET & S-Humanoid-ET & CartPole & Walker-ET \\ 
\midrule  
\multirow{4}{*}{Model-Based RL} & ME-TRPO$^*$ & 42.6$\pm$21.1 & 1272.5$\pm$500.9 & -154.9$\pm$534.3 & 160.1$\pm$69.1 & -1609.3$\pm$657.5\\
& PETS-RS$^*$ & 130.0$\pm$148.1 &  205.8$\pm$36.5 & 320.7$\pm$182.2 & 195.0$\pm$28.0 & 312.5$\pm$493.4 \\
& PETS-CEM$^*$ & 81.6$\pm$145.8 & 129.3$\pm$36.0 & 355.1$\pm$157.1 & 195.5$\pm$3.0 & 260.2$\pm$536.9 \\
& Best MBBL & 275.4$\pm$309.1 & 1272.5$\pm$500.9 & \textbf{1084.3$\pm$77.0} & \textbf{200.0$\pm$0.0} & 312.5$\pm$493.4\\
\midrule
\multirow{3}{*}{Model-Free RL} & PPO$^*$ & 80.1$\pm$17.3  & 758.0$\pm$62.0 & 454.3$\pm$36.7 & 86.5$\pm$7.8 & 306.1$\pm$17.2\\
& TRPO$^*$ & 116.8$\pm$47.3  & 237.4$\pm$33.5 & 281.3$\pm$10.9 & 47.3$\pm$15.7 & 229.5$\pm$27.1\\
% TD3$^*$ & 259.7$\pm$1.0  & 1057.1$\pm$29.5 & 1070.0$\pm$168.3 & 147.7$\pm$0.7 & \textbf{3299.7$\pm$1951.5}\\
& SAC$^*$ (3-layer) & 2012.7$\pm$571.3  & 1815.5$\pm$655.1 & 834.6$\pm$313.1 & \textbf{199.4$\pm$0.4} & \textbf{2216.4$\pm$678.7}\\%(\chenjun{my code: 1190.8$\pm$ 570.5})\\
% SPEDE-REG \\
% {\bf \algabb-REG} & \textbf{2073.1$\pm$119.7} & \textbf{2510.3$\pm$550.8} & \textbf{2710.3$\pm$277.5} & \textbf{3747.8$\pm$1078.1} & 2170.3$\pm$810.9 \\
\midrule
\multirow{4}{*}{Representation RL} & DeepSF & 768.1$\pm$44.1  & 548.9$\pm$253.3 & 533.8$\pm$154.9 & 194.5$\pm$5.8 & 165.6$\pm$127.9\\
& SPEDE & 806.2$\pm$60.2  & 732.2$\pm$263.9 & 986.4$\pm$154.7 & 138.2$\pm$39.5 & 501.6$\pm$204.0\\
% & {\bf \ucblinear} & \textbf{3032.6$\pm$495.3} & 775.2$\pm$342.7 & \textbf{1178.8$\pm$65.4} & 180.4$\pm$26.8 & 1670.5$\pm$824.2  \\
& {\bf \algabb-C} & \textbf{2511.8$\pm$460.0} & \textbf{2204.8$\pm$496.0} & 963.1$\pm$45.1 & \textbf{200.7$\pm$0.2} & \textbf{2523.5$\pm$333.9} \\
& {{\bf \algabb-D}} & {\textbf{2436.0$\pm$603.1}} & \textbf{2187.5$\pm$453.6} & 956.8$\pm$ 87.5 & \textbf{198.5 $\pm$ 2.0} & \textbf{2209.0$\pm$589.2} \\
\bottomrule 
\end{tabular}
% \vspace{-2em}
\end{table*}

\vspace{-2mm}
\subsection{Implementation Details} 
\label{sec:exp-implementation}
\vspace{-2mm}
%The pseudocode of \algabb is presented in Algorithm~\ref{alg:online_algorithm-main}.  
As discussed, the latent variable representation is learned by minimizing the ELBO \eq{eq:ELBO}. 
We consider two practical implementations. The first one applies a continuous latent variable model, where the distributions are approximated using Gaussian with parameterized mean and variance, similarly to \citep{hafner2019learning}. We call this method \algabbc. 
The second implementation considers using a discrete sparse latent variable model \citep{hafner2019learning}, which we call \algabbd.  
As discussed in line 7 of Algorithm~\ref{alg:online_algorithm-main}, we apply a planning algorithm with the learned latent representation to improve the policy. 
We use \emph{Soft Actor Critic (SAC)} \citep{haarnoja2018soft} as our planner, 
where the critic is parameterized as shown in \eq{eq:exp_q}. 
In practice, 
we find that it is beneficial to have more updates for the latent variable model than critic.   
We also use a target network for the latent variable model to stabilize training.

\vspace{-2mm}
\subsection{Dense-Reward MuJoCo Benchmarks} 
\vspace{-2mm}

We first conduct experiments on dense-reward Mujoco locomotion control tasks, 
which are commonly used test domains for both model-free and model-based RL algorithms. 
We compare \algabb with model-based algorithms, including ME-TRPO \citep{kurutach2018model}, PETS \citep{chua2018deep}, and the best model-based results from~\citep{wang2019benchmarking}, among 9 baselines~\citep{luo2018algorithmic,deisenroth2011pilco,heess2015learning,clavera2018model, nagabandi2018neural,tassa2012synthesis,levine2014learning}, as well as model-free algorithms, including PPO \citep{schulman2017proximal}, TRPO \citep{schulman2015trust} and SAC \citep{haarnoja2018soft}. 

We compare all algorithms after running $200K$ environment steps. 
Table~\ref{tab:MuJoCo_results} presents all experiment results, where all results are averaged over 4 random seeds.   
In practice we found \algabbc provides comparable or better performance (see Figure~\ref{fig:mujoco} for example), so that we report its result for \algabb in the table. 
We present the best model-based RL performance for comparison.  
The results clearly show that \algabb provides significant better or comparable performance compared to  all model-based algorithms. 
In particular, 
in the most challenging domains such as Walker and Ant, 
most model-based methods completely fail the task, while \algabb achieves the state-of-the-art performance in contrast. 
Furthermore, 
\algabb show dominant performance in all domains
comparing to two representative representation learning based RL methods, Deep Successor Feature~(DeepSF)~\citep{barreto2017successor} and SPEDE~\citep{ren2022free}. 
\algabb also achieves better performance than the strongest model-free algorithm SAC in most challenging domains except Humanoid. 

Finally, we provide learning curves of \algabbc and \algabbd in comparison to SAC in Figure~\ref{fig:mujoco}, 
which clearly shows that comparing to the SOTA model-free baseline SAC, \algabb enjoys great sample efficiency in these tasks. 

% \vspace{2mm}
\begin{figure*}[thb]
%\centering
%\subfigure[cheetah_run]
{\includegraphics[width=3.3cm]{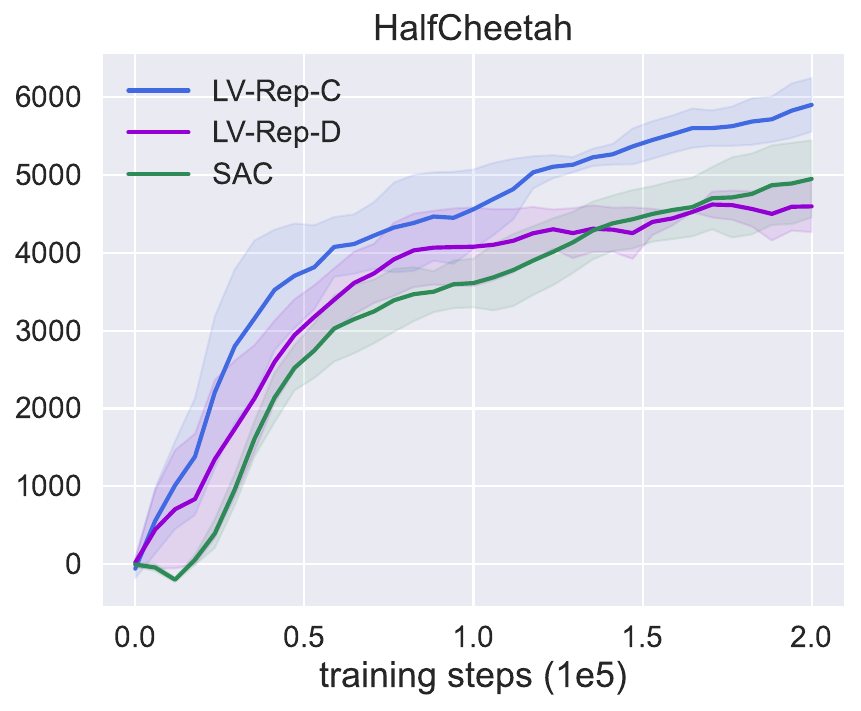}
\label{fig:gym-cheetah}
}
%\subfigure [walker_run]
{\includegraphics[width=3.3cm]{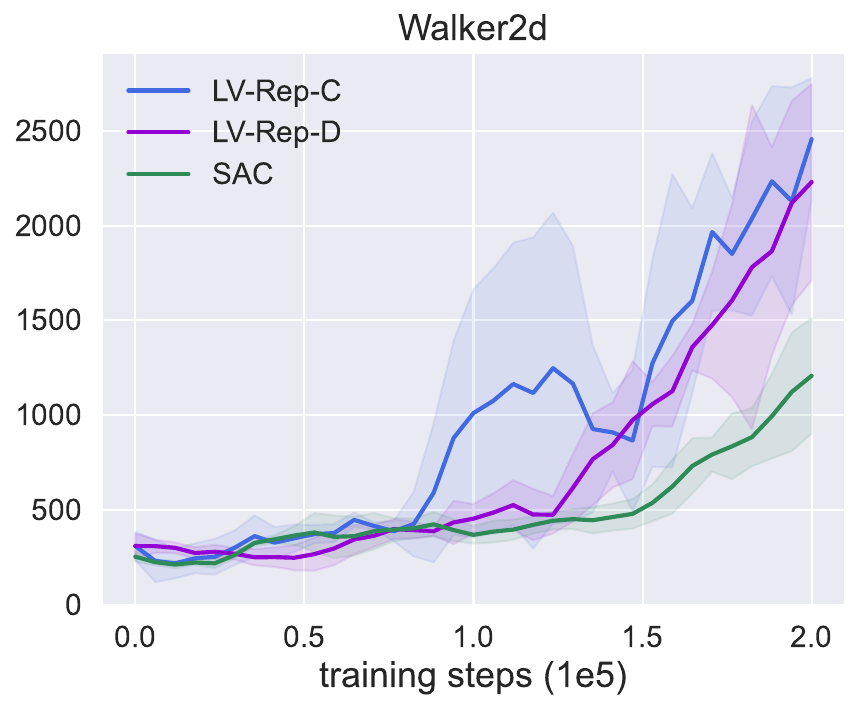}
\label{fig:gym-walker}
}
%\subfigure [humanoid_run]
{\includegraphics[width=3.3cm]{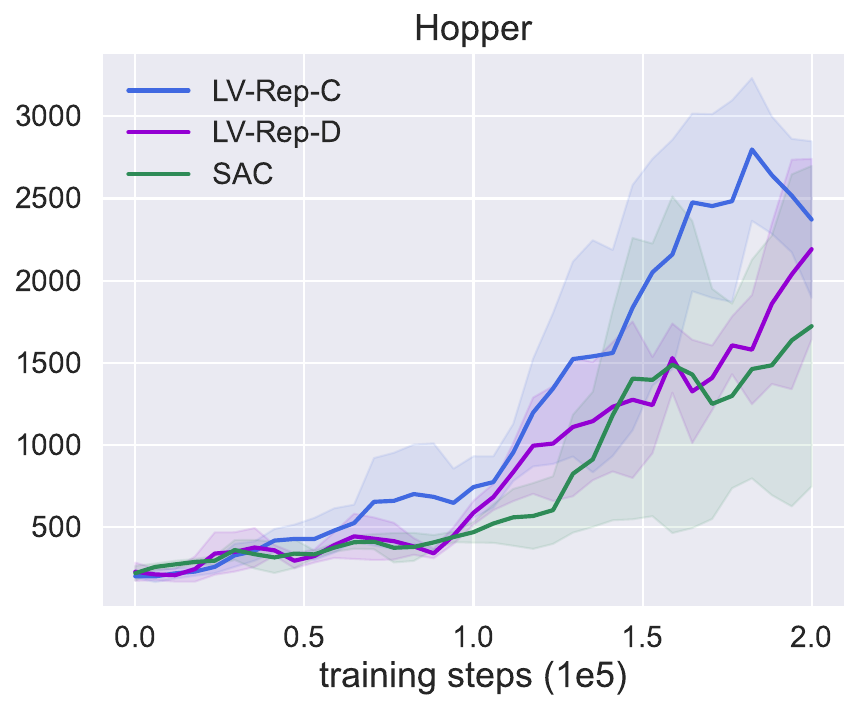}
\label{fig:gym-hopper}
}
%\subfigure [hopper_hop]
{\includegraphics[width=3.3cm]{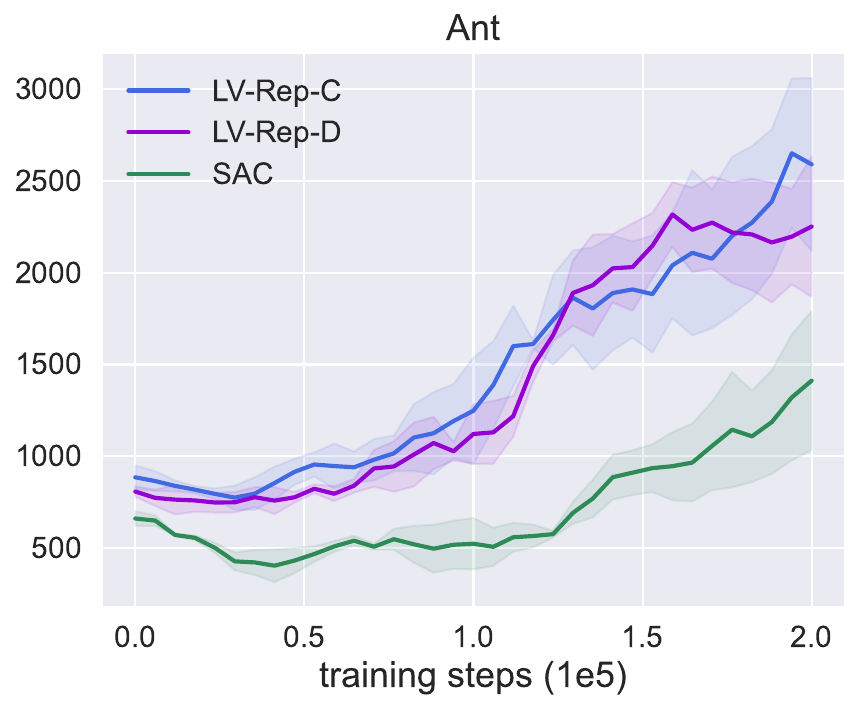}
\label{fig:gym-ant}
}
%\vspace{-2mm}
\caption{
We show the learning curves in Mujoco control compared to the baseline algorithms. 
The $x$-axis shows the training iterations and $y$-axis shows the performance. 
All plots are averaged over 4 random seeds. The shaded area shows the standard error.  
We only compare to SAC as it has the best overall performance in all baseline methods.}
%\vspace{-3mm}
\label{fig:mujoco}
\end{figure*}

% \vspace{-2mm}
\subsection{Sparse-Reward DeepMind Control Suite} 
% \vspace{-2mm}

In this experiment we show the effectiveness of  our proposed methods in sparse reward problems. 
We compare \algabb with the state-of-the-art model-free RL methods including SAC and PPO. Since the proposed~\algabb significantly dominates all the model-based RL algorithms in MuJoCo from~\citet{wang2019benchmarking}, we consider a different model-based RL method, \ie, Dreamer \citep{hafner2019learning}, and add another representation-based RL methods, \ie, Proto-RL~\citep{yarats2021reinforcement}, besides DeepSF~\citep{barreto2017successor}.

We compare all algorithms after running 200K environment steps across 4 random seeds. 
Results are presented in Table~\ref{tab:DM_results}. 
We report the result of \algabbc for \algabb as it gives better empirical performance. 
We can clearly observe that \algabb dominates the performance across all domains. 
In relatively dense-reward problems, \emph{cheetah-run} and \emph{walker-run}, \algabb outperforms all  baselines by a large margin. 
Remarkably, for sparse reward problems, \emph{hopper-hop} and \emph{humanoid-run}, 
\algabb provides reasonable results while other methods do not even start learning. 

\revise{
We also plot the learning curves of \algabb with all competitors
% compared to SAC, PPO and DeepSF 
in Figure~\ref{fig:dmc}. 
This shows that \algabb outperforms other baselines in terms of both sample efficiency and final performance. 
}

\begin{table*}[t]
% \vspace{-1.5em}
\caption{\footnotesize Performance of on various Deepmind Suite Control tasks. All the results are averaged across four random seeds and a window size of 10K. Comparing with SAC, our method achieves even better performance on sparse-reward tasks.
}
\scriptsize
\setlength\tabcolsep{3.5pt}
\label{tab:DM_results}
\centering
\begin{tabular}{p{2cm}p{1cm}p{2cm}p{2cm}p{2cm}p{2cm}}
% {lcccccccccccc}
\toprule
& & cheetah\_run &  walker\_run & humanoid\_run & hopper\_hop \\ 
\midrule
\multirow{1}{*}{Model-Based RL} & Dreamer & 
 542.0 $\pm$ 27.7 &  337.7$\pm$67.2 &  1.0$\pm$0.2 & 46.1$\pm$17.3\\
\midrule
\multirow{2}{*}{Model-Free RL} & PPO & 227.7$\pm$57.9 &  51.6$\pm$1.5 & 1.1$\pm$0.0 & 0.7$\pm$0.8\\
% TRPO$^*$ & 37.9$\pm$2.0 & -10.1$\pm$0.6 & -37.2$\pm$16.4 & 166.7$\pm$7.3 & -27.6$\pm$15.8 \\
& SAC  & 453.4$\pm$57.9 &  488.5$\pm$40.2 & 1.1$\pm$0.1 & 10.8$\pm$6.6\\
% {\bf \algabb-REG} & 40.0$\pm$3.8 & \textbf{-5.8$\pm$0.6} & 40.0$\pm$3.8 & \textbf{168.5$\pm$4.3} & \textbf{0.0$\pm$0.1}\\
\midrule  
% ME-TRPO$^*$ & 30.1$\pm$9.7 & -13.4$\pm$5.2 & -42.5$\pm$26.6 & \textbf{177.3$\pm$1.9} & -126.2$\pm$86.6\\
% PETS-RS$^*$  & 42.1$\pm$20.2 & -40.1$\pm$6.9 & -78.5$\pm$2.1 & 167.9$\pm$35.8 & -12.1$\pm$25.1\\
% PETS-CEM$^*$  & 22.1$\pm$25.2 & -12.3$\pm$5.2 & -57.9$\pm$3.6 & 167.4$\pm$53.0 & -20.5$\pm$28.9\\
\multirow{3}{*}{Representation RL} & DeepSF & 295.3$\pm$43.5 &  27.9$\pm$2.2 & 0.9$\pm$0.1 & 0.3$\pm$0.1 \\
% & SPEDE \\
& Proto RL &  305.5$\pm$37.9 &  433.5$\pm$56.8 &  0.3$\pm$0.6  & 1.0$\pm$0.2\\
& {\bf \algabb} & \textbf{639.3$\pm$24.5} &  \textbf{724.2$\pm$37.8} &  \textbf{11.8$\pm$6.8} & \textbf{72.9$\pm$40.6}\\
\bottomrule 
\end{tabular}
% \vspace{-1em}
\end{table*}

\begin{figure*}[thb]
%\centering
%\subfigure[cheetah_run]
% \vspace{2mm}
{\includegraphics[width=3.3cm]{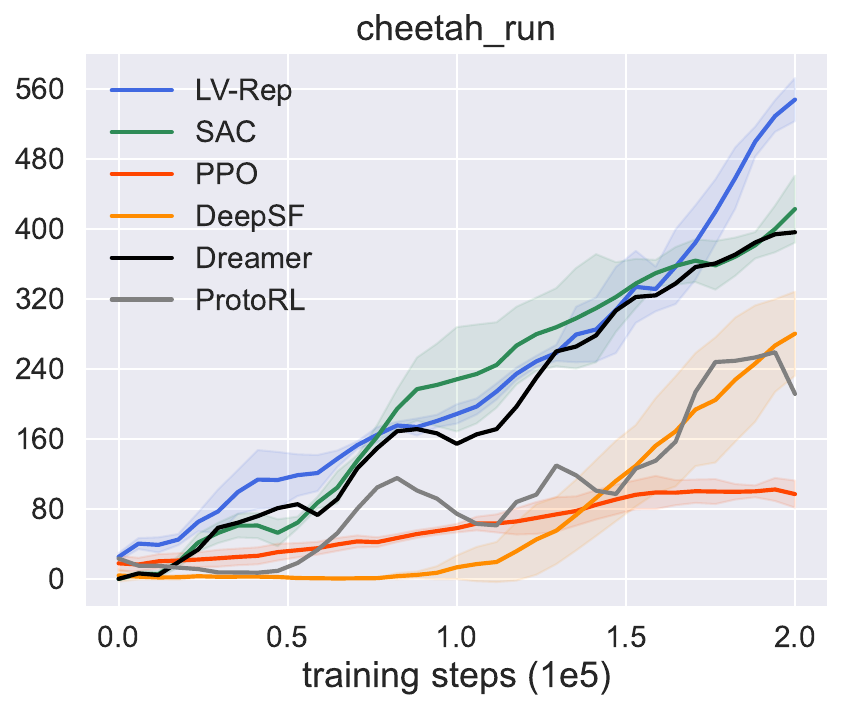}
\label{fig:dmc-cheetah}
}
%\subfigure [walker_run]
{\includegraphics[width=3.3cm]{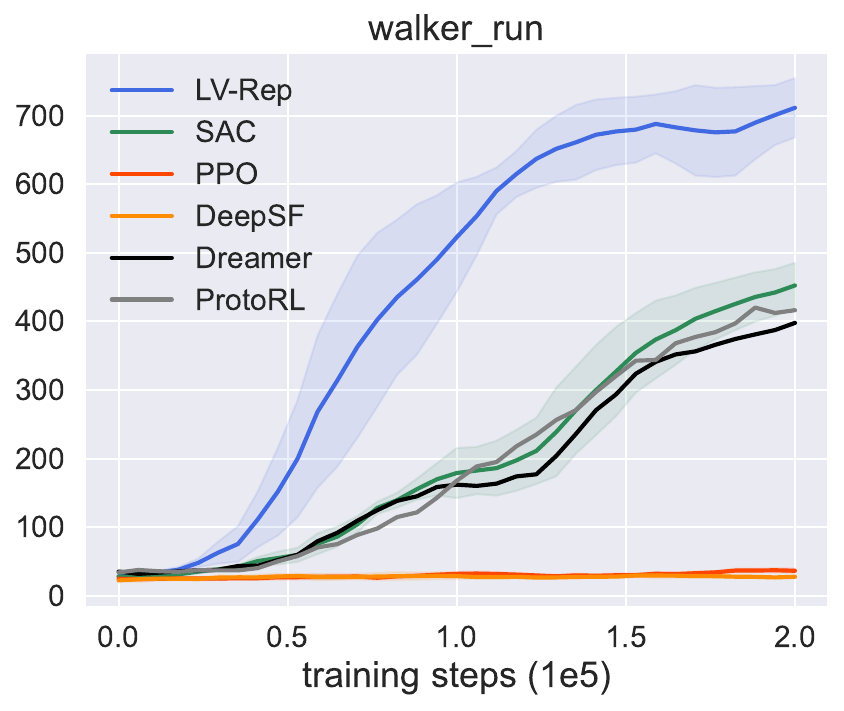}
\label{fig:dmc-walker}
}
%\subfigure [humanoid_run]
{\includegraphics[width=3.3cm]{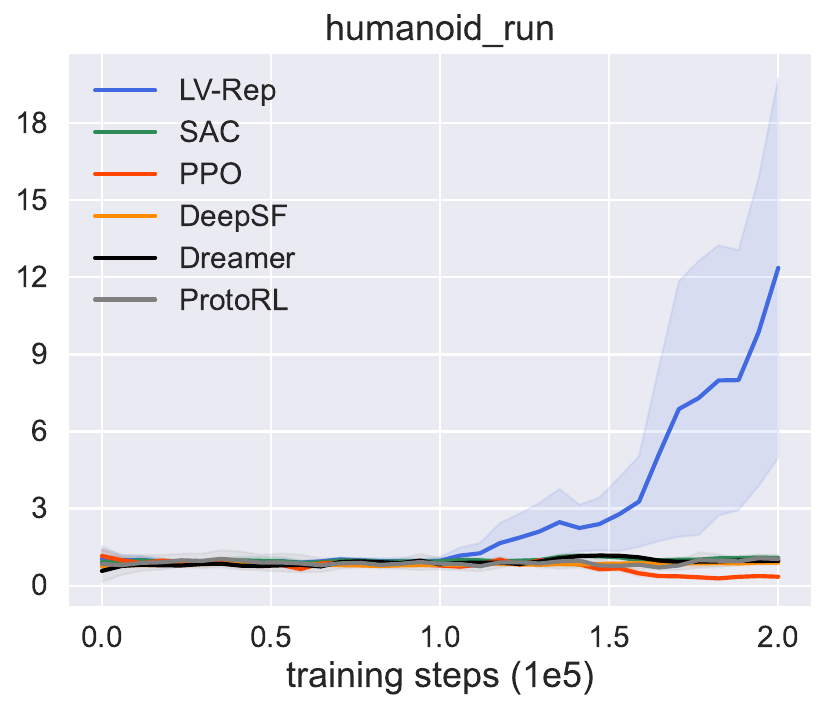}
\label{fig:dmc-humanoid}
}
%\subfigure [hopper_hop]
{\includegraphics[width=3.3cm]{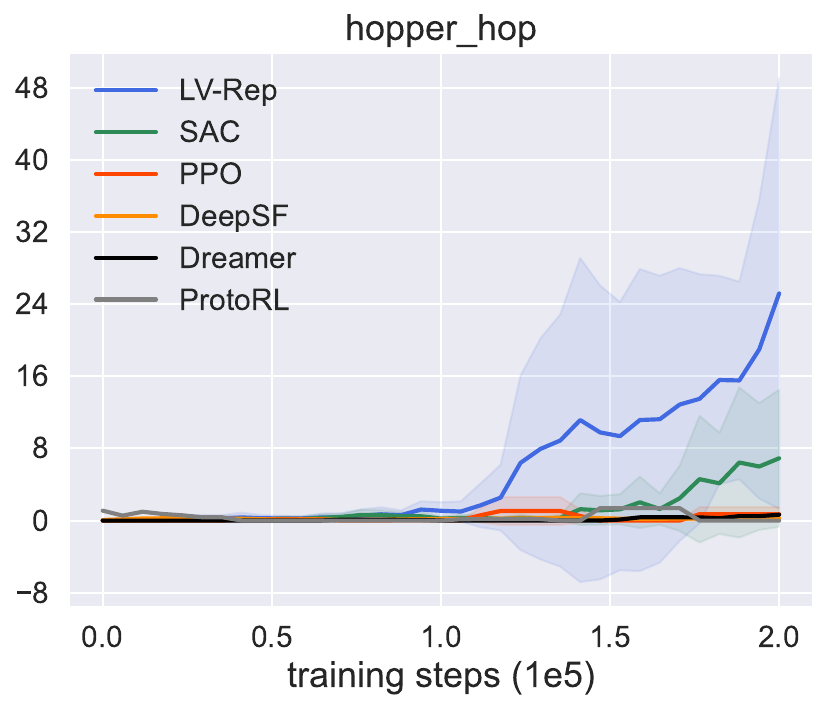}
\label{fig:dmc-hopper}
}
%\vspace{-2mm}
\caption{
\revise{We show the results in DeepMind Control Suite compared to the baseline algorithms. 
The $x$-axis shows the training iterations and $y$-axis shows the performance. 
All plots are averaged over 4 random seeds. The shaded area shows the standard error.}  
}
%\vspace{-3mm}
\label{fig:dmc}
\end{figure*}

\vspace{-2mm}
\section{Conclusion}
\vspace{-2mm}
% \Bo{TODO}
In this paper, we reveal the representation view of latent variable dynamics model, which induces the~\emph{\AlgName~(\algabb)}. Based on the~\algabb, a new provable and practical algorithm for reinforcement learning is proposed, achieving the balance between flexibility and efficiency in terms of statistical complexity, with tractable learning, planning and exploration. We provide rigorous theoretical analysis, which is applicable for~\algabb with both finite- and infinite-dimension and comprehensive empirical justification, which demonstrates the superior performances. 

\bibliography{ref}

\begin{thebibliography}{83}
\providecommand{\natexlab}[1]{#1}
\providecommand{\url}[1]{\texttt{#1}}
\expandafter\ifx\csname urlstyle\endcsname\relax
  \providecommand{\doi}[1]{doi: #1}\else
  \providecommand{\doi}{doi: \begingroup \urlstyle{rm}\Url}\fi

\bibitem[Agarwal et~al.(2020)Agarwal, Kakade, Krishnamurthy, and
  Sun]{agarwal2020flambe}
Alekh Agarwal, Sham Kakade, Akshay Krishnamurthy, and Wen Sun.
\newblock Flambe: Structural complexity and representation learning of low rank
  mdps.
\newblock \emph{Advances in neural information processing systems},
  33:\penalty0 20095--20107, 2020.

\bibitem[Aronszajn(1950)]{aronszajn1950theory}
Nachman Aronszajn.
\newblock Theory of reproducing kernels.
\newblock \emph{Transactions of the American mathematical society}, 68\penalty0
  (3):\penalty0 337--404, 1950.

\bibitem[Auer et~al.(2008)Auer, Jaksch, and Ortner]{auer2008near}
Peter Auer, Thomas Jaksch, and Ronald Ortner.
\newblock Near-optimal regret bounds for reinforcement learning.
\newblock \emph{Advances in neural information processing systems}, 21, 2008.

\bibitem[Azar et~al.(2017)Azar, Osband, and Munos]{azar2017minimax}
Mohammad~Gheshlaghi Azar, Ian Osband, and R{\'e}mi Munos.
\newblock Minimax regret bounds for reinforcement learning.
\newblock In \emph{International Conference on Machine Learning}, pp.\
  263--272. PMLR, 2017.

\bibitem[Barreto et~al.(2017)Barreto, Dabney, Munos, Hunt, Schaul, van Hasselt,
  and Silver]{barreto2017successor}
Andr{\'e} Barreto, Will Dabney, R{\'e}mi Munos, Jonathan~J Hunt, Tom Schaul,
  Hado~P van Hasselt, and David Silver.
\newblock Successor features for transfer in reinforcement learning.
\newblock \emph{Advances in neural information processing systems}, 30, 2017.

\bibitem[Bartlett \& Mendelson(2002)Bartlett and
  Mendelson]{bartlett2002rademacher}
Peter~L Bartlett and Shahar Mendelson.
\newblock Rademacher and gaussian complexities: Risk bounds and structural
  results.
\newblock \emph{Journal of Machine Learning Research}, 3\penalty0
  (Nov):\penalty0 463--482, 2002.

\bibitem[Boyan \& Moore(1994)Boyan and Moore]{boyan1994generalization}
Justin Boyan and Andrew Moore.
\newblock Generalization in reinforcement learning: Safely approximating the
  value function.
\newblock \emph{Advances in neural information processing systems}, 7, 1994.

\bibitem[Carl \& Stephani(1990)Carl and Stephani]{carl_stephani_1990}
Bernd Carl and Irmtraud Stephani.
\newblock \emph{Entropy, Compactness and the Approximation of Operators}.
\newblock Cambridge Tracts in Mathematics. Cambridge University Press, 1990.
\newblock \doi{10.1017/CBO9780511897467}.

\bibitem[Chang et~al.(2021)Chang, Uehara, Sreenivas, Kidambi, and
  Sun]{chang2021mitigating}
Jonathan Chang, Masatoshi Uehara, Dhruv Sreenivas, Rahul Kidambi, and Wen Sun.
\newblock Mitigating covariate shift in imitation learning via offline data
  with partial coverage.
\newblock \emph{Advances in Neural Information Processing Systems},
  34:\penalty0 965--979, 2021.

\bibitem[Choromanski et~al.(2018)Choromanski, Rowland, Sarl{\'o}s, Sindhwani,
  Turner, and Weller]{choromanski2018geometry}
Krzysztof Choromanski, Mark Rowland, Tam{\'a}s Sarl{\'o}s, Vikas Sindhwani,
  Richard Turner, and Adrian Weller.
\newblock The geometry of random features.
\newblock In \emph{International Conference on Artificial Intelligence and
  Statistics}, pp.\  1--9. PMLR, 2018.

\bibitem[Chua et~al.(2018)Chua, Calandra, McAllister, and Levine]{chua2018deep}
Kurtland Chua, Roberto Calandra, Rowan McAllister, and Sergey Levine.
\newblock Deep reinforcement learning in a handful of trials using
  probabilistic dynamics models.
\newblock \emph{Advances in neural information processing systems}, 31, 2018.

\bibitem[Clavera et~al.(2018)Clavera, Rothfuss, Schulman, Fujita, Asfour, and
  Abbeel]{clavera2018model}
Ignasi Clavera, Jonas Rothfuss, John Schulman, Yasuhiro Fujita, Tamim Asfour,
  and Pieter Abbeel.
\newblock Model-based reinforcement learning via meta-policy optimization.
\newblock In \emph{Conference on Robot Learning}, pp.\  617--629. PMLR, 2018.

\bibitem[Dai et~al.(2014)Dai, Xie, He, Liang, Raj, Balcan, and
  Song]{dai2014scalable}
Bo~Dai, Bo~Xie, Niao He, Yingyu Liang, Anant Raj, Maria-Florina~F Balcan, and
  Le~Song.
\newblock Scalable kernel methods via doubly stochastic gradients.
\newblock \emph{Advances in neural information processing systems}, 27, 2014.

\bibitem[Dann \& Brunskill(2015)Dann and Brunskill]{dann2015sample}
Christoph Dann and Emma Brunskill.
\newblock Sample complexity of episodic fixed-horizon reinforcement learning.
\newblock \emph{Advances in Neural Information Processing Systems}, 28, 2015.

\bibitem[Dayan(1993)]{dayan1993improving}
Peter Dayan.
\newblock Improving generalization for temporal difference learning: The
  successor representation.
\newblock \emph{Neural computation}, 5\penalty0 (4):\penalty0 613--624, 1993.

\bibitem[Deisenroth \& Rasmussen(2011)Deisenroth and
  Rasmussen]{deisenroth2011pilco}
Marc Deisenroth and Carl~E Rasmussen.
\newblock Pilco: A model-based and data-efficient approach to policy search.
\newblock In \emph{Proceedings of the 28th International Conference on machine
  learning (ICML-11)}, pp.\  465--472. Citeseer, 2011.

\bibitem[Du et~al.(2019)Du, Krishnamurthy, Jiang, Agarwal, Dudik, and
  Langford]{du2019provably}
Simon Du, Akshay Krishnamurthy, Nan Jiang, Alekh Agarwal, Miroslav Dudik, and
  John Langford.
\newblock Provably efficient rl with rich observations via latent state
  decoding.
\newblock In \emph{International Conference on Machine Learning}, pp.\
  1665--1674. PMLR, 2019.

\bibitem[Duan et~al.(2019)Duan, Ke, and Wang]{duan2019state}
Yaqi Duan, Tracy Ke, and Mengdi Wang.
\newblock State aggregation learning from markov transition data.
\newblock \emph{Advances in Neural Information Processing Systems}, 32, 2019.

\bibitem[Ferns et~al.(2004)Ferns, Panangaden, and Precup]{ferns2004metrics}
Norm Ferns, Prakash Panangaden, and Doina Precup.
\newblock Metrics for finite markov decision processes.
\newblock In \emph{UAI}, volume~4, pp.\  162--169, 2004.

\bibitem[Foster et~al.(2021)Foster, Kakade, Qian, and
  Rakhlin]{foster2021statistical}
Dylan~J Foster, Sham~M Kakade, Jian Qian, and Alexander Rakhlin.
\newblock The statistical complexity of interactive decision making.
\newblock \emph{arXiv preprint arXiv:2112.13487}, 2021.

\bibitem[Gelada et~al.(2019)Gelada, Kumar, Buckman, Nachum, and
  Bellemare]{gelada2019deepmdp}
Carles Gelada, Saurabh Kumar, Jacob Buckman, Ofir Nachum, and Marc~G Bellemare.
\newblock Deepmdp: Learning continuous latent space models for representation
  learning.
\newblock In \emph{International Conference on Machine Learning}, pp.\
  2170--2179. PMLR, 2019.

\bibitem[Haarnoja et~al.(2018)Haarnoja, Zhou, Abbeel, and
  Levine]{haarnoja2018soft}
Tuomas Haarnoja, Aurick Zhou, Pieter Abbeel, and Sergey Levine.
\newblock Soft actor-critic: Off-policy maximum entropy deep reinforcement
  learning with a stochastic actor.
\newblock In \emph{International conference on machine learning}, pp.\
  1861--1870. PMLR, 2018.

\bibitem[Hafner et~al.(2019{\natexlab{a}})Hafner, Lillicrap, Ba, and
  Norouzi]{hafner2019dream}
Danijar Hafner, Timothy Lillicrap, Jimmy Ba, and Mohammad Norouzi.
\newblock Dream to control: Learning behaviors by latent imagination.
\newblock In \emph{International Conference on Learning Representations},
  2019{\natexlab{a}}.

\bibitem[Hafner et~al.(2019{\natexlab{b}})Hafner, Lillicrap, Fischer, Villegas,
  Ha, Lee, and Davidson]{hafner2019learning}
Danijar Hafner, Timothy Lillicrap, Ian Fischer, Ruben Villegas, David Ha,
  Honglak Lee, and James Davidson.
\newblock Learning latent dynamics for planning from pixels.
\newblock In \emph{International conference on machine learning}, pp.\
  2555--2565. PMLR, 2019{\natexlab{b}}.

\bibitem[Hafner et~al.(2020)Hafner, Lillicrap, Norouzi, and
  Ba]{hafner2020mastering}
Danijar Hafner, Timothy~P Lillicrap, Mohammad Norouzi, and Jimmy Ba.
\newblock Mastering atari with discrete world models.
\newblock In \emph{International Conference on Learning Representations}, 2020.

\bibitem[Heess et~al.(2015)Heess, Wayne, Silver, Lillicrap, Erez, and
  Tassa]{heess2015learning}
Nicolas Heess, Gregory Wayne, David Silver, Timothy Lillicrap, Tom Erez, and
  Yuval Tassa.
\newblock Learning continuous control policies by stochastic value gradients.
\newblock \emph{Advances in neural information processing systems}, 28, 2015.

\bibitem[Jiang et~al.(2021)Jiang, Dai, Yang, Zhao, and Wei]{jiang2021towards}
Haoming Jiang, Bo~Dai, Mengjiao Yang, Tuo Zhao, and Wei Wei.
\newblock Towards automatic evaluation of dialog systems: A model-free
  off-policy evaluation approach.
\newblock \emph{arXiv preprint arXiv:2102.10242}, 2021.

\bibitem[Jiang et~al.(2017)Jiang, Krishnamurthy, Agarwal, Langford, and
  Schapire]{jiang2017contextual}
Nan Jiang, Akshay Krishnamurthy, Alekh Agarwal, John Langford, and Robert~E
  Schapire.
\newblock Contextual decision processes with low bellman rank are
  pac-learnable.
\newblock In \emph{International Conference on Machine Learning}, pp.\
  1704--1713. PMLR, 2017.

\bibitem[Jin et~al.(2018)Jin, Allen-Zhu, Bubeck, and Jordan]{jin2018q}
Chi Jin, Zeyuan Allen-Zhu, Sebastien Bubeck, and Michael~I Jordan.
\newblock Is q-learning provably efficient?
\newblock \emph{Advances in neural information processing systems}, 31, 2018.

\bibitem[Jin et~al.(2020)Jin, Yang, Wang, and Jordan]{jin2020provably}
Chi Jin, Zhuoran Yang, Zhaoran Wang, and Michael~I Jordan.
\newblock Provably efficient reinforcement learning with linear function
  approximation.
\newblock In \emph{Conference on Learning Theory}, pp.\  2137--2143. PMLR,
  2020.

\bibitem[Kingma et~al.(2019)Kingma, Welling, et~al.]{kingma2019introduction}
Diederik~P Kingma, Max Welling, et~al.
\newblock An introduction to variational autoencoders.
\newblock \emph{Foundations and Trends{\textregistered} in Machine Learning},
  12\penalty0 (4):\penalty0 307--392, 2019.

\bibitem[Kong et~al.(2021)Kong, Salakhutdinov, Wang, and Yang]{kong2021online}
Dingwen Kong, Ruslan Salakhutdinov, Ruosong Wang, and Lin~F Yang.
\newblock Online sub-sampling for reinforcement learning with general function
  approximation.
\newblock \emph{arXiv preprint arXiv:2106.07203}, 2021.

\bibitem[Kostrikov et~al.(2020)Kostrikov, Yarats, and
  Fergus]{kostrikov2020image}
Ilya Kostrikov, Denis Yarats, and Rob Fergus.
\newblock Image augmentation is all you need: Regularizing deep reinforcement
  learning from pixels.
\newblock \emph{arXiv preprint arXiv:2004.13649}, 2020.

\bibitem[Kulkarni et~al.(2016)Kulkarni, Saeedi, Gautam, and
  Gershman]{kulkarni2016deep}
Tejas~D Kulkarni, Ardavan Saeedi, Simanta Gautam, and Samuel~J Gershman.
\newblock Deep successor reinforcement learning.
\newblock \emph{arXiv preprint arXiv:1606.02396}, 2016.

\bibitem[Kurutach et~al.(2018)Kurutach, Clavera, Duan, Tamar, and
  Abbeel]{kurutach2018model}
Thanard Kurutach, Ignasi Clavera, Yan Duan, Aviv Tamar, and Pieter Abbeel.
\newblock Model-ensemble trust-region policy optimization.
\newblock In \emph{International Conference on Learning Representations}, 2018.

\bibitem[Laskin et~al.(2020{\natexlab{a}})Laskin, Srinivas, and
  Abbeel]{laskin2020curl}
Michael Laskin, Aravind Srinivas, and Pieter Abbeel.
\newblock Curl: Contrastive unsupervised representations for reinforcement
  learning.
\newblock In \emph{International Conference on Machine Learning}, pp.\
  5639--5650. PMLR, 2020{\natexlab{a}}.

\bibitem[Laskin et~al.(2020{\natexlab{b}})Laskin, Lee, Stooke, Pinto, Abbeel,
  and Srinivas]{laskin2020reinforcement}
Misha Laskin, Kimin Lee, Adam Stooke, Lerrel Pinto, Pieter Abbeel, and Aravind
  Srinivas.
\newblock Reinforcement learning with augmented data.
\newblock \emph{Advances in neural information processing systems},
  33:\penalty0 19884--19895, 2020{\natexlab{b}}.

\bibitem[Lee et~al.(2020)Lee, Nagabandi, Abbeel, and Levine]{lee2020stochastic}
Alex~X Lee, Anusha Nagabandi, Pieter Abbeel, and Sergey Levine.
\newblock Stochastic latent actor-critic: Deep reinforcement learning with a
  latent variable model.
\newblock \emph{Advances in Neural Information Processing Systems},
  33:\penalty0 741--752, 2020.

\bibitem[Levine \& Abbeel(2014)Levine and Abbeel]{levine2014learning}
Sergey Levine and Pieter Abbeel.
\newblock Learning neural network policies with guided policy search under
  unknown dynamics.
\newblock \emph{Advances in neural information processing systems}, 27, 2014.

\bibitem[Luo et~al.(2018)Luo, Xu, Li, Tian, Darrell, and
  Ma]{luo2018algorithmic}
Yuping Luo, Huazhe Xu, Yuanzhi Li, Yuandong Tian, Trevor Darrell, and Tengyu
  Ma.
\newblock Algorithmic framework for model-based deep reinforcement learning
  with theoretical guarantees.
\newblock \emph{arXiv preprint arXiv:1807.03858}, 2018.

\bibitem[Mahadevan \& Maggioni(2007)Mahadevan and Maggioni]{mahadevan2007proto}
Sridhar Mahadevan and Mauro Maggioni.
\newblock Proto-value functions: A laplacian framework for learning
  representation and control in markov decision processes.
\newblock \emph{Journal of Machine Learning Research}, 8\penalty0 (10), 2007.

\bibitem[Minsker(2017)]{minsker2017some}
Stanislav Minsker.
\newblock On some extensions of bernstein’s inequality for self-adjoint
  operators.
\newblock \emph{Statistics \& Probability Letters}, 127:\penalty0 111--119,
  2017.

\bibitem[Misra et~al.(2020)Misra, Henaff, Krishnamurthy, and
  Langford]{misra2020kinematic}
Dipendra Misra, Mikael Henaff, Akshay Krishnamurthy, and John Langford.
\newblock Kinematic state abstraction and provably efficient rich-observation
  reinforcement learning.
\newblock In \emph{International conference on machine learning}, pp.\
  6961--6971. PMLR, 2020.

\bibitem[Mnih et~al.(2013)Mnih, Kavukcuoglu, Silver, Graves, Antonoglou,
  Wierstra, and Riedmiller]{mnih2013playing}
Volodymyr Mnih, Koray Kavukcuoglu, David Silver, Alex Graves, Ioannis
  Antonoglou, Daan Wierstra, and Martin Riedmiller.
\newblock Playing atari with deep reinforcement learning.
\newblock \emph{arXiv preprint arXiv:1312.5602}, 2013.

\bibitem[Modi et~al.(2021)Modi, Chen, Krishnamurthy, Jiang, and
  Agarwal]{modi2021model}
Aditya Modi, Jinglin Chen, Akshay Krishnamurthy, Nan Jiang, and Alekh Agarwal.
\newblock Model-free representation learning and exploration in low-rank mdps.
\newblock \emph{arXiv preprint arXiv:2102.07035}, 2021.

\bibitem[Munos \& Szepesv{\'a}ri(2008)Munos and
  Szepesv{\'a}ri]{munos2008finite}
R{\'e}mi Munos and Csaba Szepesv{\'a}ri.
\newblock Finite-time bounds for fitted value iteration.
\newblock \emph{Journal of Machine Learning Research}, 9\penalty0 (5), 2008.

\bibitem[Nachum \& Yang(2021)Nachum and Yang]{nachum2021provable}
Ofir Nachum and Mengjiao Yang.
\newblock Provable representation learning for imitation with contrastive
  fourier features.
\newblock \emph{Advances in Neural Information Processing Systems},
  34:\penalty0 30100--30112, 2021.

\bibitem[Nagabandi et~al.(2018)Nagabandi, Kahn, Fearing, and
  Levine]{nagabandi2018neural}
Anusha Nagabandi, Gregory Kahn, Ronald~S Fearing, and Sergey Levine.
\newblock Neural network dynamics for model-based deep reinforcement learning
  with model-free fine-tuning.
\newblock In \emph{2018 IEEE International Conference on Robotics and
  Automation (ICRA)}, pp.\  7559--7566. IEEE, 2018.

\bibitem[Oord et~al.(2018)Oord, Li, and Vinyals]{oord2018representation}
Aaron van~den Oord, Yazhe Li, and Oriol Vinyals.
\newblock Representation learning with contrastive predictive coding.
\newblock \emph{arXiv preprint arXiv:1807.03748}, 2018.

\bibitem[Osband \& Van~Roy(2014)Osband and Van~Roy]{osband2014model}
Ian Osband and Benjamin Van~Roy.
\newblock Model-based reinforcement learning and the eluder dimension.
\newblock \emph{Advances in Neural Information Processing Systems}, 27, 2014.

\bibitem[Paulsen \& Raghupathi(2016)Paulsen and
  Raghupathi]{paulsen2016introduction}
Vern~I Paulsen and Mrinal Raghupathi.
\newblock \emph{An introduction to the theory of reproducing kernel Hilbert
  spaces}, volume 152.
\newblock Cambridge university press, 2016.

\bibitem[Peng et~al.(2018)Peng, Andrychowicz, Zaremba, and Abbeel]{peng2018sim}
Xue~Bin Peng, Marcin Andrychowicz, Wojciech Zaremba, and Pieter Abbeel.
\newblock Sim-to-real transfer of robotic control with dynamics randomization.
\newblock In \emph{2018 IEEE international conference on robotics and
  automation (ICRA)}, pp.\  3803--3810. IEEE, 2018.

\bibitem[Puterman(2014)]{puterman2014markov}
Martin~L Puterman.
\newblock \emph{Markov decision processes: discrete stochastic dynamic
  programming}.
\newblock John Wiley \& Sons, 2014.

\bibitem[Qiu et~al.(2022)Qiu, Wang, Bai, Yang, and Wang]{qiu2022contrastive}
Shuang Qiu, Lingxiao Wang, Chenjia Bai, Zhuoran Yang, and Zhaoran Wang.
\newblock Contrastive ucb: Provably efficient contrastive self-supervised
  learning in online reinforcement learning.
\newblock In \emph{International Conference on Machine Learning}, pp.\
  18168--18210. PMLR, 2022.

\bibitem[Rahimi \& Recht(2007)Rahimi and Recht]{rahimi2007random}
Ali Rahimi and Benjamin Recht.
\newblock Random features for large-scale kernel machines.
\newblock \emph{Advances in neural information processing systems}, 20, 2007.

\bibitem[Ren et~al.(2022{\natexlab{a}})Ren, Zhang, Lee, Gonzalez, Schuurmans,
  and Dai]{ren2022spectral}
Tongzheng Ren, Tianjun Zhang, Lisa Lee, Joseph~E Gonzalez, Dale Schuurmans, and
  Bo~Dai.
\newblock Spectral decomposition representation for reinforcement learning.
\newblock \emph{arXiv preprint arXiv:2208.09515}, 2022{\natexlab{a}}.

\bibitem[Ren et~al.(2022{\natexlab{b}})Ren, Zhang, Szepesv{\'a}ri, and
  Dai]{ren2022free}
Tongzheng Ren, Tianjun Zhang, Csaba Szepesv{\'a}ri, and Bo~Dai.
\newblock A free lunch from the noise: Provable and practical exploration for
  representation learning.
\newblock In \emph{Uncertainty in Artificial Intelligence}, pp.\  1686--1696.
  PMLR, 2022{\natexlab{b}}.

\bibitem[Riesz \& Nagy(2012)Riesz and Nagy]{riesz2012functional}
Frigyes Riesz and B{\'e}la~Sz Nagy.
\newblock \emph{Functional analysis}.
\newblock Courier Corporation, 2012.

\bibitem[Schulman et~al.(2015)Schulman, Levine, Abbeel, Jordan, and
  Moritz]{schulman2015trust}
John Schulman, Sergey Levine, Pieter Abbeel, Michael Jordan, and Philipp
  Moritz.
\newblock Trust region policy optimization.
\newblock In \emph{International conference on machine learning}, pp.\
  1889--1897. PMLR, 2015.

\bibitem[Schulman et~al.(2017)Schulman, Wolski, Dhariwal, Radford, and
  Klimov]{schulman2017proximal}
John Schulman, Filip Wolski, Prafulla Dhariwal, Alec Radford, and Oleg Klimov.
\newblock Proximal policy optimization algorithms.
\newblock \emph{arXiv preprint arXiv:1707.06347}, 2017.

\bibitem[Seeger et~al.(2008)Seeger, Kakade, and Foster]{seeger2008information}
Matthias~W Seeger, Sham~M Kakade, and Dean~P Foster.
\newblock Information consistency of nonparametric gaussian process methods.
\newblock \emph{IEEE Transactions on Information Theory}, 54\penalty0
  (5):\penalty0 2376--2382, 2008.

\bibitem[Srinivas et~al.(2010)Srinivas, Krause, Kakade, and
  Seeger]{srinivas2010gaussian}
Niranjan Srinivas, Andreas Krause, Sham Kakade, and Matthias Seeger.
\newblock Gaussian process optimization in the bandit setting: no regret and
  experimental design.
\newblock In \emph{Proceedings of the 27th International Conference on
  International Conference on Machine Learning}, pp.\  1015--1022, 2010.

\bibitem[Steinwart \& Christmann(2008)Steinwart and
  Christmann]{steinwart2008support}
Ingo Steinwart and Andreas Christmann.
\newblock \emph{Support vector machines}.
\newblock Springer Science \& Business Media, 2008.

\bibitem[Steinwart \& Scovel(2012)Steinwart and Scovel]{steinwart2012mercer}
Ingo Steinwart and Clint Scovel.
\newblock Mercer’s theorem on general domains: On the interaction between
  measures, kernels, and rkhss.
\newblock \emph{Constructive Approximation}, 35\penalty0 (3):\penalty0
  363--417, 2012.

\bibitem[Steinwart et~al.(2009)Steinwart, Hush, Scovel,
  et~al.]{steinwart2009optimal}
Ingo Steinwart, Don~R Hush, Clint Scovel, et~al.
\newblock Optimal rates for regularized least squares regression.
\newblock In \emph{COLT}, pp.\  79--93, 2009.

\bibitem[Tassa et~al.(2012)Tassa, Erez, and Todorov]{tassa2012synthesis}
Yuval Tassa, Tom Erez, and Emanuel Todorov.
\newblock Synthesis and stabilization of complex behaviors through online
  trajectory optimization.
\newblock In \emph{2012 IEEE/RSJ International Conference on Intelligent Robots
  and Systems}, pp.\  4906--4913. IEEE, 2012.

\bibitem[Tassa et~al.(2018)Tassa, Doron, Muldal, Erez, Li, Casas, Budden,
  Abdolmaleki, Merel, Lefrancq, et~al.]{tassa2018deepmind}
Yuval Tassa, Yotam Doron, Alistair Muldal, Tom Erez, Yazhe Li, Diego de~Las
  Casas, David Budden, Abbas Abdolmaleki, Josh Merel, Andrew Lefrancq, et~al.
\newblock Deepmind control suite.
\newblock \emph{arXiv preprint arXiv:1801.00690}, 2018.

\bibitem[Todorov et~al.(2012)Todorov, Erez, and Tassa]{todorov2012mujoco}
Emanuel Todorov, Tom Erez, and Yuval Tassa.
\newblock Mujoco: A physics engine for model-based control.
\newblock In \emph{2012 IEEE/RSJ international conference on intelligent robots
  and systems}, pp.\  5026--5033. IEEE, 2012.

\bibitem[Tsitsiklis \& Van~Roy(1996)Tsitsiklis and
  Van~Roy]{tsitsiklis1996analysis}
John Tsitsiklis and Benjamin Van~Roy.
\newblock Analysis of temporal-diffference learning with function
  approximation.
\newblock \emph{Advances in neural information processing systems}, 9, 1996.

\bibitem[Uehara et~al.(2022)Uehara, Zhang, and Sun]{uehara2021representation}
Masatoshi Uehara, Xuezhou Zhang, and Wen Sun.
\newblock Representation learning for online and offline rl in low-rank mdps.
\newblock In \emph{International Conference on Learning Representations}, 2022.

\bibitem[Valko et~al.(2013)Valko, Korda, Munos, Flaounas, and
  Cristianini]{valko2013finite}
Michal Valko, Nathan Korda, R{\'e}mi Munos, Ilias Flaounas, and Nello
  Cristianini.
\newblock Finite-time analysis of kernelised contextual bandits.
\newblock In \emph{Uncertainty in Artificial Intelligence}, pp.\  654.
  Citeseer, 2013.

\bibitem[Wang et~al.(2020)Wang, Salakhutdinov, and Yang]{wang2020reinforcement}
Ruosong Wang, Russ~R Salakhutdinov, and Lin Yang.
\newblock Reinforcement learning with general value function approximation:
  Provably efficient approach via bounded eluder dimension.
\newblock \emph{Advances in Neural Information Processing Systems},
  33:\penalty0 6123--6135, 2020.

\bibitem[Wang et~al.(2019)Wang, Bao, Clavera, Hoang, Wen, Langlois, Zhang,
  Zhang, Abbeel, and Ba]{wang2019benchmarking}
Tingwu Wang, Xuchan Bao, Ignasi Clavera, Jerrick Hoang, Yeming Wen, Eric
  Langlois, Shunshi Zhang, Guodong Zhang, Pieter Abbeel, and Jimmy Ba.
\newblock Benchmarking model-based reinforcement learning.
\newblock \emph{arXiv preprint arXiv:1907.02057}, 2019.

\bibitem[Wu et~al.(2022)Wu, Escontrela, Hafner, Goldberg, and
  Abbeel]{wu2022daydreamer}
Philipp Wu, Alejandro Escontrela, Danijar Hafner, Ken Goldberg, and Pieter
  Abbeel.
\newblock Daydreamer: World models for physical robot learning.
\newblock \emph{arXiv preprint arXiv:2206.14176}, 2022.

\bibitem[Wu et~al.(2018)Wu, Tucker, and Nachum]{wu2018laplacian}
Yifan Wu, George Tucker, and Ofir Nachum.
\newblock The laplacian in rl: Learning representations with efficient
  approximations.
\newblock \emph{arXiv preprint arXiv:1810.04586}, 2018.

\bibitem[Yang \& Wang(2020)Yang and Wang]{yang2020reinforcement}
Lin Yang and Mengdi Wang.
\newblock Reinforcement learning in feature space: Matrix bandit, kernels, and
  regret bound.
\newblock In \emph{International Conference on Machine Learning}, pp.\
  10746--10756. PMLR, 2020.

\bibitem[Yang et~al.(2020)Yang, Jin, Wang, Wang, and Jordan]{yang2020provably}
Zhuoran Yang, Chi Jin, Zhaoran Wang, Mengdi Wang, and Michael Jordan.
\newblock Provably efficient reinforcement learning with kernel and neural
  function approximations.
\newblock \emph{Advances in Neural Information Processing Systems},
  33:\penalty0 13903--13916, 2020.

\bibitem[Yarats \& Kostrikov(2020)Yarats and Kostrikov]{pytorch_sac}
Denis Yarats and Ilya Kostrikov.
\newblock Soft actor-critic (sac) implementation in pytorch.
\newblock \url{https://github.com/denisyarats/pytorch_sac}, 2020.

\bibitem[Yarats et~al.(2021)Yarats, Fergus, Lazaric, and
  Pinto]{yarats2021reinforcement}
Denis Yarats, Rob Fergus, Alessandro Lazaric, and Lerrel Pinto.
\newblock Reinforcement learning with prototypical representations.
\newblock In \emph{International Conference on Machine Learning}, pp.\
  11920--11931. PMLR, 2021.

\bibitem[Zanette et~al.(2021)Zanette, Cheng, and
  Agarwal]{zanette2021cautiously}
Andrea Zanette, Ching-An Cheng, and Alekh Agarwal.
\newblock Cautiously optimistic policy optimization and exploration with linear
  function approximation.
\newblock In \emph{Conference on Learning Theory}, pp.\  4473--4525. PMLR,
  2021.

\bibitem[Zhang et~al.(2020)Zhang, McAllister, Calandra, Gal, and
  Levine]{zhang2020learning}
Amy Zhang, Rowan McAllister, Roberto Calandra, Yarin Gal, and Sergey Levine.
\newblock Learning invariant representations for reinforcement learning without
  reconstruction.
\newblock \emph{arXiv preprint arXiv:2006.10742}, 2020.

\bibitem[Zhang et~al.(2022{\natexlab{a}})Zhang, Ren, Yang, Gonzalez,
  Schuurmans, and Dai]{zhang2022making}
Tianjun Zhang, Tongzheng Ren, Mengjiao Yang, Joseph Gonzalez, Dale Schuurmans,
  and Bo~Dai.
\newblock Making linear mdps practical via contrastive representation learning.
\newblock In \emph{International Conference on Machine Learning}, pp.\
  26447--26466. PMLR, 2022{\natexlab{a}}.

\bibitem[Zhang et~al.(2022{\natexlab{b}})Zhang, Song, Uehara, Wang, Agarwal,
  and Sun]{zhang2022efficient}
Xuezhou Zhang, Yuda Song, Masatoshi Uehara, Mengdi Wang, Alekh Agarwal, and Wen
  Sun.
\newblock Efficient reinforcement learning in block mdps: A model-free
  representation learning approach.
\newblock In \emph{International Conference on Machine Learning}, pp.\
  26517--26547. PMLR, 2022{\natexlab{b}}.

\end{thebibliography}
\bibliographystyle{iclr2023_conference}

\newpage
\appendix

\section{More Related Work}\label{appendix:more_related}

Our method is also closely related to the model-based reinforcement learning. 
% Some of the work estimate the model via a simple $\ell_2$ loss minimization \citep[e.g.][]{kurutach2018model, clavera2018model}, which potentially ignores the structure of the noise and constraints the performance.
% \lina{check whether this sentence  ``which potentially ignores... performance'' is really what you want to say. I feel the sentence is missing words. ``the structure of the noise and constraints"? the ``the performance"? }. 
These methods maintain an estimation of the dynamics and reward from the data, and extract the optimal policy via planning modules. The major differences among these methods are in terms of {\bf i)}, model parameterization and learning objectives, and {\bf ii)}, the approximated algorithms used for planning.

Specifically, Gaussian processes are exploited in~\citep{deisenroth2011pilco}. A stochastic deep dynamics with restricts Gaussian noise assumption is widely used~\citep{heess2015learning,kurutach2018model,chua2018deep,clavera2018model, nagabandi2018neural}.
~\citet{hafner2019dream, hafner2019learning, hafner2020mastering, lee2020stochastic} recently exploits recurrent latent state space model, but focused on Partially Observable MDP setting, which is beyond the scope of our paper. Different approximated planning algorithms, including Dyna-style, shooting, and policy search with backpropagation through time, have been tailored in these methods. Please refer to~\citep{wang2019benchmarking} for detailed discussion. 

% Other work assumes the next state follows a Gaussian distribution with the mean and covariance as a deterministic function of the current state-action pair \citep[e.g.][]{chua2018deep}, which is still quite limited. 

% Recently, there have been lots of work aiming at introducing latent variable models to construct the model. \citet{lee2020stochastic} proposed to learn a compact latent state with the latent variable model and conduct planning on the latent state. \citet{hafner2019dream, hafner2019learning, hafner2020mastering} focused on the Partially Observable MDP setting, and proposed to construct a recurrent state space model to model the environment and use variational methods to estimate the model. Such kinds of models can be much more flexible, and have shown superior performance on complicated environments like Atari, which shows the effectiveness of the latent variable model.

As we discussed in~\Secref{sec:intro}, these algorithms did not balance the flexibility in modeling and tractability in planning and exploration, which may lead to sub-optimal performances. While the proposed~\algabb not only is more flexible beyond Gaussian noise assumption, but also lead to provable and tractable learning, planning, and exploration, and thus, achieving better empirical performances. 
\section{Algorithms and Theoretical Analysis for Offline Exploitation}
\label{sec:offline}
\begin{algorithm}[ht] 
\caption{Offline Exploitation with \algabb} \label{alg:offline_algorithm}
\begin{algorithmic}[1]
  \State \textbf{Input:} Model class $\mathcal{P}=\{(p(z|s, a), p(s^\prime|z))\}, \mathcal{Q} = \{q(z|s, a, s^\prime)\}$, Offline Dataset $\mathcal{D}_n$.
  \State Learn the latent variable model $\hat{p}(z|s, a)$ with $\mathcal{D}_n $ via maximizing the ELBO defined in \eqref{eq:ELBO}, and obtain the learned model $\hat{T}$.\label{line:representation_offline} 
  \State Set the exploitation penalty $\hat{b}(s,a)$ as \eqref{eq:empirical_bonus}.\label{line:penalty}
  \State Obtain policy \label{line:offline_plan}
    $ \hat{\pi}=\mathop{\arg\max}_{\pi}V^{\pi}_{\hat{T},r-\hat{b}}$.
  \State \textbf{Return } $\hat{\pi}$.
\end{algorithmic}
\end{algorithm}
In this section, we show the algorithms for offline exploitation. For offline exploitation, we have the access to a offline dataset, which we assume is collected from the stationary distribution of the fixed behavior policy $\pi_b$, which we will denote as $\rho$. And we are not allowed to interact with the environments to collect new data. The only difference between the algorithms for offline exploitation and online exploration is that, as we do not have access to the new data from the environment, we cannot further explore the state-action pair that the offline dataset do not cover. Hence, we need to penalize the visitation to the unseen state action pair to avoid the risky behavior.
% \newpage
\section{Implementation Details}
\label{sec:appendix-impl}

Our algorithm is implemented using Pytorch. 
For DeepMind control, 
we use an open source implementation as our SAC baseline \citep{pytorch_sac}. 
\revise{
As discussed in Section~\ref{sec:exp-implementation}, 
we find it is beneficial to have more updates for the latent variable model than critic in practice. 
We use a parameter \emph{feature-updates-per-step} that decides how many updates are performed for the latent variable model at each training step.
For all Mujoco and DeepMind Control experiments, 
we tune this parameter from $\{1,3,5,10,20\}$ and report the best result. 
}
Finally, 
in Table~\ref{tab:hyper_online}, we list all other hyperparameters and network architecture we use for our experiments.

{\color{black} For evaluation in Mujoco, in each evaluation (every 5K steps) we test our algorithm for 10 episodes. We average the results over the last 4 evaluations and 4 random seeds. For Dreamer and Proto-RL, we change their network from CNN to 3-layer MLP and disable the image data augmentation part (since we test on the state space). We tune some of their hyperparameters (e.g., exploration steps in Proto-RL) and report the best number across our runs.}

\begin{table*}[h]
\caption{Hyperparameters used for \algabb in all the environments in MuJoCo and DM Control Suite.}
\footnotesize
\setlength\tabcolsep{3.5pt}
\label{tab:hyper_online}
\centering
\begin{tabular}{p{6cm}p{3cm}p{5cm}p{2.5cm}p{2.5cm}p{2cm}p{2cm}}
% {lcccccccccccc}
\toprule
& Hyperparameter Value \\ 
\midrule
Actor lr & 0.0003 \\
Model lr & 0.0003 \\
Actor Network Size (MuJoCo) & (256, 256) \\
Actor Network Size (DM Control) & (1024, 1024) \\
\algabb Feature Embedding Dim (MuJoCo) & 256 \\
\algabb Feature Embedding Dim (DM Control) & 1024  \\
ERP Embedding Network Size (DM Control) & (1024, 1024, 1024) \\
Critic Network Size (MuJoCo) & (256, 256, 1) \\
Critic Network Size (DM Control) & (1024, 1024, 1) \\
Discount & 0.99\\
Critic Target Update Tau & 0.005 \\
Latent Variable Target Update Tau & 0.005 \\
Batch Size & 256 \\
\bottomrule 
\end{tabular}
\end{table*}

\newpage
\section{Technical Backgrounds}\label{appendix:tech_back}
In this section, we introduce several important concepts from functional analysis that will be repeatedly used in our theoretical analysis. We start from the concept of the $\mathbb{R}$-vector space.
\begin{definition}[$\mathbb{R}$-vector space \citep{steinwart2008support}]
    An $\mathbb{R}$-vector space is defined as a triple $(E, +, \cdot)$, where $E$ is a non-empty set, $+:E\times E\to E$ and $\cdot:\mathbb{R}\times E \to E$ satisfies the following properties:
    \begin{itemize}
        \item $\forall x, y, z \in E$, $(x + y) + z = x + (y + z)$.
        \item $\forall x, y\in E$, $x + y = y + x$.
        \item $\exists$ an element $0\in E$, such that $\forall x\in E$, $x + 0 = x$.
        \item $\forall x\in E$, $\exists -x \in E$, such that $x + (-x) = 0$.
        \item $\forall \alpha, \beta\in\mathbb{R}, x\in E$, $(\alpha\beta)\cdot x = \alpha\cdot(\beta\cdot x)$.
        \item $\forall x\in E$, $1\cdot x = x$.
        \item $\forall \alpha, \beta\in\mathbb{R}, x\in E$, $(\alpha + \beta)\cdot x = \alpha\cdot x + \beta\cdot x$.
        \item $\forall \alpha\in\mathbb{R}, x, y\in E$, $\alpha\cdot(x + y) = \alpha \cdot x + \alpha \cdot y$.
    \end{itemize}
    The $\cdot$ denotes the scalar multiplication will be omitted if there will be no confusion.
\end{definition}
\begin{definition}[Norm and Banach Space \citep{steinwart2008support}] Let $E$ be a $\mathbb{R}$-vector space. A map $\|\cdot\|:E \to [0, \infty)$ is a norm on $E$ if
\begin{itemize}
    \item $\|x\| = 0 \Leftrightarrow x = 0$.
    \item $\forall \alpha\in\mathbb{R}, x\in E, \|\alpha x\| = \alpha \|x\|$.
    \item $\forall x, y\in E$, $\|x + y\|\leq \|x\| + \|y\|$.
\end{itemize}
In this case, the pair $(E, \|\cdot\|)$ is called a Banach space, and we use $E$ to denote the Banach space for simplicity if there will be no confusion.    
\end{definition}

\begin{definition}[Bounded Linear Operator \citep{steinwart2008support}]
    Let $E$ and $F$ be two Banach spaces. A map $S:E\to F$ is a bounded linear operator if
    \begin{itemize}
        \item $\forall x, y\in E, S(x + y) = S x + S y$.
        \item $\forall \alpha \in \mathbb{R}, x \in E, S(\alpha x) = \alpha(Sx)$.
        \item $S 0 = 0$.
    \end{itemize}
    Furthermore, $S$ satisfies the following properties
    \begin{itemize}
        \item $\exists c\in [0, \infty]$, such that $\forall x\in E$, $\|Sx\|_{F} \leq c\|x\|_E$.
    \end{itemize}
    Note that, all of the bounded linear operator itself can define an $\mathbb{R}$-vector space, and we can define an operator norm of $S$ as $\|S\|_{\mathrm{op}}:=\sup_{x\in \mathcal{B}_E} \|Sx\|_{F}$, where $\mathcal{B}_{E} = \{x\in E : \|x\|_E \leq 1\}$ is the unit ball of $E$.
\end{definition}

\begin{definition}[Compact Operator \citep{steinwart2008support}]
    A bounded linear operator $S:E\to F$ is compact if the closure of $S\mathcal{B}_E$ is compact in $F$.
\end{definition}
\begin{definition}[Inner Product and Hilbert Space \citep{steinwart2008support}]
    A map $\langle \cdot, \cdot \rangle: \mathcal{H}\times \mathcal{H} \to \mathbb{R}$ on a $\mathbb{R}$-vector space is an inner product if
    \begin{itemize}
        \item $\forall x, y, z\in \mathcal{H}$, $\langle x + y, z\rangle = \langle x, z\rangle + \langle y, z\rangle$.
        \item $\forall \alpha\in\mathbb{R}, x, y\in \mathcal{H}$, $\langle \alpha x, y\rangle=\alpha\langle x, y\rangle$.
        \item $\forall x, y\in \mathcal{H}$, $\langle x, y\rangle = \langle y, x\rangle$.
        \item $\forall x \in \mathcal{H}$, $\langle x, x\rangle \geq 0$, and $\langle x, x\rangle = 0 \Leftrightarrow x = 0$.
    \end{itemize}
    If the norm induced by the inner product $\|x\|_{\mathcal{H}}:=\sqrt{\langle x, x\rangle}$ is complete, the pair $(\mathcal{H}, \langle \cdot, \cdot\rangle)$ is called a Hilbert space. We sometimes use $\mathcal{H}$ to denote the Hilbert space and use $\langle \cdot, \cdot\rangle_\mathcal{H}$ to distinguish between different inner products. Note that, the inner product satisfies the following Cauchy-Schwartz inequality:
    \begin{align*}
        \forall x, y\in \mathcal{H},\quad |\langle x, y\rangle_{\mathcal{H}}|\leq \|x\|_{\mathcal{H}}\|y\|_{\mathcal{H}}.
    \end{align*}
\end{definition}
\begin{definition}[(Self-)Adjoint Operator \citep{steinwart2008support}]
    Let $H_1$ and $H_2$ be two Hilbert spaces, For the operator $S:\mathcal{H}_1 \to \mathcal{H}_2$, the adjoint operator $S^*:\mathcal{H}_2 \to \mathcal{H}_1$ is defined by
    \begin{align*}
        \forall x\in \mathcal{H}_1, y\in \mathcal{H}_2, \quad \langle Sx, y \rangle_{\mathcal{H}_2} = \langle x, S^* y\rangle_{\mathcal{H}_1}.
    \end{align*}
    Furthermore, $S$ is a self-adjoint operator, if $S:\mathcal{H}_1 \to \mathcal{H}_1$, and
    \begin{align*}
        \forall x, y\in \mathcal{H}_1, \quad \langle Sx, y\rangle_{\mathcal{H}_1} = \langle x, Sy \rangle_{\mathcal{H}_1}.
    \end{align*}
    For self-adjoint operator $S$, if $\langle Sx, x\rangle\geq 0$, $S$ is called a positive semi-definite operator, and if $\langle Sx, x\rangle > 0$, $S$ is called a positive definite operator.
\end{definition}
\begin{remark}
    Note that, the definition of the adjoint operator can be generalized to Banach spaces. But adjoint operators for Hilbert spaces are sufficient for our purposes. So we omit the definition of the adjoint operators on Banach spaces.
\end{remark}

\begin{definition}[Orthonormal System and Orthonormal Basis \citep{steinwart2008support}]
    For the Hilbert space $\mathcal{H}$, the family $\{e_i\}_{i\in I}, e_i\in \mathcal{H}$ is an orthonormal system if $\langle e_i, e_i \rangle = 1$, and $\langle e_i, e_j\rangle = 0$ if $i\neq j$. Furthermore, if the closure of the linear span of $\{e_i\}_{i\in I}$ equals to $\mathcal{H}$, it is an orthonormal basis. 
    Note that, each Hilbert space $H$ has an orthonormal basis, and if $\mathcal{H}$ is separable, $\mathcal{H}$ has a countable orthonormal basis. Furthermore, $\forall x\in H$, we have
    \begin{align*}
        x = \sum_{i\in I}\langle x, e_i\rangle e_i.
    \end{align*}
\end{definition}

\begin{theorem}[Spectral Theorem \citep{steinwart2008support}]
    Let $\mathcal{H}$ be a Hilbert space and $T:\mathcal{H}\to \mathcal{H}$ is compact and self-adjoint. Then their exists at most countable $\{\mu_i(T)\}_{i\in I}$ converging to $0$ such that $|\mu_1(T)| \geq |\mu_2(T)| \geq \cdots > 0$ and an orthonormal system $\{e_i\}_{i\in I}$, such that
    \begin{align*}
        \forall x\in \mathcal{H}, \quad Tx = \sum_{i\in I}\mu_i(T) \langle x, e_i\rangle_{\mathcal{H}} e_i.
    \end{align*}
    Here $\{\mu_i(T)\}_{i\in I}$ can be viewed as the set of eigen-value of $T$, as $T e_i = \mu_i(T)$.
\end{theorem}

\begin{definition}[Trace and Trace class \citep{steinwart2008support}]
    For a compact and self-adjoint operator $T$, if $\sum_{i=1}^{\infty} \mu_i(T) < \infty$, we say $T$ is nuclear or of trace class, and define the nuclear norm and the trace as:
    \begin{align*}
        \|T\|_{*} = \mathrm{Tr}(T) = \sum_{i\in I} \mu_i(T).
    \end{align*}
    % \Tongzheng{Not sure if we need the definition of Fredholm determinant. We need to define matrix logarithm and matrix exponential if we want to formally define the Fredholm determinant. Check it later.}
\end{definition}

\begin{definition}[Hilbert-Schmidt Operator \citep{steinwart2008support}]
    Let $\mathcal{H}_1, \mathcal{H}_2$ be two Hilbert spaces. An operator $S:\mathcal{H}_1\to \mathcal{H}_2$ is Hilbert-Schmidt if 
    \begin{align*}
        \|S\|_{\mathrm{HS}} := \left(\sum_{i\in I} \|S e_i\|_{\mathcal{H}_2}^2\right)^{1/2} < \infty,
    \end{align*}
    where $\{e_i\}_{i\in I}$ is an arbitrary orthonormal basis of $\mathcal{H}_1$. Furthermore, the set of Hilbert-Schmidt operators defined on $\mathcal{H} \to \mathcal{H}$ where $\mathcal{H}$ is a Hilbert space is indeed a Hilbert space with the following inner product:
    \begin{align*}
        \langle T_1, T_2\rangle_{\mathrm{HS}(H)} = \sum_{i\in I} \langle T_1 e_i, T_2 e_i\rangle_{\mathcal{H}}, \quad T_1, T_2 \in \mathrm{HS}(\mathcal{H}),
    \end{align*}
    where $\{e_i\}_{i\in I}$ is an arbitrary orthonormal basis of $\mathcal{H}$.
\end{definition}

% \begin{definition}[Approximation Number of Bounded Linear Operator]
%     \Tongzheng{See if we need this, as we can directly use eigenvalue and refer to Steinwart.}
% \end{definition}
\begin{definition}[$L^2(\mu)$ space]
    Let $(\mathcal{X}, \mathcal{A}, \mu)$ be a measure space. The $L^2(\mu)$ space is defined as the Hilbert space consists of square-integrable function with respect to $\mu$, with inner product
    \begin{align*}
        \langle f, g\rangle_{L_2(\mu)} := \int_{\mathcal{X}} f g d \mu,
    \end{align*}
    and the norm
    \begin{align*}
        \|f\|_{L_2(\mu)} := \left(\int_{\mathcal{X}} f^2 d \mu\right)^{1/2}.
    \end{align*}
    Throughout the paper, $\mu$ is specified as the Lebesgue measure for continuous $\mathcal{X}$ and the counting measure for discrete $\mathcal{X}$. Specifically, when $\mathcal{X}$ is discrete, we can represent $f$ as a sequence $[f(x)]_{x\in\mathcal{X}}$, and the corresponding $L_2(\mu)$ inner product and $L_2(\mu)$ norm is identical to the $\ell^2$ inner product and $\ell^2$ norm, which is defined as
    \begin{align*}
        \langle f, g\rangle_{l^2} = \sum_{x\in\mathcal{X}} f(x) g(x), \quad \|f\|_{l^2} = \left(\sum_{x\in\mathcal{X}} f^2(x)\right)^{1/2},
    \end{align*}
    that is closely related to the inner product and norm of the Euclidean space.
\end{definition}

\begin{definition}[Kernel and Reproducing Kernel Hilbert Space (RKHS) \citep{aronszajn1950theory,paulsen2016introduction}]
    A function $k:\mathcal{X} \times \mathcal{X} \to \mathbb{R}$ is a kernel on non-empty set $\mathcal{X}$, if there exists a Hilbert space $\mathcal{H}$ and a feature map $\phi:\mathcal{X} \to \mathcal{H}$, such that $\forall x, x^\prime\in\mathcal{X}$, $k(x, x^\prime) = \langle \phi(x), \phi(x^\prime) \rangle_{\mathcal{H}}$. Furthermore, if $\forall n \geq 1, \{a_i\}_{i\in [n]}\subset \mathbb{R}$ and mutually distinct $\{x_i\}_{i\in [n]}$,
    \begin{align*}
        \sum_{i\in [n]} \sum_{j\in [n]} a_i a_j k(x_i, x_j) \geq 0,
    \end{align*}
    the kernel $k$ is said to be positive semi-definite. And if the inequality is strict, the kernel $k$ is said to be positive definite.

    Given the kernel $k$, the Hilbert space $\mathcal{H}_k$ consists of $\mathbb{R}$-valued function on non-empty set $\mathcal{X}$ is said to be a reproducing kernel Hilbert space associated with $k$ if the following two conditions hold:
    \begin{itemize}
        \item $\forall x\in\mathcal{X}$, $k(x, \cdot) \in \mathcal{H}_k$.
        \item Reproducing Property: $\forall x\in \mathcal{X}, f\in \mathcal{H}_k, f(x) = \langle f, k(x, \cdot)\rangle_{\mathcal{H}_k}$. 
    \end{itemize}
    Here $k$ is also called the reproducing kernel of $\mathcal{H}_k$. The RKHS norm of $f\in\mathcal{H}_k$ is defined as $\|f\|_{\mathcal{H}_k} := \sqrt{\langle f, f \rangle_{\mathcal{H}_k}}$. 
\end{definition}
Some of the well-known kernels include:
\begin{itemize}
    \item Linear Kernel: $k(x, x^\prime) = x^\top x^\prime$, where $x, x^\prime \in \mathbb{R}^d$;
    \item Polynomial Kernel: $k(x, x^\prime) = (1 + x^\top x^\prime)^n$, where $x, x^\prime\in\mathbb{R}^d$, $n\in \mathbb{N}^{+}$.
    \item Gaussian Kernel: $k(x, x^\prime) = \exp\left(-\frac{\|x-x^\prime\|_2^2}{\sigma^2}\right)$, where $x, x^\prime\in\mathbb{R}^d$, $\sigma > 0$ is the scale parameter.
    \item Mat\'ern Kernel: $k(x, x^\prime) = \frac{2^{1-\nu}}{\Gamma(\nu)} r^{\nu} B_{\nu} (r)$, where $x, x^\prime\in\mathbb{R}^d$, $\nu > 0$ is the smoothness parameter, $l > 0$ is the scale parameter, $r = \frac{\sqrt{2\nu}}{l}\|x - x^\prime\|$, $\Gamma(\cdot)$ is the Gamma function and $B_\nu(\cdot)$ is the modified Bessel function of the second kind.
\end{itemize}

\begin{theorem}[Mercer's Theorem \citep{riesz2012functional, steinwart2012mercer}]
    \label{thm:mercer}
    Let $(\mathcal{X}, \mathcal{A}, \mu)$ be a measure space with compact support $\mathcal{X}$ and strictly positive Borel measure $\phi$. $k$ is a continuous positive definite kernel defined on $\mathcal{X} \times \mathcal{X}$. Then there exists at most countable $\{\mu_i\}_{i\in I}$ with $\mu_1 \geq \mu_2 \geq \cdots > 0$ and $\{e_i\}_{i\in I}$ where $\{e_i\}_{i\in I}$ is the set of orthonormal basis of $L_2(\mu)$, such that
    \begin{align*}
        \forall x, x^\prime \in \mathcal{X}, \quad k(x, x^\prime) = \sum_{i\in I} \mu_i e_i(x) e_i(x^\prime),
    \end{align*}
    where the convergence is absolute and uniform. 
\end{theorem}
\begin{remark}[Spectral Characterization of RKHS]
    \label{remark:spectral_characterization_rkhs}
    With the representer property, we know that 
    \begin{align*}
        \sum_{i\in I}\mu_i e_i(x) e_i(\cdot) = k(x, \cdot) \in \mathcal{H}_k.
    \end{align*} 
    Note that, for $\beta$-finite spectrum, we can choose $I$ such that $\mu_i > 0$. If we define the inner product
    \begin{align*}
        \left\langle \sum_{i \in I} a_i e_i(\cdot), \sum_{i \in I} b_i e_i(\cdot) \right\rangle_{\mathcal{H}_k} = \sum_{i \in I}\frac{ a_i b_i}{\mu_i},
    \end{align*} 
    then we have the reproducing property 
    \begin{align*}
         \left\langle \sum_{i \in I} a_i e_i(\cdot), k(x, \cdot)\right\rangle_{\mathcal{H}_k} = \left\langle \sum_{i \in I} a_i e_i(\cdot), \sum_{i\in I} \mu_i e_i(x) e_i(\cdot)\right\rangle_{\mathcal{H}_k} = \sum_{i \in I} a_i e_i(x).
    \end{align*}
    With the spectral characterization, we know that
    \begin{align*}
        \left\|\sum_{i\in I} a_i e_i(\cdot)\right\|_{\mathcal{H}_k} = \sum_{i\in I}\frac{a_i^2}{\mu_i} \geq \frac{\sum_{i\in I} a_i^2}{\mu_1} = \frac{\left\|\sum_{i \in I} a_i e_i(\cdot)\right\|_{L_2(\mu)}}{\mu_1}.
    \end{align*}
    Hence, we know $\forall f\in\mathcal{H}_k$, $f\in L_2(\mu)$. Furthermore, note that
    \begin{align*}
        k(x, x) = \left\langle \sum_{i\in I} \mu_i e_i(x) e_i(\cdot), \sum_{i\in I} \mu_i e_i(x) e_i(\cdot)\right\rangle_{\mathcal{H}_k} = \sum_{i\in I} \mu_i e_i^2(x).
    \end{align*}
    Hence,
    \begin{align*}
        \sum_{i\in I} \mu_i = \int_{\mathcal{X}} \sum_{i\in I} \mu_i e_i^2(x) d \mu = \int_{\mathcal{X}} k(x, x) d \mu.
    \end{align*}
    The following Hilbert-Schmidt integral operator is useful in our analysis:
    \begin{align*}
        T_k:L_2(\mu) \to L_2(\mu), \quad T_k f = \int_{\mathcal{X}} k(x, x^\prime) f(x^\prime) d \mu.
    \end{align*}
    Obviously, $T_k$ is self-adjoint. Use the fact that $k(x, x^\prime) = \sum_{i\in I} \mu_i e_i(x) e_i(x^\prime)$, we know $T_k e_i = \mu_i e_i$, which means $e_i$ is the eigenfunction of $T_k$ with the corresponding eigenvalue as $\mu_i$. 

    With the spectral characterization of $T_k$, we can define the power operator $T_k$, by
    \begin{align*}
        T_k^{\tau} f: L_2(\mu) \to L_2(\mu), \quad T_k^{\tau} f = \sum_{i\in I} \mu_i^\tau \langle f, e_i\rangle e_i.
    \end{align*}
    And these power operators are all self-adjoint. Note that, $\|f\|_{\mathcal{H}_k} = \langle f, T_k^{-1} f\rangle_{L_2(\mu)}$. Throughout the paper, we work on the $L_2(\mu)$ space, and all of the operators are defined on $L_2(\mu) \to L_2(\mu)$. As $\mathcal{H}_k \subset L_2(\mu)$, all of these operators can also operator on the elements from $\mathcal{H}_k$.  
    
    The power RKHS induced by the following kernel will be helpful in our analysis:
    \begin{align*}
        \forall x, x^\prime\in\mathcal{X}, \widetilde{k}(x, x^\prime) = \sum_{i\in I} \mu_i^2 e_i(x) e_i(x^\prime). 
    \end{align*}
    And it is straightforward to see $\|f\|_{H_{\widetilde{k}}} = \langle f, T_k^{-2} f\rangle_{L_2(\mu)}$, which will be useful in the proof.
\end{remark}

\begin{definition}[Kernel with Random Feature Representation]
    \label{def:random_feature_representation}
    A kernel $k:\mathcal{X} \times \mathcal{X} \to \mathbb{R}$ is said to have a random feature representation if there exists a function $\psi:\mathcal{X} \times \Xi \to \mathbb{R}$ and a probability measure $P$ over $\Xi$ such that
    \begin{align*}
        k(x, x^\prime) = \int_{\Xi} \psi(x; \xi) \psi(x^\prime; \xi) dP(\xi).
    \end{align*}
    We then show that, $\mathcal{H}_k$ coincides with the following $\mathbb{R}$-valued function space 
    \begin{align*}
        \left\{f:\mathcal{X} \to \mathbb{R}~\bigg|~f(x) = \int_{\Xi} \widetilde{f}(\xi) \psi(x; \xi) d P(\xi), \widetilde{f} \in L_2(P) \right\},
    \end{align*} 
    with the inner product defined as $\langle f, g\rangle_{\mathcal{H}_k} = \int_{\Xi} \widetilde{f}(\xi) \widetilde{g}(\xi) d P(\xi)$. Note that, $\widetilde{k}(x, \cdot) = \psi(x; \xi)$. 
    % \lina{I feel all  $\widetilde{k}(x, \cdot)$ should be  $\widetilde{k}(x, \xi)$.} 
    % \Bo{we use $\widetilde{k}(x, \cdot)$ to denote the function by default.}
    Hence, it is straightforward to show that $\forall x\in\mathcal{X}$, $k(x, \cdot) \in \mathcal{H}_k$. Furthermore, we have 
    \begin{align*}
        f(x) = \int_{\Xi} \widetilde{f}(\xi) \psi(x; \xi) dP(\xi) = \int_{\Xi} \widetilde{f}(\xi) \widetilde{k}(x, \cdot) d P(\xi) = \langle f, k(x, \cdot) \rangle_{\mathcal{H}_k},
    \end{align*}
    which shows the reproducing property. As a result, we obtain an equivalent representation of the RKHS $\mathcal{H}_k$, which means $\forall f\in\mathcal{H}_k$, we can obtain a random feature representation.
\end{definition}
Examples of such kernel $k$ includes the Gaussian kernel and the Mat\'ern kernel. See \citet{rahimi2007random, dai2014scalable, choromanski2018geometry} for the details.

\begin{definition}[$\varepsilon$-net and $\varepsilon$-covering number and $i$-th (dyadic) entropy number~\citep{steinwart2008support}]
    Let $(T, d)$ be a metric space. $S \subset T$ is an $\varepsilon$-net for $\varepsilon > 0$, if $\forall t\in T$, $\exists s\in S$, such that $d(s, t) \leq \varepsilon$. Furthermore, the $\varepsilon$-covering number $\mathcal{N}(T, d, \varepsilon)$ is defined as as the minimum cardinality of the $\varepsilon$-net for $T$ under metric $d$, and the $i$-th entropy number $e_i(T, d)$ is the minimum $\varepsilon$ that there exists an $\varepsilon$ cover of cardinality $2^{i-1}$.
\end{definition}
% \newpage
\section{Main Proof}
\subsection{Technical Conditions}
\label{sec:technical_conditions}
\begin{assumption}[Regularity Conditions for Kernel]
    \label{assump:trace}
    $\mathcal{Z}$ is a compact metric space with the Lebesgue measure $\mu$ if $\mathcal{Z}$ is continuous, and $\int_{\mathcal{Z}} k(z, z) d \mu \leq 1.$
\end{assumption}

\begin{remark}
    Assumption~\ref{assump:trace} is mainly for the ease of presentation. The assumption that $\mathcal{Z}$ is compact when $\mathcal{Z}$ is continuous can be relaxed to $\mathcal{Z}$ is a general domain but requires much more involved techniques from e.g. \citet{steinwart2012mercer}. Furthermore, with Mercer's theorem (see Theorem~\ref{thm:mercer} and Remark~\ref{remark:spectral_characterization_rkhs} for the details), we know $\sum_{i\in I} \mu_i = \int_{\mathcal{Z}} k(z, z) d \mu \leq 1$.
    As $\forall i\in I, \mu_i > 0$, we know $\mu_1 \leq 1$, and $\|f\|_{L_2(\mu)} \leq \|f\|_{\mathcal{H}_k}$ without any other absolute constant, that can keep the eventual result clean. We can relax the assumption $\int_{\mathcal{Z}} k(z, z) d\mu \leq 1$ to $\int_{\mathcal{Z}} k(z, z) d \mu \leq c$ for some positive constant $c > 1$, at the cost of additional terms at most  $\mathrm{poly}(c)$  in the sample complexity.
\end{remark}

\begin{assumption}[Eigendecay Conditions for Kernel] For the reproducing kernel $k$, we assume $\mu_i$, the $i$-th eigenvalue of the operator $T_k:L_2(\mu) \to L_2(\mu)$, $T_k f =\int_{\mathcal{Z}} f(z^\prime) k(z, z^\prime)d \mu(z^\prime)$, satisfies one of the following conditions:
\begin{itemize}[leftmargin=20pt, parsep=0pt, partopsep=0pt]
\item $\beta$-finite spectrum: $\mu_i = 0$, $\forall i > \beta$, where $\beta$ is a positive integer.
\item $\beta$-polynomial decay: $\mu_i \leq C_0 i^{-\beta}$, where $C_0$ is an absolute constant and $\beta > 1$.
\item $\beta$-exponential decay: $\mu_i \leq C_1 \exp(-C_2 i^\beta)$, where $C_1$ and $C_2$ are absolute constants and $\beta > 0$.
\end{itemize}
For ease of presentation, we use $C_{\mathrm{poly}}$ to denote constants appeared in the analysis of $\beta$-polynomial decay that only depends on $C_0$ and $\beta$, and $C_{\mathrm{exp}}$ to denote constants appeared in the analysis of $\beta$-exponential decay that only depends on $C_1$, $C_2$ and $\beta$. Both of them can be varied step by step.
\end{assumption}
\begin{remark}
    We remark that, most of the existing kernels satisfy one of these eigendecay conditions. Specifically, as discussed in \citet{seeger2008information, yang2020provably}, the linear kernel and the polynomial kernel satisfy the $\beta$-finite spectrum condition, the Mat\'ern kernel satisfies the $\beta$-polynomial decay and the Gaussian kernel satisfies the $\beta$-exponential decay. Furthermore, for discrete $\mathcal{Z}$, we can directly observe that it corresponds to the case of $\beta$-finite spectrum with $\beta \leq |\mathcal{Z}|$.
\end{remark}
\subsection{Proof for the Online Setting}
\label{sec:online_proof}
\begin{theorem}[PAC Guarantee for Online Exploration, Formal]
\label{thm:pac_online}
If we choose the bonus $\hat{b}_n(s, a)$ as:
\begin{align*}
    \hat{b}_n(s, a) = \min\left\{\alpha_n\|\hat{p}_n(\cdot|s, a)\|_{L_2(\mu), \hat{\Sigma}_{n, \hat{p}_n}^{-1}}, 2\right\},
\end{align*}
where 
\begin{align*}
    \hat{\Sigma}_{n, \hat{p}_n}:L_2(\mu) \to L_2(\mu), \quad \hat{\Sigma}_{n, \hat{p}_n} := \sum_{(s_i, a_i) \in \mathcal{D}_n} \left[\hat{p}_n(z|s_i, a_i) \hat{p}_n(z|s_i, a_i)^\top\right] + \lambda T_k^{-1},
\end{align*}
$\left\|f\right\|_{L_2(\mu), \Sigma} := \sqrt{\langle f, \Sigma f\rangle_{L_2(\mu)}}$ for self-adjoint operator $\Sigma$, $\lambda$ for different eigendecay conditions is given by:
\begin{itemize}[leftmargin=20pt, parsep=0pt, partopsep=0pt]
    \item $\beta$-finite spectrum: $\lambda = \Theta(\beta \log N + \log (N|\mathcal{P}|/\delta))$
    \item $\beta$-polynomial decay: $\lambda = \Theta(C_{\mathrm{poly}}N^{1/(1 + \beta)} + \log (N|\mathcal{P}|/\delta))$;
    \item $\beta$-exponential decay: $\lambda = \Theta(C_{\mathrm{exp}}(\log N)^{1/\beta} + \log (N|\mathcal{P}|/\delta))$;
\end{itemize}
and $\alpha_n = \Theta\left(\frac{\gamma}{1-\gamma}\sqrt{|\mathcal{A}|\log(n|\mathcal{P}|/\delta) + \lambda C}\right)$, then with probability at least $1-\delta$, After interacting with the environments for $N$ episodes where 
    \begin{itemize}[leftmargin=20pt, parsep=0pt, partopsep=0pt]
        \item $N =  \Theta\left(\frac{C \beta^3 |\mathcal{A}|^2 \log (|\mathcal{P}|/\delta)}{(1-\gamma)^4\varepsilon^2}\log^3 \left(\frac{C \beta^3 |\mathcal{A}|^2 \log (|\mathcal{P}|/\delta)}{(1-\gamma)^4\varepsilon^2}\right)\right)$ for $\beta$-finite spectrum;
        \item $N = \Theta\left(C_{\mathrm{poly}}\left(\frac{|\mathcal{A}|\sqrt{C\log (|\mathcal{P}|/\delta)}}{(1-\gamma)^2\varepsilon} \log^{3/2} \left(\frac{\sqrt{C}|\mathcal{A}|\log (|\mathcal{P}|/\delta)}{(1-\gamma)^2\varepsilon} \right)\right)^{\frac{2(1+\beta)}{\beta - 1}}\right)$ for $\beta$-polynomial decay;
        \item $N =\Theta\left(\frac{C_{\mathrm{exp}}C|\mathcal{A}|^2\log (|\mathcal{P}|/\delta)}{(1-\gamma)^4\varepsilon^2} \log^{\frac{3 + 2\beta}{\beta}}\left(\frac{C|\mathcal{A}|^2\log (|\mathcal{P}|/\delta)}{(1-\gamma)^4\varepsilon^2} \right)\right)$ for $\beta$-exponential decay;
    \end{itemize} 
    we can obtain an $\varepsilon$-optimal policy with probability at least $1-\delta$.
\end{theorem}

\paragraph{Notation}
Following the notation of \citet{uehara2021representation}, we define
\begin{align*}
    \rho_n(s) := \frac{1}{n}\sum_{i=1}^{n-1} d_{T^*}^{\pi_i}(s),
\end{align*}
and with slight abuse of notation, we overload the above notation and define 
\begin{align*}
    \rho_n(s, a) := \frac{1}{n}\sum_{i=1}^{n-1} d_{T^*}^{\pi_i}(s, a).
\end{align*}
Furthermore, we define $\rho_n^\prime(s^\prime)$ as the marginal distribution of $s^\prime$ for the following joint distribution 
\begin{align*}
    (s, a, s^\prime) \sim \rho_n(s) \times \mathcal{U}(\mathcal{A}) \times T^*(s^\prime|s, a).
\end{align*}
Finally, we define the following operators in the space of $L_2(\mu) \to L_2(\mu)$:
\begin{align*}
    \Sigma_{\rho_n\times\mathcal{U}(\mathcal{A}), \phi} = & n \mathbb{E}_{s\sim \rho_n, a\sim \mathcal{U}(\mathcal{A})} \left[\phi(s, a) \phi^\top(s, a)\right] + \lambda T_k^{-1}\\
    \Sigma_{\rho_n, \phi} = & n \mathbb{E}_{(s, a)\sim \rho_n}\left[\phi(s, a) \phi^\top(s, a)\right] + \lambda T_k^{-1}\\
    % \hat{\Sigma}_{n, \phi} = & \sum_{(s_i, a_i) \in \mathcal{D}_n}\left[\phi(s_i, a_i) \phi^\top(s_i, a_i)\right] + \lambda T_k^{-1}.
\end{align*}
% For the compact and self-adjoint operator $T:L_2(\mu) \to L_2(\mu)$, we define the weighted $L_2(\mu)$ norm as
% \begin{align*}
%     \|f\|_{L_2(\mu), T} = \langle f, Tf\rangle_{L_2(\mu)}.
% \end{align*}
Note that, by the spectral theorem, if $\left\|T^{-1/2} x^\prime\right\|_{L_2(\mu)} \leq \infty$ for $x^\prime\in L_2(\mu)$, we have the following Cauchy-Schwartz inequality for weighted $L_2(\mu)$ norm: $\forall x\in L_2(\mu)$,
\begin{align*}
    \left\langle x, x^\prime \right\rangle_{L_2(\mu)} = \left\langle T^{1/2} x, T^{-1/2}x^\prime\right\rangle_{L_2(\mu)} \leq \|x\|_{L_2(\mu), T} \left\|x^\prime\right\|_{L_2(\mu), T^{-1}}.
\end{align*}
\begin{lemma}[One Step Back Inequality for the Learned Model]
    \label{lem:back_learned_online}
    Assume $g:\mathcal{S} \times \mathcal{A} \to \mathbb{R}$ such that $\|g\|_{\infty} \leq B$, then conditioning on the event that the following MLE generalization bound holds:
    \begin{align*}
        \mathbb{E}_{s\sim \rho_n, a\sim\mathcal{U}(\mathcal{A})}\left[\|\hat{T}(s, a) - T^*(s, a)\|_1\right] \leq \zeta_n,
    \end{align*} 
    $\forall \pi$, we have
    \begin{align*}
        & \left|\mathbb{E}_{(s, a) \sim d_{\hat{T}_n}^\pi}[g(s, a)]\right| \\
        \leq & \gamma \mathbb{E}_{(\widetilde{s}, \widetilde{a})\sim d_{\hat{T}_n}^\pi} \left\|\hat{p}_n(\cdot|\widetilde{s}, \widetilde{a})\right\|_{L_2(\mu), \Sigma_{\rho_n \times \mathcal{U}(\mathcal{A}), \hat{p}_n}^{-1}} \sqrt{ n|\mathcal{A}|\mathbb{E}_{s\sim \rho_n^\prime, a\sim \mathcal{U}(\mathcal{A})}  \left[g^2(s, a) \right] + \lambda B^2 C + n B^2 \zeta_n}\\
        & + \sqrt{(1-\gamma) |\mathcal{A}|\mathbb{E}_{s\sim \rho_n, a\sim \mathcal{U}(\mathcal{A})} [g^2(s, a)]}.
    \end{align*}
\end{lemma}
\begin{proof}
    We start from the following equality:
    \begin{align}
        \mathbb{E}_{(s, a) \sim d_{\hat{T}_n}^\pi} [g(s, a)] = \gamma \mathbb{E}_{(\widetilde{s}, \widetilde{a}) \sim d_{\hat{T}_n}^\pi, s \sim \hat{T}_n(\cdot|\widetilde{s}, \widetilde{a}), a\sim \pi(\cdot|s)} [g(s, a)] + (1-\gamma) \mathbb{E}_{s\sim d_0, a\sim \pi(\cdot|s)} [g(s, a)],
    \end{align}
    which is obtained by the property of the stationary distribution. For the second term, with Jensen's inequality and an importance sampling step, we have that
    \begin{align*}
        (1-\gamma) \mathbb{E}_{s\sim d_0, a\sim \pi(\cdot|s)} [g(s, a)] \leq \sqrt{(1-\gamma)|\mathcal{A}|\mathbb{E}_{s\sim \rho_n, a\sim \mathcal{U}(\mathcal{A})} [g^2(s, a)]}.
    \end{align*}
    Now we consider the first term. With Cauchy-Schwartz inequality of $L_2(\mu)$ inner product, we have that
    \begin{align*}
        & \gamma \mathbb{E}_{(\widetilde{s}, \widetilde{a}) \sim d_{\hat{T}_n}^\pi, s \sim \hat{T}_n(\cdot\widetilde{s}, \widetilde{a}), a\sim \pi(\cdot|s)} [g(s, a)]\\
        = & \gamma \mathbb{E}_{(\widetilde{s}, \widetilde{a})\sim d_{\hat{T}_n}^\pi} \left\langle \hat{p}_n(\cdot|\widetilde{s}, \widetilde{a}), \int_{\mathcal{S}} \sum_{a\in \mathcal{A}} \hat{p}_n(s|\cdot) \pi(a|s) g(s, a) ds\right\rangle_{L_2(\mu)}\\
        \leq & \gamma \mathbb{E}_{(\widetilde{s}, \widetilde{a})\sim d_{\hat{T}_n}^\pi} \left\|\hat{p}_n(\cdot|\widetilde{s}, \widetilde{a})\right\|_{L_2(\mu), \Sigma_{\rho_n \times \mathcal{U}(\mathcal{A}), \hat{p}_n}^{-1}} \left\| \int_{\mathcal{S}} \sum_{a\in\mathcal{A}} \hat{p}_n(s|\cdot) \pi(a|s) g(s, a) ds\right\|_{L_2(\mu), \Sigma_{\rho_n \times \mathcal{U}(\mathcal{A}), \hat{p}_n}}
    \end{align*}
    Note that
    \begin{align*}
        & \left\| \int_{\mathcal{S}} \sum_{a\in\mathcal{A}} \hat{p}_n(s|\cdot) \pi(a|s) g(s, a) ds\right\|_{L_2(\mu), \Sigma_{\rho_n \times \mathcal{U}(\mathcal{A}), \hat{p}_n}}^2\\
        = & n \mathbb{E}_{\widetilde{s} \sim \rho_n, \widetilde{a} \sim \mathcal{U}(\mathcal{A})}\left\{\mathbb{E}_{s\sim \hat{T}_n(\cdot|\widetilde{s}, \widetilde{a}), a\sim \pi(\cdot|s)}  \left[g(s, a) \right]\right\}^2 + \lambda \left\|\int_{\mathcal{S}} \sum_{a\in\mathcal{A}} \hat{p}_n(s|\cdot) \pi(a|s) g(s, a) ds\right\|_{\mathcal{H}_k}\\
        \leq & n \mathbb{E}_{\widetilde{s} \sim \rho_n, \widetilde{a} \sim \mathcal{U}(\mathcal{A})}\left\{\mathbb{E}_{s\sim T^*(\cdot|\widetilde{s}, \widetilde{a}), a\sim \pi(\cdot|s)}  \left[g(s, a) \right]\right\}^2 + \lambda B^2 C + n B^2 \zeta_n\\
        \leq & n \mathbb{E}_{\widetilde{s} \sim \rho_n, \widetilde{a} \sim \mathcal{U}(\mathcal{A}), s\sim T^*(\cdot|\widetilde{s}, \widetilde{a}), a\sim \pi(\cdot|s)}  \left[g^2(s, a) \right] + \lambda B^2 C + n B^2 \zeta_n\\
        \leq & n|\mathcal{A}|\mathbb{E}_{\widetilde{s} \sim \rho_n, \widetilde{a} \sim \mathcal{U}(\mathcal{A}), s\sim T^*(\cdot|\widetilde{s}, \widetilde{a}), a\sim \mathcal{U}(\mathcal{A})}  \left[g^2(s, a) \right] + \lambda B^2 C + n B^2 \zeta_n\\
        = & n|\mathcal{A}|\mathbb{E}_{s\sim \rho_n^\prime, a\sim \mathcal{U}(\mathcal{A})}  \left[g^2(s, a) \right] + \lambda B^2 C + n B^2 \zeta_n
    \end{align*}
    Substitute back, we obtain the desired result.
\end{proof}

\begin{lemma}[One Step Back Inequality for the True Model]
\label{lem:back_true_online}
    Assume $g:\mathcal{S} \times \mathcal{A} \to \mathbb{R}$ such that $\|g\|_{\infty} \leq B$, then
    \begin{align*}
        \mathbb{E}_{(s, a)\sim d_{T^*}^\pi[g(s, a)]} \leq & \mathbb{E}_{(\widetilde{s}, \widetilde{a})\sim d_{T^*}^\pi} \left\|p^*(\cdot|\widetilde{s}, \widetilde{a})\right\|_{L_2(\mu), \Sigma_{\rho_n, p^*}^{-1}} \sqrt{n\gamma|\mathcal{A}|\mathbb{E}_{s\sim \rho_n, a\sim \mathcal{U}(\mathcal{A})}  \left[g^2(s, a) \right] + \lambda \gamma^2 B^2 C}\\
        & + \sqrt{(1-\gamma)|\mathcal{A}| \mathbb{E}_{s\sim \rho_n, a\sim \mathcal{U}(\mathcal{A})} [g^2(s, a)]}.
    \end{align*}
\end{lemma}
\begin{proof}
    By the property of the stationary distribution, we have
    \begin{align}
        \mathbb{E}_{(s, a) \sim d_{T^*}^\pi} [g(s, a)] = \gamma \mathbb{E}_{(\widetilde{s}, \widetilde{a}) \sim d_{T^*}^\pi, s \sim T^*(\cdot|\widetilde{s}, \widetilde{a}), a\sim \pi(\cdot|s)} [g(s, a)] + (1-\gamma) \mathbb{E}_{s\sim d_0, a\sim \pi(\cdot|s)} [g(s, a)].
    \end{align}
    For the second term, we still use the following upper bound:
    \begin{align*}
        (1-\gamma) \mathbb{E}_{s\sim d_0, a\sim \pi(\cdot|s)} [g(s, a)] \leq \sqrt{(1-\gamma) |\mathcal{A}|\mathbb{E}_{s\sim \rho_n, a\sim \mathcal{U}(\mathcal{A})} [g^2(s, a)]}.
    \end{align*}
    For the first term, with the Cauchy-Schwartz inequality, we have
    \begin{align*}
        & \gamma \mathbb{E}_{(\widetilde{s}, \widetilde{a}) \sim d_{T^*}^\pi, s \sim T^*(\cdot\widetilde{s}, \widetilde{a}), a\sim \pi(\cdot|s)} [g(s, a)]\\
        = & \gamma \mathbb{E}_{(\widetilde{s}, \widetilde{a})\sim d_{T^*}^\pi} \left\langle p^*(\cdot|\widetilde{s}, \widetilde{a}), \int_{\mathcal{S}} \sum_{a\in \mathcal{A}} p^*(s|\cdot) \pi(a|s) g(s, a) ds\right\rangle_{L_2(\mu)}\\
        \leq & \gamma \mathbb{E}_{(\widetilde{s}, \widetilde{a})\sim d_{T^*}^\pi} \left\|p^*(\cdot|\widetilde{s}, \widetilde{a})\right\|_{L_2(\mu), \Sigma_{\rho_n, p^*}^{-1}} \left\| \int_{\mathcal{S}} \sum_{a\in\mathcal{A}} p^*(s|\cdot) \pi(a|s) g(s, a) ds\right\|_{L_2(\mu), \Sigma_{\rho_n, p^*}^2}.
    \end{align*}
    Note that
    \begin{align*}
        & \left\| \int_{\mathcal{S}} \sum_{a\in\mathcal{A}} p^*(s|\cdot) \pi(a|s) g(s, a) ds\right\|_{L_2(\mu), \Sigma_{\rho_n, p^*}^2}^2\\
        = & n \mathbb{E}_{(\widetilde{s},\widetilde{a}) \sim \rho_n}\left\{\mathbb{E}_{s\sim T^*(\cdot|\widetilde{s}, \widetilde{a}), a\sim \pi(\cdot|s)}  \left[g(s, a) \right]\right\}^2 + \lambda \left\|\int_{\mathcal{S}} \sum_{a\in\mathcal{A}} \hat{p}_n(s|\cdot) \pi(a|s) g(s, a) ds\right\|_{\mathcal{H}_k}\\
        \leq & n \mathbb{E}_{(\widetilde{s}, \widetilde{a}) \sim \rho_n, s\sim T^*(\cdot|\widetilde{s}, \widetilde{a}), a\sim \pi(\cdot|s)}  \left[g^2(s, a) \right] + \lambda B^2 C\\
        \leq & n|\mathcal{A}|\mathbb{E}_{(\widetilde{s}, \widetilde{a}) \sim \rho_n, s\sim T^*(\cdot|\widetilde{s}, \widetilde{a}), a\sim \mathcal{U}(\mathcal{A})}  \left[g^2(s, a) \right] + \lambda B^2 C \\
        = & \frac{n|\mathcal{A}|}{\gamma}\mathbb{E}_{s\sim \rho_n, a\sim \mathcal{U}(\mathcal{A})}  \left[g^2(s, a) \right] + \lambda B^2 C
    \end{align*}
    Substitute back, we obtain the desired result.
\end{proof}

\begin{lemma}[Almost Optimism at the Initial Distribution]
\label{lem:optimism} 
Consider an episode $n \in [N]$, if we set $\zeta_n = \Theta\left(\frac{\log (n|\mathcal{P}|/\delta)}{n}\right)$ (such that the MLE generalization bound holds by Lemma~\ref{lem:mle}), $\lambda$ for different eigendecay condition as follows:
\begin{itemize}
    \item $\beta$-finite spectrum: $\lambda = \Theta(\beta\log N + \log (N|\mathcal{P}|/\delta))$
    \item $\beta$-polynomial decay: $\lambda = \Theta(C_{\mathrm{poly}}N^{1/(1 + \beta)} + \log (N|\mathcal{P}|/\delta))$;
    \item $\beta$-exponential decay: $\lambda = \Theta(C_{\mathrm{exp}}(\log N)^{1/\beta} + \log (N|\mathcal{P}|/\delta))$;
\end{itemize}
and $\alpha_n = \Theta\left(\frac{\gamma}{1-\gamma}\sqrt{|\mathcal{A}|\log(n|\mathcal{P}|/\delta) + \lambda C}\right)$, the following events hold with probability at least $1-\delta$:
\begin{align*}
    \forall n \in [N], \quad \forall \pi, \quad V_{\hat{T}_n, r + \hat{b}_n}^\pi - V_{T^*, r}^\pi \geq - \sqrt{\frac{|\mathcal{A}|\zeta_n}{(1-\gamma)^3}}.
\end{align*}
\end{lemma}
\begin{proof}
    With Lemma~\ref{lem:bonus_concentration} and a union bound over $\mathcal{P}$, we know using the chosen $\lambda$, $\forall \hat{T}_n \in \mathcal{P}$, with probability at least $1-\delta$, 
    \begin{align*}
        \left\|\hat{p}_n(\cdot|s, a)\right\|_{L_2(\mu), \hat{\Sigma}_{n, \hat{p}_n}^{-1}} = \Theta\left(\left\|\hat{p}_n(\cdot|s, a)\right\|_{L_2(\mu), \Sigma_{\rho_n \times \mathcal{U}(\mathcal{A}), \hat{p}_n}}^{-1}\right).
    \end{align*} 
    With Lemma~\ref{lem:simulation}, we have that
    \begin{align*}
        & (1-\gamma)\left(V_{\hat{T}_n, r + \hat{b}_n}^\pi - V_{T^*, r}^\pi\right)\\
        = & \mathbb{E}_{(s, a)\sim d_{\hat{T}_n}^\pi}\left[b_n(s, a) + \gamma \mathbb{E}_{\hat{T}_n(s^\prime|s, a)}\left[V_{T, r}^{\pi}(s^\prime)\right] - \gamma\mathbb{E}_{P(s^\prime|s, a)}\left[V_{T, r}^\pi(s^\prime)\right]\right]\\
        \succsim & \mathbb{E}_{(s, a) \sim d_{\hat{T}_n}^\pi} \left[\min\left\{\alpha_n \left\|\hat{p}_n(\cdot|s, a)\right\|_{L_2(\mu), \Sigma_{\rho_n\times \mathcal{U}(\mathcal{A}), \hat{p}_n}^{-1}}, 2\right\} + \gamma \mathbb{E}_{\hat{T}_n(s^\prime|s, a)}\left[V_{T^*, r}^{\pi}(s^\prime)\right] - \gamma\mathbb{E}_{T^*(s^\prime|s, a)}\left[V_{T^*, r}^\pi(s^\prime)\right]\right].
    \end{align*}
    Denote $f_n(s, a) = \mathrm{TV}(T^*(s^\prime|s, a), \hat{T}_n(s^\prime|s, a))$ with $\|f_n\|_{\infty} \leq 2$, with H\"older's inequality, we have that
    \begin{align*}
         \left|\mathbb{E}_{(s, a) \sim d_{\hat{T}_n}^\pi}\left[\mathbb{E}_{\hat{T}_n(s^\prime|s, a)}\left[V_{T, r}^{\pi}(s^\prime)\right] - \mathbb{E}_{T^*(s^\prime|s, a)}\left[V_{T, r}^\pi(s^\prime)\right]\right]\right| \leq   \mathbb{E}_{(s, a) \sim d_{\hat{T}_n}^\pi} \left[\frac{f_n(s, a)}{1-\gamma}\right].
    \end{align*}
    With Lemma~\ref{lem:back_learned_online}, we have that
    \begin{align*}
        & \mathbb{E}_{(s, a) \sim d_{\hat{T}_n}^\pi} \left[\frac{f_n(s, a)}{1-\gamma}\right] \\
        \leq & \mathbb{E}_{(\widetilde{s}, \widetilde{a}) \sim d_{\hat{T}_n}^\pi} \left\|\hat{p}_n(\cdot|s, a)\right\|_{L_2(\mu), \Sigma_{\rho_n\times \mathcal{U}(\mathcal{A}), \hat{p}_n}^{-1}} \sqrt{\frac{n\gamma^2|\mathcal{A}|\mathbb{E}_{s\sim\rho_n^\prime, a\sim\mathcal{U}(\mathcal{A})}\left[f_n^2(s, a)\right]}{(1-\gamma)^2} + \frac{4\lambda \gamma^2 C}{(1-\gamma)^2} + \frac{4n\gamma^2\zeta_n}{(1-\gamma)^2} } \\
        & + \sqrt{\frac{|\mathcal{A}|\mathbb{E}_{s\sim\rho_n, a\sim\mathcal{U}(\mathcal{A})}\left[f_n^2(s, a)\right]}{1-\gamma}}\\
        \leq & \mathbb{E}_{(\widetilde{s}, \widetilde{a}) \sim d_{\hat{T}_n}^\pi} \left\|\hat{p}_n(\cdot|s, a)\right\|_{L_2(\mu), \Sigma_{\rho_n\times \mathcal{U}(\mathcal{A}), \hat{p}_n}^{-1}} \sqrt{\frac{n\gamma^2|\mathcal{A}|\zeta_n}{(1-\gamma)^2} + \frac{4\lambda \gamma^2C}{(1-\gamma)^2} + \frac{4 \gamma^2n \zeta_n}{(1-\gamma)^2}} + \sqrt{\frac{|\mathcal{A}|\zeta_n}{1-\gamma}}
    \end{align*}
    Note that, we set $\alpha_n$ such that
    \begin{align*}
         \sqrt{\frac{n\gamma^2|\mathcal{A}|\zeta_n}{(1-\gamma)^2} + \frac{4\lambda \gamma^2 C}{(1-\gamma)^2} + \frac{4 n \gamma^2 \zeta_n}{(1-\gamma)^2}} \lesssim \alpha_n,
    \end{align*}
    which concludes the proof.
\end{proof}
\begin{lemma}[Regret]
    \label{lem:regret}
    With probability at least $1-\delta$, we have that 
    \begin{itemize}
    \item For $\beta$-finite spectrum, we have
    \begin{align*}
        \sum_{n=1}^N V_{T^*, r}^{\pi^*} - V_{T^*}^{\pi_n} \lesssim \frac{\beta^{3/2}|\mathcal{A}|\sqrt{CN}\log (N |\mathcal{P}|/\delta)}{(1-\gamma)^2}.
    \end{align*}
    \item For $\beta$-polynomial decay, we have
    \begin{align*}
        \sum_{n=1}^N V_{T^*, r}^{\pi^*} - V_{T^*}^{\pi_n} \lesssim \frac{C_{\mathrm{poly}}\sqrt{C} |\mathcal{A}|N^{\frac{1}{2} + \frac{1}{2(1 + \beta)}}  \log (N |\mathcal{P}|/\delta)}{(1-\gamma)^2}.
    \end{align*}
    \item For $\beta$-exponential decay, we have
    \begin{align*}
        \sum_{n=1}^N V_{T^*, r}^{\pi^*} - V_{T^*}^{\pi_n} \lesssim \frac{C_{\mathrm{exp}}|\mathcal{A}| \sqrt{CN} (\log N)^{\frac{3 + 2\beta}{2\beta}} \log (N |\mathcal{P}|/\delta)}{(1-\gamma)^2}.
    \end{align*}
\end{itemize}
\end{lemma}
\begin{proof}
With Lemma~\ref{lem:optimism} and Lemma~\ref{lem:simulation}, we have that
\begin{align*}
    & V_{T^*, r}^{\pi^*} - V_{T^*, r}^{\pi_n}\\
    \leq & V_{\hat{T}_n, r + b_n}^{\pi^*} + \sqrt{\frac{|\mathcal{A}|\zeta_n}{(1-\gamma)^3}} - V_{T^*, r}^{\pi_n}\\
    \leq & V_{\hat{T}_n, r + b_n}^{\pi_n} + \sqrt{\frac{|\mathcal{A}|\zeta_n}{(1-\gamma)^3}}  - V_{T^*, r}^{\pi_n}\\
    \leq & \frac{1}{1-\gamma}\mathbb{E}_{(s, a)\sim d_{T^*}^{\pi_n}}\left[b_n(s, a) + \gamma\mathbb{E}_{\hat{T}_n(s^\prime|s, a)}[V_{\hat{T}_n, r + b_n}^{{\pi_n}}(s^\prime)] - \gamma\mathbb{E}_{T^*(s^\prime|s, a)}[V_{\hat{T}_n, r + b_n}^{\pi_n}(s^\prime)]\right] + \sqrt{\frac{|\mathcal{A}|\zeta_n}{(1-\gamma)^3}}.
\end{align*}
Applying Lemma~\ref{lem:back_true_online} and note that $b_n = O(1)$, we have that
\begin{align*}
    & \mathbb{E}_{(s, a)\sim d_{T^*}^{\pi_n}}[b_n(s, a)]\\
    \lesssim & \mathbb{E}_{(s, a) \sim d_{T^*}^{\pi_n}}\left[\min\left\{\alpha_n \left\|\hat{p}_n(\cdot|s, a)\right\|_{L_2(\mu), \Sigma_{\rho_n\times \mathcal{U}(\mathcal{A}), \hat{p}_n}^{-1}}, 2\right\}\right]\\
    \lesssim & \mathbb{E}_{(\widetilde{s}, \widetilde{a}) \sim d_{T^*}^{\pi_n}} \left\|p^*(\cdot|\widetilde{s}, \widetilde{a})\right\|_{\Sigma_{\rho_n, p^*}^{-1}} \sqrt{n\gamma |\mathcal{A}|\alpha_n^2 \mathbb{E}_{s\sim \rho_n, a\sim\mathcal{U}(\mathcal{A})}\left[\left\|\hat{p}_n(\cdot|s, a)\right\|_{L_2(\mu), \Sigma_{\rho_n\times \mathcal{U}(\mathcal{A}), \hat{p}_n}^{-1}}^2\right] + \lambda \gamma^2 C}\\
    & + \sqrt{(1-\gamma) |\mathcal{A}|\alpha_n^2 \mathbb{E}_{s\sim \rho_n, a\sim\mathcal{U}(\mathcal{A})}\left[\left\|\hat{p}_n(\cdot|s, a)\right\|_{L_2(\mu), \Sigma_{\rho_n\times \mathcal{U}(\mathcal{A}), \hat{p}_n}^{-1}}^2\right]}.
\end{align*}
Note that,
\begin{align*}
     & n\mathbb{E}_{s\sim \rho_n, a\sim\mathcal{U}(\mathcal{A})}\left[\left\|\hat{p}_n(\cdot|s, a)\right\|_{L_2(\mu), \Sigma_{\rho_n\times \mathcal{U}(\mathcal{A}), \hat{p}_n}^{-1}}^2\right]\\
     = & \mathrm{Tr} \left(n\mathbb{E}_{s\sim\rho_n, a\sim\mathcal{U}(\mathcal{A})}\left[\hat{p}_n(\cdot|s, a) \hat{p}_n(\cdot|s, a)^\top \right]\left(n \mathbb{E}_{s\sim\rho_n, a\sim\mathcal{U}(\mathcal{A})}\left[\hat{p}_n(\cdot|s, a) \hat{p}_n(\cdot|s, a)^\top \right] + \lambda T_k^{-1}\right)^{-1}\right)\\
     = & \mathrm{Tr} \left(n\mathbb{E}_{s\sim\rho_n, a\sim\mathcal{U}(\mathcal{A})}\left[T_k^{1/2}\hat{p}_n(\cdot|s, a) \hat{p}_n(\cdot|s, a)^\top T_k^{1/2} \right]\left(n \mathbb{E}_{s\sim\rho_n, a\sim\mathcal{U}(\mathcal{A})}\left[T_k^{1/2}\hat{p}_n(\cdot|s, a) \hat{p}_n(\cdot|s, a)^\top T_k^{1/2} \right] + \lambda I\right)^{-1}\right)\\
     \leq & \log \mathrm{det}\left(I + \frac{n}{\lambda} \mathbb{E}_{s\sim \rho_n, a\sim \mathcal{U}(\mathcal{A})}\left[T_k^{1/2}\hat{p}_n(\cdot|s, a)^\top\hat{p}_n(\cdot|s, a)T_k^{1/2}\right]\right),
\end{align*}
where the first equality is due to the definition of Hilbert-Schmidt inner product and the expectation operator is a linear operator, the second equality is due to the fact that $\mathrm{Tr}(A(A + B)^{-1}) = \mathrm{Tr}\left(\left(B^{-1/2} A B^{-1/2}\right) \left(I + B^{-1/2} A B^{-1/2}\right)^{-1}\right)$ for positive semi-definite operator $A$ and positive definite operator $B$, and the last inequality is due to the fact that if $A$ has the eigensystem $\{\mu_i, e_i\}$, then $A(A + \lambda I)^{-1}$ has the eigensystem $\{\frac{\mu_i}{\mu_i + \lambda}, e_i\}$, and $\frac{x}{1+x} \leq \log(1 + x)$. Here $\mathrm{det}$ denotes the Fredholm determinant. Note that, if $x \in \mathcal{B}_{\mathcal{H}_{k}}$, $T_k^{1/2} x \in \mathcal{B}_{\mathcal{H}_{\widetilde{k}}}$, and we know $\mathbb{E}_{s\sim \rho_n, a\sim \mathcal{U}(\mathcal{A})}\left[T_k^{1/2}\hat{p}_n(\cdot|s, a)^\top\hat{p}_n(\cdot|s, a)T_k^{1/2}\right]$ is in the trace class and the Fredholm determinant is well-defined. Invoking Lemma~\ref{lem:potential_function_RKHS}, we have that
\begin{itemize}
    \item For $\beta$-finite spectrum, as $\lambda = \Theta(\beta \log N + \log (N |\mathcal{P}|/\delta))$, we have $n/\lambda = O(n)$
    \begin{align*}
        \log\mathrm{det}\left(I + \frac{n}{\lambda} \mathbb{E}_{s\sim \rho_n, a\sim \mathcal{U}(\mathcal{A})}\left[T_k^{1/2}\hat{p}_n(\cdot|s, a)^\top\hat{p}_n(\cdot|s, a)T_k^{1/2}\right]\right) = O\left(\beta \log n\right),
    \end{align*}
    which means
    \begin{align*}
        \mathbb{E}_{(s, a)\sim d_{T^*}^{\pi_n}}[b_n(s, a)]\lesssim &\sqrt{\gamma|\mathcal{A}|\alpha_n^2 \beta \log (n)  + \lambda\gamma^2 C}\cdot \mathbb{E}_{(\widetilde{s}, \widetilde{a}) \sim d_{T^*}^{\pi_n}} \left\|p^*(\cdot|\widetilde{s}, \widetilde{a})\right\|_{\Sigma_{\rho_n, p^*}^{-1}} \\
        & + \sqrt{\frac{(1-\gamma)|\mathcal{A}|\alpha_n^2 \beta\log (n)}{n}}.
    \end{align*}
    \item For $\beta$-polynomial decay, as $\lambda = \Theta(C_{\mathrm{poly}}N^{1/(1+\beta)} + \log (N|\mathcal{P}|/\delta))$ and $n \leq N$, we have $n/\lambda = O\left(C_{\mathrm{poly}} n^{\frac{\beta}{1+\beta}}\right)$ and 
    \begin{align*}
        & \log\mathrm{det}\left(I + \frac{n}{\lambda} \mathbb{E}_{s\sim \rho_n, a\sim \mathcal{U}(\mathcal{A})}\left[T_k^{1/2}\hat{p}_n(\cdot|s, a)^\top\hat{p}_n(\cdot|s, a)T_k^{1/2}\right]\right) 
        = O\left(C_{\mathrm{poly}} n^{\frac{1}{2(1 + \beta)}}\log(n)\right),
    \end{align*}
    This leads to
    \begin{align*}
        \mathbb{E}_{(s, a)\sim d_{T^*}^{\pi_n}}[b_n(s, a)]\lesssim &\sqrt{\gamma|\mathcal{A}|C_{\mathrm{poly}}\alpha_n^2 n^{\frac{1}{2(1+\beta)}} \log n  + \lambda\gamma^2 C}\cdot \mathbb{E}_{(\widetilde{s}, \widetilde{a}) \sim d_{T^*}^{\pi_n}} \left\|p^*(\cdot|\widetilde{s}, \widetilde{a})\right\|_{\Sigma_{\rho_n, p^*}^{-1}} \\
        & + \sqrt{(1-\gamma)|\mathcal{A}| C_{\mathrm{poly}} n^{-1 + \frac{1}{2(1 + \beta)}}\log (n)\alpha_n^2}.
    \end{align*}
    \item For $\beta$-exponential decay, as $\lambda = \Theta\left(C_{\mathrm{exp}}(\log N)^{1/\beta} + \log(N|\mathcal{P}|/\delta)\right)$, we have $n/\lambda = O\left(C_{\mathrm{exp}} n\right)$ and
    \begin{align*}
        & \log \mathrm{det} \left(I + \frac{n}{\lambda} \mathbb{E}_{s\sim \rho_n, a\sim \mathcal{U}(\mathcal{A})}\left[T_k^{1/2}\hat{p}_n(\cdot|s, a)^\top\hat{p}_n(\cdot|s, a)T_k^{1/2}\right]\right) 
        = O\left(C_{\mathrm{exp}}(\log n)^{1 + 1/\beta}\right),
    \end{align*}
    This leads to
    \begin{align*}
        \mathbb{E}_{(s, a)\sim d_{T^*}^{\pi_n}}[b_n(s, a)]\lesssim &\sqrt{\gamma |\mathcal{A}|C_{\mathrm{exp}} (\log n)^{1 + 1/\beta}\alpha_n^2 + \lambda \gamma^2 C}\cdot \mathbb{E}_{(\widetilde{s}, \widetilde{a}) \sim d_{T^*}^{\pi_n}} \left\|p^*(\cdot|\widetilde{s}, \widetilde{a})\right\|_{\Sigma_{\rho_n, p^*}^{-1}} \\
        & + \sqrt{\frac{(1-\gamma)|\mathcal{A}|C_{\mathrm{exp}}(\log n)^{1 + 1/\beta}  
       \alpha_n^2}{n}}.
    \end{align*}
\end{itemize}
For the remaining terms, denote $f_n(s, a) = \mathrm{TV}(T^*(s^\prime|s, a), \hat{T}_n(s^\prime)|s, a)$ with $\|f_n\|_{\infty} \leq 2$. With H\"older's inequality, we have
\begin{align*}
    \left|\mathbb{E}_{(s, a) \sim d_{T^*}^{\pi_n}}\left[\mathbb{E}_{\hat{T}_n(s^\prime|s, a)}\left[V_{\hat{T}_n, r + b_n}^{\pi_n}(s^\prime)\right] - \mathbb{E}_{T^*(s^\prime|s, a)}\left[V_{\hat{T}_n, r + b_n}^{\pi_n}(s^\prime)\right]\right]\right| \lesssim \mathbb{E}_{(s, a) \sim d_{T^*}^{\pi_n}} \left[\frac{f_n(s, a)}{1-\gamma}\right].
\end{align*}
With Lemma~\ref{lem:back_true_online}, we have that
\begin{align*}
    \mathbb{E}_{(s, a) \sim d_{T^*}^{\pi_n}} \left[\frac{f_n(s, a)}{1-\gamma}\right]\leq & \mathbb{E}_{(\widetilde{s}, \widetilde{a})\sim d_{T^*}^\pi} \left\|p^*(\cdot|\widetilde{s}, \widetilde{a})\right\|_{L_2(\mu), \Sigma_{\rho_n, p^*}^{-1}} \sqrt{\frac{n\gamma|\mathcal{A}|\mathbb{E}_{s\sim \rho_n, a\sim \mathcal{U}(\mathcal{A})}  \left[f_n^2(s, a) \right]}{(1-\gamma)^2} + \frac{4\lambda \gamma^2 C}{(1-\gamma)^2}}\\
    & + \sqrt{\frac{|\mathcal{A}|\mathbb{E}_{s\sim \rho_n, a\sim \mathcal{U}(\mathcal{A})} [f_n^2(s, a)]}{1-\gamma}}\\
    \leq&  \mathbb{E}_{(\widetilde{s}, \widetilde{a})\sim d_{T^*}^\pi} \left\|p^*(\cdot|\widetilde{s}, \widetilde{a})\right\|_{L_2(\mu), \Sigma_{\rho_n, p^*}^{-1}} \cdot \sqrt{\frac{n\gamma|\mathcal{A}|\zeta_n}{(1-\gamma)^2} + \frac{4\lambda \gamma^2 C}{(1-\gamma)^2}} + \sqrt{\frac{|\mathcal{A}|\zeta_n}{1-\gamma}}\\
    \lesssim & \alpha_n \mathbb{E}_{(\widetilde{s}, \widetilde{a})\sim d_{T^*}^\pi} \left\|p^*(\cdot|\widetilde{s}, \widetilde{a})\right\|_{L_2(\mu), \Sigma_{\rho_n, p^*}^{-1}} + \sqrt{\frac{|\mathcal{A}|\zeta_n}{1-\gamma}}
\end{align*}
Combine with the previous results and take the dominating terms out, we have that
\begin{itemize}
    \item For $\beta$-finite spectrum,
    \begin{align*}
         V_{T^*, r}^{\pi^*} - V_{T^*}^{\pi_n} \\
        \lesssim & \frac{1}{1-\gamma} \sqrt{\gamma |\mathcal{A}| \alpha_n^2\beta \log n + \lambda \gamma^2 C}\cdot \mathbb{E}_{(\widetilde{s}, \widetilde{a}) \sim d_{T^*}^{\pi_n}} \left\|p^*(\cdot|\widetilde{s}, \widetilde{a})\right\|_{\Sigma_{\rho_n, p^*}^{-1}} \\
        & + \sqrt{\frac{|\mathcal{A}|\alpha_n^2 \beta \log n}{(1-\gamma) n}} +  \sqrt{\frac{|\mathcal{A}|\zeta_n}{(1-\gamma)^3}}.
    \end{align*}
    \item For $\beta$-polynomial decay, 
    \begin{align*}
        V_{T^*, r}^{\pi^*} - V_{T^*}^{\pi_n} \\
        \lesssim & \frac{1}{1-\gamma} \sqrt{\gamma |\mathcal{A}| C_{\mathrm{poly}}\alpha_n^2n^{\frac{1}{2(1+\beta)}} \log n + \lambda \gamma^2 C}\cdot \mathbb{E}_{(\widetilde{s}, \widetilde{a}) \sim d_{T^*}^{\pi_n}} \left\|p^*(\cdot|\widetilde{s}, \widetilde{a})\right\|_{\Sigma_{\rho_n, p^*}^{-1}} \\
        & + \sqrt{\frac{|\mathcal{A}|C_{\mathrm{poly}}\alpha_n^2 n^{-1 + \frac{1}{2(1 + \beta)}}\log n}{1-\gamma}} +  \sqrt{\frac{|\mathcal{A}|\zeta_n}{(1-\gamma)^3}}.
    \end{align*}
    \item For $\beta$-exponential decay,
    \begin{align*}
        V_{T^*, r}^{\pi^*} - V_{T^*}^{\pi_n} \\
        \lesssim & \frac{1}{1-\gamma}\sqrt{\gamma |\mathcal{A}|C_{\mathrm{exp}}\alpha_n^2(\log n)^{1 + 1/\beta} + \lambda \gamma^2 C}\cdot \mathbb{E}_{(\widetilde{s}, \widetilde{a}) \sim d_{T^*}^{\pi_n}} \left\|p^*(\cdot|\widetilde{s}, \widetilde{a})\right\|_{\Sigma_{\rho_n, p^*}^{-1}} \\
        & + \sqrt{\frac{|\mathcal{A}|C_{\mathrm{exp}} \alpha_n^2(\log n)^{1 + 1/\beta}  }{(1-\gamma)n}} +  \sqrt{\frac{|\mathcal{A}|\zeta_n}{(1-\gamma)^3}}.
    \end{align*}
\end{itemize}

Finally, with Cauchy-Schwartz inequality, we have
\begin{align*}
    \sum_{n=1}^N \mathbb{E}_{(\widetilde{s}, \widetilde{a}) \sim d_{T^*}^{\pi_n}} \left\|p^*(\cdot|\widetilde{s}, \widetilde{a})\right\|_{L_2(\mu), \Sigma_{\rho_n, p^*}^{-1}} \leq & \sqrt{N \sum_{n=1}^N\mathbb{E}_{(\widetilde{s}, \widetilde{a}) \sim d_{T^*}^{\pi_n}} \left\langle p^*(\cdot|\widetilde{s}, \widetilde{a}), \Sigma_{\rho_n, p^*}^{-1}p^*(\cdot|\widetilde{s}, \widetilde{a})\right\rangle_{L_2(\mu)}}.
\end{align*}
Note that
\begin{align*}
    & \mathbb{E}_{(\widetilde{s}, \widetilde{a}) \sim d_{T^*}^{\pi_n}} \left\langle p^*(\cdot|\widetilde{s}, \widetilde{a}), \Sigma_{\rho_n, p^*}^{-1}p^*(\cdot|\widetilde{s}, \widetilde{a})\right\rangle_{L_2(\mu)} \\
    = & \mathbb{E}_{(\widetilde{s}, \widetilde{a}) \sim d_{T^*}^{\pi_n}} \left\langle p^*(\cdot|\widetilde{s}, \widetilde{a}), \left(n \mathbb{E}_{(s, a) \sim \rho_n}\left[p^*(\cdot|s, a) p^*(\cdot|s, a)^\top\right] + \lambda T_k^{-1}\right)^{-1}p^*(\cdot|\widetilde{s}, \widetilde{a})\right\rangle_{L_2(\mu)} \\
    = & \mathbb{E}_{(\widetilde{s}, \widetilde{a}) \sim d_{T^*}^{\pi_n}} \left\langle T_k^{1/2}p^*(\cdot|\widetilde{s}, \widetilde{a}), \left(n \mathbb{E}_{(s, a) \sim \rho_n}\left[T_k^{1/2}p^*(\cdot|s, a) p^*(\cdot|s, a)^\top T_k^{1/2}\right] + \lambda I\right)^{-1} T_k^{1/2}p^*(\cdot|\widetilde{s}, \widetilde{a})\right\rangle_{L_2(\mu)}\\
    = &  \mathrm{Tr}\left(\left(\frac{n}{\lambda} \mathbb{E}_{(s, a) \sim \rho_n}\left[T_k^{1/2}p^*(\cdot|s, a) p^*(\cdot|s, a)^\top T_k^{1/2}\right] + I\right)^{-1}, \frac{\mathbb{E}_{(\widetilde{s}, \widetilde{a}) \sim d_{T^*}^{\pi_n}}\left[T_k^{1/2}p^*(\cdot|\widetilde{s}, \widetilde{a}) p^*(\cdot|\widetilde{s}, \widetilde{a}) T_k^{1/2}\right]}{\lambda}\right)\\
    \leq & \log \mathrm{det}\left(\left(\frac{n}{\lambda} \mathbb{E}_{(s, a) \sim \rho_{n}}\left[T_k^{1/2}p^*(\cdot|s, a) p^*(\cdot|s, a)^\top T_k^{1/2}\right] + I\right)\right) \\
    & - \log \mathrm{det}\left(\left(\frac{n-1}{\lambda} \mathbb{E}_{(s, a) \sim \rho_{n-1}}\left[T_k^{1/2}p^*(\cdot|s, a) p^*(\cdot|s, a)^\top T_k^{1/2}\right] + I\right)\right),
\end{align*}
where in the last inequality, we use the fact that $\log\mathrm{det}(X)$ is concave with positive definite operators $X$ and $\frac{d \log\mathrm{det}(X)}{dX} = (X^\top)^{-1}$.
% \Tongzheng{See if we need to become much more detailed. Shall we refer to our spectral paper?}

Telescoping and applying Lemma~\ref{lem:potential_function_RKHS}, we have that:
\begin{itemize}
    \item For $\beta$-finite spectrum: as $N/\lambda = O(N)$, we have
        \begin{align*}
            \sum_{n=1}^N\mathbb{E}_{(\widetilde{s}, \widetilde{a}) \sim d_{T^*}^{\pi_n}} \left\langle p^*(\cdot|\widetilde{s}, \widetilde{a}), \Sigma_{\rho_n, p^*}^{-1}p^*(\cdot|\widetilde{s}, \widetilde{a})\right\rangle_{L_2(\mu)} = O(\beta \log N).
        \end{align*}
    \item For $\beta$-polynomial decay: as $\lambda = \Theta(C_{\mathrm{poly}}N^{1/(1+\beta)} + \log (N|\mathcal{P}|/\delta))$, we have $N/\lambda = O\left(C_{\mathrm{poly}} N^{\frac{\beta}{1+\beta}}\right)$ and 
        \begin{align*}
            \sum_{n=1}^N\mathbb{E}_{(\widetilde{s}, \widetilde{a}) \sim d_{T^*}^{\pi_n}} \left\langle p^*(\cdot|\widetilde{s}, \widetilde{a}), \Sigma_{\rho_n, p^*}^{-1}p^*(\cdot|\widetilde{s}, \widetilde{a})\right\rangle_{L_2(\mu)} = O\left(C_{\mathrm{poly}} N^{\frac{1}{2(1 + \beta)}}\log N\right).
        \end{align*}
    \item For $\beta$-exponential decay:
        \begin{align*}
            \sum_{n=1}^N\mathbb{E}_{(\widetilde{s}, \widetilde{a}) \sim d_{T^*}^{\pi_n}} \left\langle p^*(\cdot|\widetilde{s}, \widetilde{a}), \Sigma_{\rho_n, p^*}^{-1}p^*(\cdot|\widetilde{s}, \widetilde{a})\right\rangle_{L_2(\mu)} = O\left(C_{\mathrm{exp}}(\log N)^{1 + 1/\beta}\right).
        \end{align*}
\end{itemize}

Hence, after we substitute $\alpha_n$ and $\lambda$ back and take the dominating term out, we can conclude that:
\begin{itemize}
    \item For $\beta$-finite spectrum, we have
    \begin{align*}
        \sum_{n=1}^N V_{T^*, r}^{\pi^*} - V_{T^*}^{\pi_n}  \lesssim \frac{\beta^{3/2} |\mathcal{A}|\log N\sqrt{C N \log (N|\mathcal{P}|/\delta)}}{(1-\gamma)^2}.
    \end{align*}
    \item For $\beta$-polynomial decay, we have
    \begin{align*}
        \sum_{n=1}^N V_{T^*, r}^{\pi^*} - V_{T^*}^{\pi_n} \lesssim \frac{C_{\mathrm{poly}}|\mathcal{A}|N^{\frac{1}{2} + \frac{1}{1 + \beta}} \log N \sqrt{C \log (N |\mathcal{P}|/\delta)}}{(1-\gamma)^2}.
    \end{align*}
    \item For $\beta$-exponential decay, we have
    \begin{align*}
        \sum_{n=1}^N V_{T^*, r}^{\pi^*} - V_{T^*}^{\pi_n} \lesssim \frac{C_{\mathrm{exp}}|\mathcal{A}| \sqrt{CN \log (N|\mathcal{P}|/\delta)} (\log N)^{\frac{3 + 2\beta}{2\beta}}  }{(1-\gamma)^2}.
    \end{align*}
\end{itemize}
This finishes the proof.
\end{proof}
\begin{theorem}[PAC Guarantee for Online Setting]
    \label{thm:pac_guarantee_online}
    After interacting with the environments for $N$ episodes where 
    \begin{itemize}
        \item $N =  \Theta\left(\frac{C \beta^3 |\mathcal{A}|^2 \log (|\mathcal{P}|/\delta)}{(1-\gamma)^4\varepsilon^2}\log^3 \left(\frac{C \beta^3 |\mathcal{A}|^2 \log (|\mathcal{P}|/\delta)}{(1-\gamma)^4\varepsilon^2}\right)\right)$ for $\beta$-finite spectrum;
        \item $N = \Theta\left(C_{\mathrm{poly}}\left(\frac{|\mathcal{A}|\sqrt{C\log (|\mathcal{P}|/\delta)}}{(1-\gamma)^2\varepsilon} \log^{3/2}\left(\frac{\sqrt{C}|\mathcal{A}|\log (|\mathcal{P}|/\delta)}{(1-\gamma)^2\varepsilon} \right)\right)^{\frac{2(1+\beta)}{\beta - 1}}\right)$ for $\beta$-polynomial decay;
        \item $N =\Theta\left(\frac{C_{\mathrm{exp}}C|\mathcal{A}|^2\log (|\mathcal{P}|/\delta)}{(1-\gamma)^4\varepsilon^2} \log^{\frac{3 + 2\beta}{\beta}}\left(\frac{C|\mathcal{A}|^2\log (|\mathcal{P}|/\delta)}{(1-\gamma)^4\varepsilon^2} \right)\right)$ for $\beta$-exponential decay;
    \end{itemize} 
    we can obtain an $\varepsilon$-optimal policy with high probability.
\end{theorem}
\begin{proof} 
    Note that, $\log \log x = O(\log x)$. We consider the case with different eigendecay conditions separately.
    \begin{itemize}
        \item For $\beta$-finite spectrum, by taking the output policy as the uniform policy over $\{\pi_i\}_{i\in [n]}$, we can obtain a policy with the sub-optimality gap 
        \begin{align*}
            O\left(\frac{\beta^{3/2}|\mathcal{A}|\log N\sqrt{C\log (N|\mathcal{P}|/\delta)}}{(1-\gamma)^2\sqrt{N}}\right).
        \end{align*}
        Take $N = \Theta\left(\frac{C \beta^3 |\mathcal{A}|^2 \log (|\mathcal{P}|/\delta)}{(1-\gamma)^4\varepsilon^2}\log^3 \left(\frac{C \beta^3 |\mathcal{A}|^2 \log (|\mathcal{P}|/\delta)}{(1-\gamma)^4\varepsilon^2}\right)\right)$, we can see the sub-optimality gap is smaller than $\varepsilon$, which finishes the proof for $\beta$-finite spectrum.
        \item For $\beta$-polynomial decay, by taking the output policy as the uniform policy over $\{\pi_i\}_{i\in[n]}$, we can obtain a policy with the sub-optimality gap
        \begin{align*}
            O\left(\frac{C_{\mathrm{poly}} N^{\frac{\beta - 1}{2(1+\beta)}} \log N\sqrt{C \log (N|\mathcal{P}|/\delta)}}{(1-\gamma)^2}\right).
        \end{align*}
        Take $N = \Theta\left(C_{\mathrm{poly}}\left(\frac{|\mathcal{A}|\sqrt{C\log (|\mathcal{P}|/\delta)}}{(1-\gamma)^2\varepsilon} \log^{3/2}\left(\frac{\sqrt{C}|\mathcal{A}|\log (|\mathcal{P}|/\delta)}{(1-\gamma)^2\varepsilon} \right)\right)^{\frac{2(1+\beta)}{\beta - 1}}\right)$, we can see the sub-optimality gap is smaller than $\varepsilon$, which finishes the proof for $\beta$-exponential decay.
        \item For $\beta$-exponential decay, by taking the output policy as the uniform policy over $\{\pi_i\}_{i\in [n]}$, we can obtain a policy with the sub-optimality gap
        \begin{align*}
        O\left(\frac{C_{\mathrm{exp}}|\mathcal{A}|\sqrt{C\log(N|\mathcal{P}|/\delta)}(\log N)^{\frac{3 + 2\beta}{2\beta}}}{(1-\gamma)^2 \sqrt{N}}\right).
        \end{align*}
        Take $N = \Theta\left(\frac{C_{\mathrm{exp}}C|\mathcal{A}|^2\log (|\mathcal{P}|/\delta)}{(1-\gamma)^4\varepsilon^2} \log^{\frac{3 + 2\beta}{\beta}}\left(\frac{C|\mathcal{A}|^2\log (|\mathcal{P}|/\delta)}{(1-\gamma)^4\varepsilon^2} \right)\right)$, we can see the sub-optimality gap is smaller than $\varepsilon$, which finishes the proof for $\beta$-exponential decay.
    \end{itemize}
    As a result, we finish the proof for the PAC guarantee.
\end{proof}
\subsection{Proof for the Offline Setting}
\label{sec:offline_proof}
Similar to the online exploration case, we can obtain the upper bound of the statistical error for $\hat{\pi}$, which is stated in the following:
\begin{theorem}[PAC Guarantee for Offline Exploitation]
\label{thm:pac_offline}
Define $\omega := \max_{s, a} \pi_b^{-1}(a|s)$, and 
\begin{align*}
    C_{\pi}^* := \sup_{x\in L_2(\mu)} \frac{\mathbb{E}_{(s, a) \sim d_{T^*}^\pi} \left[\left\langle p^*(\cdot|s, a), x\right\rangle_{L_2(\mu)}\right]^2}{\mathbb{E}_{(s, a) \sim \rho} \left[\left\langle p^*(\cdot|s, a), x\right\rangle_{L_2(\mu)}\right]^2}.
\end{align*}
If the penalty and its corresponding parameters are identical to the bonus we define in Theorem~\ref{thm:pac_online}, then with probability at least $1-\delta$, for any competitor policy $\pi$ including non-Markovian history-dependent policy, we have
\begin{itemize}
    \item For $\beta$-finite spectrum, we have
    \begin{align*}
        V_{T^*, r}^{\pi} - V_{T^*, r}^{\hat{\pi}} \lesssim \frac{\omega \beta^{3/2}\log n}{(1-\gamma)^2} \sqrt{\frac{C C_{\pi}^* \log (|\mathcal{P}|/\delta)}{n}}
    \end{align*}
    \item For $\beta$-polynomial decay, we have
    \begin{align*}
        V_{T^*, r}^{\pi} - V_{T^*, r}^{\hat{\pi}} \lesssim \frac{C_{\mathrm{poly}} \omega n^{\frac{1-\beta}{2(1+\beta)}}\log n\sqrt{C C_{\pi}^* \log (|\mathcal{P}|/\delta)}}{(1-\gamma)^2}
    \end{align*}
    \item For $\beta$-exponential decay, we have
    \begin{align*}
    V_{T^*, r}^{\pi} - V_{T^*, r}^{\hat{\pi}} \lesssim \frac{C_{\mathrm{exp}} \omega (\log n)^{\frac{3 + 2\beta}{2\beta}}}{(1-\gamma)^2} \sqrt{\frac{C C_{\pi}^* \log (|\mathcal{P}|/\delta)}{n}}
    \end{align*}
\end{itemize} 
\end{theorem}

We start by showing that $C_{\pi}^*$ can be viewed as a measure of the offline data quality, which can be demonstrated by the following lemma, that was first introduced in \citet{chang2021mitigating}:
\begin{lemma}[Distribution Shift Lemma]
    \label{lem:distribution_shift}
    For any positive definite operator $\Lambda:L_2(\mu) \to L_2(\mu)$, we have that
    \begin{align*}
        \mathbb{E}_{(s, a) \sim d_{T^*}^\pi} \langle p^*(\cdot|s, a), \Lambda p^*(\cdot|s, a)\rangle_{L_2(\mu)} \leq C_{\pi}^* \mathbb{E}_{(s, a) \sim \rho} \langle p^*(\cdot|s, a), \Lambda p^*(\cdot|s, a)\rangle_{L_2(\mu)}.
    \end{align*}
\end{lemma}
\begin{proof}
    We denote the eigendecomposition of $\Lambda$ as $\Lambda = U \Sigma U$ where $\{\sigma_i, u_i\}$ is the eigensystem of $\Lambda$. Then we have
    \begin{align*}
        & \mathbb{E}_{(s, a) \sim d_{T^*}^\pi} \langle p^*(\cdot|s, a), \Lambda p^*(\cdot|s, a)\rangle_{L_2(\mu)} \\
        = & \sum_{i\in I} \sigma_i \mathbb{E}_{(s, a) \sim d_{T^*}^\pi} \langle u_i, p^*(\cdot|s, a)^\top\rangle_{L_2(\mu)}^2\\
        \leq & C_{\pi}^* \sum_{i\in I} \sigma_i \mathbb{E}_{(s, a) \sim \rho} \langle u_i, p^*(\cdot|s, a)^\top\rangle_{L_2(\mu)}^2\\
        = & C_{\pi}^* \mathbb{E}_{(s, a) \sim \rho} \langle p^*(\cdot|s, a), \Lambda p^*(\cdot|s, a)\rangle_{L_2(\mu)},
    \end{align*}
    which finishes the proof.
\end{proof}

We also define the $\Sigma_{\rho, \phi}:L_2(\mu) \to L_2(\mu)$:
\begin{align*}
    \Sigma_{\rho, \phi} := n\mathbb{E}_{(s, a) \sim \rho} \left[\phi(s, a) \phi^\top(s, a)\right] + \lambda T_k^{-1},
\end{align*}
where $\rho$ is the stationary distribution of $\pi_b$.
\begin{lemma}[One Step Back Inequality for the Learned Model in Offline Setting]
\label{lem:back_learned_offline}
    Assume $g:\mathcal{S}\times \mathcal{A} \to \mathbb{R}$, such that $\|g\|_{\infty} \leq B$. Then conditioning on the following generalization bound:
    \begin{align*}
        \mathbb{E}_{(s, a) \sim \rho} \|\hat{T}(s, a) - T^*(s, a)\|_1^2 \leq \zeta,
    \end{align*}
    we have that $\forall \pi$
    \begin{align*}
        & \left|\mathbb{E}_{(s, a) \sim d_{\hat{T}}^\pi}[g(s, a)]\right| \\
        \leq & \gamma \mathbb{E}_{(\widetilde{s}, \widetilde{a})\sim d_{\hat{T}}^\pi} \left\|\hat{p}(\cdot|\widetilde{s}, \widetilde{a})\right\|_{L_2(\mu), \Sigma_{\rho , \hat{p}}^{-1}} \sqrt{ n\omega\gamma\mathbb{E}_{(s, a) \sim \rho}  \left[g^2(s, a) \right] + \lambda  B^2 C + n B^2 \zeta}\\
        & + \sqrt{(1-\gamma) \omega\mathbb{E}_{(s, a) \sim \rho} [g^2(s, a)]}.
    \end{align*}
\end{lemma}
\begin{proof}
    We still start from the following inequality:
    \begin{align*}
         \mathbb{E}_{(s, a) \sim d_{\hat{T}}^\pi} [g(s, a)] = \gamma \mathbb{E}_{(\widetilde{s}, \widetilde{a}) \sim d_{\hat{T}}^\pi, s \sim \hat{p}_n(\cdot|\widetilde{s}, \widetilde{a}), a\sim \pi(\cdot|s)} [g(s, a)] + (1-\gamma) \mathbb{E}_{s\sim d_0, a\sim \pi(\cdot|s)} [g(s, a)].
    \end{align*}
    For the second term, with Jensen's inequality and an importance sampling step, we have that
    \begin{align*}
        (1-\gamma) \mathbb{E}_{s\sim d_0, a\sim \pi(\cdot|s)} [g(s, a)] \leq \sqrt{(1-\gamma)\omega \mathbb{E}_{(s, a)\sim \rho}[g^2(s, a)]}.
    \end{align*}
    For the first term, with Cauchy-Schwartz inequality of $L_2(\mu)$ inner product, we have that
    \begin{align*}
        & \gamma\mathbb{E}_{(\widetilde{s}, \widetilde{a}) \sim d_{\hat{T}}^\pi, s \sim \hat{T}(\cdot|\widetilde{s}, \widetilde{a}), a\sim \pi(\cdot|s)} [g(s, a)]\\
        = & \gamma \mathbb{E}_{(\widetilde{s}, \widetilde{a}) \sim d_{\hat{T}}^\pi} \left\langle \hat{p}(\cdot|\widetilde{s}, \widetilde{a}), \int_{\mathcal{S}} \sum_{a\in\mathcal{A}} \hat{p}(s|\cdot)\pi(a|s)g(s, a) ds \right\rangle_{L_2(\mu)}\\
        \leq & \gamma \mathbb{E}_{(\widetilde{s}, \widetilde{a}) \sim d_{\hat{T}}^\pi} \left\|\hat{p}(\cdot|\widetilde{s}, \widetilde{a})\right\|_{L_2(\mu), \Sigma_{\rho, \hat{p}}^{-1}} \left\|\int_{\mathcal{S}}\sum_{a\in\mathcal{A}} \hat{p}(s|\cdot)\pi(a|s)g(s, a) ds\right\|_{L_2(\mu), \Sigma_{\rho, \hat{p}}}.
    \end{align*}
    Note that
    \begin{align*}
        & \left\|\int_{\mathcal{S}}\sum_{a\in\mathcal{A}} \hat{p}(s|\cdot)\pi(a|s)g(s, a) ds\right\|_{L_2(\mu), \Sigma_{\rho, \hat{p}}}^2\\
        = & n\mathbb{E}_{(\widetilde{s}, \widetilde{a}) \sim \rho} \left\{\mathbb{E}_{s\sim \hat{T}(\cdot|\widetilde{s}, \widetilde{a}), a\sim \pi(\cdot|s)} [g(s, a)]\right\}^2 + \lambda \left\|\int_{\mathcal{S}}\sum_{a\in\mathcal{A}} \hat{p}(s|\cdot)\pi(a|s)g(s, a) ds\right\|_{\mathcal{H}_k}\\
        \leq & n\mathbb{E}_{(\widetilde{s}, \widetilde{a}) \sim \rho} \left\{\mathbb{E}_{s\sim T^*(\cdot|\widetilde{s}, \widetilde{a}), a\sim \pi(\cdot|s)} [g(s, a)]\right\}^2 + \lambda B^2 C + n B^2 \zeta\\
        \leq & n\mathbb{E}_{(\widetilde{s}, \widetilde{a}) \sim \rho} \left\{\mathbb{E}_{s\sim T^*(\cdot|\widetilde{s}, \widetilde{a}), a\sim \pi(\cdot|s)} [g^2(s, a)]\right\} + \lambda B^2 C + n B^2 \zeta\\
        \leq & n \omega \mathbb{E}_{(\widetilde{s}, \widetilde{a}) \sim \rho} \left\{\mathbb{E}_{s\sim T^*(\cdot|\widetilde{s}, \widetilde{a}), a\sim \pi_b(\cdot|s)} [g^2(s, a)]\right\} + \lambda B^2 C + n B^2 \zeta\\
        \leq & \frac{n\omega}{\gamma} \mathbb{E}_{(s, a) \sim \rho} [g^2(s, a)] + \lambda B^2 C + n B^2 \zeta.
    \end{align*}
    Substitute back, we have the desired result.
\end{proof}
\begin{lemma}[One Step Back Inequality for the True Model in Offline Setting]
\label{lem:back_true_offline}
    Assume $g:\mathcal{S}\times \mathcal{A} \to \mathbb{R}$, such that $\|g\|_{\infty}\leq B$. Then we have that $\forall \pi$,
    \begin{align*}
        & \left|\mathbb{E}_{(s, a) \sim d_{T^*}^\pi}[g(s, a)]\right| \\
        \leq & \gamma \mathbb{E}_{(\widetilde{s}, \widetilde{a})\sim d_{T^*}^\pi} \left\|p^*(\cdot|\widetilde{s}, \widetilde{a})\right\|_{L_2(\mu), \Sigma_{\rho , \hat{p}}^{-1}} \sqrt{ n\omega\gamma\mathbb{E}_{(s, a) \sim \rho}  \left[g^2(s, a) \right] + \lambda \gamma^2 B^2 C}\\
        & + \sqrt{(1-\gamma) \omega\mathbb{E}_{(s, a) \sim \rho} [g^2(s, a)]}.
    \end{align*}
\end{lemma}
\begin{proof}
The proof for this lemma is nearly identical to the previous lemma, and we omit it for simplicity.
\end{proof}
\begin{lemma}[Almost Pessimism at the Initial Distribution] 
\label{lem:pessimism}
If we set $\zeta = \Theta\left(\frac{\log (|\mathcal{P}|/\delta)}{n}\right)$, $\lambda$ for different eigendecay condition as follows:
\begin{itemize}
    \item $\beta$-finite spectrum: $\lambda = \Theta(\beta \log n + \log (|\mathcal{P}|/\delta))$
    \item $\beta$-polynomial decay: $\lambda = \Theta(C_{\mathrm{poly}}n^{1/(1 + \beta)} + \log (|\mathcal{P}|/\delta))$;
    \item $\beta$-exponential decay: $\lambda = \Theta(C_{\mathrm{exp}}(\log n)^{1/\beta} +\log (|\mathcal{P}|/\delta))$;
\end{itemize}
and $\alpha = \Theta\left(\frac{\gamma}{1-\gamma}\sqrt{\omega \log(|\mathcal{P}|/\delta) + \lambda \gamma^2 C}\right)$, the following events hold with probability at least $1-\delta$:
\begin{align*}
    \forall \pi, \quad V_{\hat{T}, r - b}^{\pi} - V_{T^*, r}^\pi \leq \sqrt{\frac{\omega \zeta}{(1-\gamma)^3}}.
\end{align*}
\end{lemma}
\begin{proof}
    Note that, with the proof of Lemma~\ref{lem:bonus_concentration} and a union bound over $\mathcal{P}$ (but not over $n$), we know using the chosen $\lambda$, $\forall \hat{T} \in \mathcal{P}$, with probability at least $1-\delta$, 
    \begin{align*}
        \left\|\hat{p}(\cdot|s, a)\right\|_{L_2(\mu), \hat{\Sigma}_{n, \hat{p}}^{-1}} = \Theta\left(\left\|\hat{p}(\cdot|s, a)\right\|_{L_2(\mu), \Sigma_{\rho, \hat{p}}}^{-1}\right).
    \end{align*} 
    With Lemma~\ref{lem:simulation}, we have that
    \begin{align*}
        & (1-\gamma)\left(V_{\hat{T}, r-b}^\pi - V_{T^*, r}^{\pi}\right)\\
        = & \mathbb{E}_{(s, a) \sim d_{\hat{T}}^\pi} \left[-b(s, a) + \gamma \mathbb{E}_{\hat{T}(s^\prime|s, a)} \left[V_{T^*, r}^\pi(s^\prime)\right] - \gamma \mathbb{E}_{T^*(s^\prime|s, a)}\left[V_{T^*, r}^\pi(s^\prime)\right]\right] \\
        \lesssim & \mathbb{E}_{(s, a) \sim d_{\hat{T}}^\pi} \left[-\min\left\{\alpha\left\|\hat{p}(\cdot|s, a)\right\|_{L_2(\mu), \Sigma_{\rho, \hat{p}}^{-1}}, 2\right\} + \gamma \mathbb{E}_{\hat{T}(s^\prime|s, a)} \left[V_{T^*, r}^\pi(s^\prime)\right] - \gamma \mathbb{E}_{T^*(s^\prime|s, a)}\left[V_{T^*, r}^\pi(s^\prime)\right]\right]
    \end{align*}
    Denote $f(s, a) = \mathrm{TV}(\hat{T}(s, a), T^*(s, a))$, we know $\|f\|_{\infty} \leq 2$. With H\"older's inequality, we can obtain that
    \begin{align*}
        \left|\mathbb{E}_{\hat{T}(s^\prime|s, a)} \left[V_{T^*, r}^\pi(s^\prime)\right] - \mathbb{E}_{T^*(s^\prime|s, a)}\left[V_{T^*, r}^\pi(s^\prime)\right]\right| \leq \mathbb{E}_{(s, a) \sim d_{\hat{T}}^\pi} \left[\frac{f(s, a)}{1-\gamma}\right].
    \end{align*}
    With Lemma~\ref{lem:back_learned_offline}, we have that
    \begin{align*}
        & \mathbb{E}_{(s, a) \sim d_{\hat{T}}^\pi} \left[\frac{f(s, a)}{1-\gamma}\right]\\
        \leq & \mathbb{E}_{(\widetilde{s}, \widetilde{a}) \sim d_{\hat{T}}^\pi} \left\|\hat{p}(\cdot|\widetilde{s}, \widetilde{a})\right\|_{L_2(\mu), \Sigma_{\rho, \hat{p}}^{-1}} \sqrt{\frac{n\omega \gamma^2\mathbb{E}_{(s, a) \sim \rho}[f^2(s, a)]}{(1-\gamma)^2} + \frac{4\lambda \gamma^2 C}{(1-\gamma)^2} + \frac{4n\gamma^2 \zeta}{(1-\gamma)^2}}\\
        & + \sqrt{\frac{\omega \mathbb{E}_{(s, a) \sim \rho}[f^2(s, a)]}{1-\gamma}}\\
        \leq &  \mathbb{E}_{(\widetilde{s}, \widetilde{a}) \sim d_{\hat{T}}^\pi} \left\|\hat{p}(\cdot|\widetilde{s}, \widetilde{a})\right\|_{L_2(\mu), \Sigma_{\rho, \hat{p}}^{-1}} \sqrt{\frac{n\omega \gamma^2\zeta_n}{(1-\gamma)^2} + \frac{4\lambda \gamma^2 C}{(1-\gamma)^2} + \frac{4n\gamma^2 \zeta}{(1-\gamma)^2}} + \sqrt{\frac{\omega \zeta}{1-\gamma}}.
    \end{align*}
    With the choice of $\alpha$, we can conclude the proof.
\end{proof}
\begin{theorem}[PAC Guarantee for Offline Setting]
With probability at least $1-\delta$, for any competitor policy $\pi$ including non-Markovian history-dependent policy, we have
\begin{itemize}
    \item For $\beta$-finite spectrum, we have
    \begin{align*}
        V_{T^*, r}^{\pi} - V_{T^*, r}^{\hat{\pi}} \lesssim \frac{\omega \beta^{3/2}\log n}{(1-\gamma)^2} \sqrt{\frac{C C_{\pi}^* \log (|\mathcal{P}|/\delta)}{n}}
    \end{align*}
    \item For $\beta$-polynomial decay, we have
    \begin{align*}
        V_{T^*, r}^{\pi} - V_{T^*, r}^{\hat{\pi}} \lesssim \frac{C_{\mathrm{poly}} \omega n^{\frac{1-\beta}{2(1+\beta)}}\log n\sqrt{C C_{\pi}^* \log (|\mathcal{P}|/\delta)}}{(1-\gamma)^2}
    \end{align*}
    \item For $\beta$-exponential decay, we have
    \begin{align*}
    V_{T^*, r}^{\pi} - V_{T^*, r}^{\hat{\pi}} \lesssim \frac{C_{\mathrm{exp}} \omega (\log n)^{\frac{3 + 2\beta}{2\beta}}}{(1-\gamma)^2} \sqrt{\frac{C C_{\pi}^* \log (|\mathcal{P}|/\delta)}{n}}
    \end{align*}
\end{itemize}
\end{theorem}
\begin{proof}
    With Lemma~\ref{lem:pessimism} and Lemma~\ref{lem:simulation}, we have that
    \begin{align*}
        & V_{T^*, r}^{\pi} - V_{T^*, r}^{\hat{\pi}}\\
        \leq & V_{T^*, r}^{\pi} - V_{\hat{T}, r-b}^{\hat{\pi}} + \sqrt{\frac{\omega \zeta}{(1-\gamma)^3}}\\
        \leq & V_{T^*, r}^{\pi} - V_{\hat{T}, r-b}^{\pi} + \sqrt{\frac{\omega \zeta}{(1-\gamma)^3}}\\
        \leq & \frac{1}{1-\gamma} \mathbb{E}_{(s, a) \sim d_{T^*}^{\pi}} \left[b(s, a) + \gamma \mathbb{E}_{T^*(s^\prime|s, a)} \left[V_{T^*, r}^{\pi}(s^\prime)\right] - \gamma \mathbb{E}_{\hat{T}(s^\prime|s, a)} \left[V_{T^*, r}^{\pi}(s^\prime)\right]\right] + \sqrt{\frac{\omega \zeta}{(1-\gamma)^3}}.
    \end{align*}
    As $b = O(1)$, with Lemma~\ref{lem:back_true_offline}, we have that
    \begin{align*}
        & \mathbb{E}_{(s, a) \sim d_{T^*}^{\pi}} [b(s, a)]\\
        \lesssim & \mathbb{E}\left[\min\left\{\alpha \left\|\hat{p}(\cdot|s, a)\right\|_{L_2(\mu), \Sigma_{\rho, \hat{p}}^{-1}}, 2\right\}\right]\\
        \lesssim & \mathbb{E}_{(\widetilde{s}, \widetilde{a})\sim d_{T^*}^{\pi}}\left\|p^*(\cdot|\widetilde{s}, \widetilde{a})\right\|_{L_2(\mu), \Sigma_{\rho, p^*}^{-1}}\sqrt{n\omega\gamma \alpha^2 \mathbb{E}_{(s, a)\sim\rho}\left[\left\|\hat{p}(\cdot|s, a)\right\|_{L_2(\mu), \Sigma_{\rho, \hat{p}}^{-1}}^2\right] + \lambda \gamma^2 C}\\
        & + \sqrt{(1-\gamma)\omega \alpha^2 \mathbb{E}_{(s, a)\sim\rho}\left[\left\|\hat{p}(\cdot|s, a)\right\|_{L_2(\mu), \Sigma_{\rho, \hat{p}}^{-1}}^2\right]}.
    \end{align*}
    With the reasoning similar to the proof in Lemma~\ref{lem:regret}, we have that 
    \begin{itemize}
        \item For $\beta$-finite spectrum,
        \begin{align*}
            \mathbb{E}_{(s, a)\sim\rho}\left[\left\|\hat{p}(\cdot|s, a)\right\|_{L_2(\mu), \Sigma_{\rho, \hat{p}}^{-1}}^2\right] = O\left(\beta \log n\right),
        \end{align*}
        which leads to
        \begin{align*}
            \mathbb{E}_{(s, a) \sim d_{T^*}^{\pi}} [b(s, a)] \lesssim & \mathbb{E}_{(\widetilde{s}, \widetilde{a})\sim d_{T^*}^{\pi}}\left\|p^*(\cdot|\widetilde{s}, \widetilde{a})\right\|_{L_2(\mu), \Sigma_{\rho, p^*}^{-1}}\sqrt{\omega\gamma \beta \alpha^2  \log n + \lambda \gamma^2 C}\\
            & + \sqrt{\frac{(1-\gamma) \omega \beta \alpha^2  \log n}{n}}.
        \end{align*}
        \item For $\beta$-polynomial decay,
        \begin{align*}
        \mathbb{E}_{(s, a)\sim\rho}\left[\left\|\hat{p}(\cdot|s, a)\right\|_{L_2(\mu), \Sigma_{\rho, \hat{p}}^{-1}}^2\right] = O\left(C_{\mathrm{poly}}n^{\frac{1}{2(1+\beta)}} \log n\right),
        \end{align*}
        which leads to
        \begin{align*}
            \mathbb{E}_{(s, a) \sim d_{T^*}^{\pi}} [b(s, a)] \lesssim & \mathbb{E}_{(\widetilde{s}, \widetilde{a})\sim d_{T^*}^{\pi}}\left\|p^*(\cdot|\widetilde{s}, \widetilde{a})\right\|_{L_2(\mu), \Sigma_{\rho, p^*}^{-1}}\sqrt{\omega\gamma C_{\mathrm{poly}}\alpha^2 n^{\frac{1}{2(1+\beta)}} \log n + \lambda \gamma^2 C}\\
            & + \sqrt{(1-\gamma) \omega C_{\mathrm{poly}}\alpha^2 n^{-1 + \frac{1}{2(1+\beta)}} \log n}.
        \end{align*}
        \item For $\beta$-exponential decay,
        \begin{align*}
            \mathbb{E}_{(s, a) \sim \rho}\left[\left\|\hat{p}(\cdot|s, a)\right\|_{L_2(\mu), \Sigma_{\rho, \hat{p}}^{-1}}^2\right] = O\left(C_{\mathrm{exp}}(\log n)^{1 + 1/\beta} \right),
        \end{align*}
        which leads to
        \begin{align*}
            \mathbb{E}_{(s, a) \sim d_{T^*}^{\pi}} [b(s, a)] \lesssim & \mathbb{E}_{(\widetilde{s}, \widetilde{a})\sim d_{T^*}^{\pi}}\left\|p^*(\cdot|\widetilde{s}, \widetilde{a})\right\|_{L_2(\mu), \Sigma_{\rho, p^*}^{-1}}\sqrt{\omega\gamma C_{\mathrm{exp}}\alpha^2 (\log n)^{1 + 1/\beta} + \lambda \gamma^2 C}\\
            & + \sqrt{\frac{(1-\gamma) \omega \alpha^2 C_{\mathrm{exp}} (\log n)^{1 + 1/\beta} }{n}}.
        \end{align*}
    \end{itemize}
    Furthermore, denote $f(s, a) = \mathrm{TV}(\hat{T}(s, a), T^*(s, a))$, we have $\|f\|_{\infty} \leq 2$. With H\"older's inequality, we have
    \begin{align*}
        \left|\mathbb{E}_{(s, a) \sim d_{T^*}^{\pi}}\left[\mathbb{E}_{T^*(s^\prime|s, a)}\left[V_{T^*, r}^{\pi}(s^\prime)\right] - \mathbb{E}_{\hat{T}(s^\prime|s, a)}\left[V_{T^*, r}^{\pi}(s^\prime)\right]\right]\right| \leq \mathbb{E}_{(s, a) \sim d_{T^*}^{\pi}}\left[\frac{f(s, a
        )}{1-\gamma}\right].
    \end{align*}
    With Lemma~\ref{lem:back_true_offline}, we have
    \begin{align*}
        & \mathbb{E}_{(s, a) \sim d_{T^*}^{\pi}}\left[\frac{f(s, a)}{1-\gamma}\right] \leq \mathbb{E}_{(\widetilde{s}, \widetilde{a})\sim d_{T^*}^{\pi}} \left\|p^*(\cdot|\widetilde{s}, \widetilde{a})\right\|_{L_2(\mu), \Sigma_{\rho, p^*}^{-1}} \sqrt{\frac{n\omega \gamma \mathbb{E}_{(s, a) \sim \rho}[f^2(s, a)]}{(1-\gamma)^2} + \frac{4\lambda \gamma^2 C}{(1-\gamma)^2}} \\
        & + \sqrt{\frac{\omega \mathbb{E}_{(s, a) \sim \rho} [f^2(s, a)]}{1-\gamma}}\\
        \leq & \mathbb{E}_{(\widetilde{s}, \widetilde{a})\sim d_{T^*}^{\pi}}\left\|p^*(\cdot|\widetilde{s}, \widetilde{a})\right\|_{L_2(\mu), \Sigma_{\rho, p^*}^{-1}} \cdot \sqrt{\frac{n\omega \gamma \zeta}{(1-\gamma^2)} + \frac{4\lambda \gamma^2 C}{(1-\gamma)^2}} + \sqrt{\frac{\omega \zeta}{(1-\gamma)}}\\
        \lesssim & \alpha_n \mathbb{E}_{(\widetilde{s}, \widetilde{a})\sim d_{T^*}^{\pi}} \left\|p^*(\cdot|\widetilde{s}, \widetilde{a})\right\|_{L_2(\mu), \Sigma_{\rho, p^*}^{-1}} + \sqrt{\frac{\omega \zeta}{1-\gamma}}.
    \end{align*}
    Combine with the previous results and take the dominating terms out, we have that
    \begin{itemize}
        \item For $\beta$-finite spectrum,
        \begin{align*}
            V_{T^*, r}^{\pi} - V_{T^*, r}^{\hat{\pi}}\\
            \lesssim & \frac{1}{1-\gamma}\mathbb{E}_{(\widetilde{s}, \widetilde{a})\sim d_{T^*}^{\pi}}\left\|p^*(\cdot|\widetilde{s}, \widetilde{a})\right\|_{L_2(\mu), \Sigma_{\rho, p^*}^{-1}}\sqrt{\omega\gamma \beta \alpha^2  \log n + \lambda \gamma^2 C}\\
            & + \sqrt{ \frac{\omega \beta \alpha^2  \log n}{(1-\gamma) n}} + \sqrt{\frac{\omega \zeta}{(1-\gamma)^3}}.
        \end{align*}
        \item For $\beta$-polynomial decay,
        \begin{align*}
            V_{T^*, r}^{\pi} - V_{T^*}^{\hat{\pi}}\\
            \lesssim & \frac{1}{1-\gamma}\mathbb{E}_{(\widetilde{s}, \widetilde{a})\sim d_{T^*}^{\pi}}\left\|p^*(\cdot|\widetilde{s}, \widetilde{a})\right\|_{L_2(\mu), \Sigma_{\rho, p^*}^{-1}}\sqrt{\omega\gamma C_{\mathrm{poly}}\alpha^2 n^{\frac{1}{2(1+\beta)}} \log n + \lambda \gamma^2 C}\\
            & + \sqrt{ \frac{\omega C_{\mathrm{poly}}\alpha^2 n^{-1 + \frac{1}{2(1+\beta)}} \log n}{1-\gamma}} + \sqrt{\frac{\omega \zeta}{(1-\gamma)^3}}.
        \end{align*}
        \item For $\beta$-exponential decay,
        \begin{align*}
            V_{T^*, r}^{\pi} - V_{T^*, r}^{\hat{\pi}}\\
            \lesssim & \frac{1}{1-\gamma}\mathbb{E}_{(\widetilde{s}, \widetilde{a})\sim d_{T^*}^{\pi}}\left\|p^*(\cdot|\widetilde{s}, \widetilde{a})\right\|_{L_2(\mu), \Sigma_{\rho, p^*}^{-1}}\sqrt{\omega\gamma C_{\mathrm{exp}} \alpha^2 (\log n)^{1/\beta} \log \log n + \lambda \gamma^2 C}\\
            & + \sqrt{\frac{\omega C_{\mathrm{exp}}\alpha^2 (\log n)^{1 + 1/\beta}}{(1-\gamma) n}} + \sqrt{\frac{\omega \zeta}{(1-\gamma)^3}}.
        \end{align*}
    \end{itemize}
    We now deal with the term $\mathbb{E}_{(\widetilde{s}, \widetilde{a})\sim d_{T^*}^{\pi}}\left\|p^*(\cdot|\widetilde{s}, \widetilde{a})\right\|_{L_2(\mu), \Sigma_{\rho, p^*}^{-1}}$. With Lemma~\ref{lem:distribution_shift}, we know
    \begin{align*}
        & \mathbb{E}_{(\widetilde{s}, \widetilde{a})\sim d_{T^*}^{\pi}}\left\|p^*(\cdot|\widetilde{s}, \widetilde{a})\right\|_{L_2(\mu), \Sigma_{\rho, p^*}^{-1}}\\
        \leq & \sqrt{\mathbb{E}_{(\widetilde{s}, \widetilde{a})\sim d_{T^*}^{\pi}}\left\|p^*(\cdot|\widetilde{s}, \widetilde{a})\right\|_{L_2(\mu), \Sigma_{\rho, p^*}^{-1}}^2}\\
        \leq & \sqrt{C\mathbb{E}_{(\widetilde{s}, \widetilde{a})\sim \rho}\left\|p^*(\cdot|\widetilde{s}, \widetilde{a})\right\|_{L_2(\mu), \Sigma_{\rho, p^*}^{-1}}^2}.
    \end{align*}
    Applying the identical method used in the proof of Lemma~\ref{lem:regret}, we have that:
    \begin{itemize}
        \item For $\beta$-finite spectrum, we have
        \begin{align*}
            \mathbb{E}_{(\widetilde{s}, \widetilde{a})\sim\rho}\left[\left\|p^*(\cdot|\widetilde{s}, \widetilde{a})\right\|_{L_2(\mu), \Sigma_{\rho, p^*}^{-1}}^2\right] = O\left(\beta \log n\right).
        \end{align*}
        \item For $\beta$-polynomial decay, we have
        \begin{align*}
            \mathbb{E}_{(\widetilde{s}, \widetilde{a})\sim\rho}\left[\left\|p^*(\cdot|\widetilde{s}, \widetilde{a})\right\|_{L_2(\mu), \Sigma_{\rho, p^*}^{-1}}^2\right] = O\left(C_{\mathrm{poly}} n^{\frac{1}{2(1 + \beta)}} \log n\right).
        \end{align*}
        \item For $\beta$-exponential decay, we have
        \begin{align*}
            \mathbb{E}_{(\widetilde{s}, \widetilde{a})\sim\rho}\left[\left\|p^*(\cdot|\widetilde{s}, \widetilde{a})\right\|_{L_2(\mu), \Sigma_{\rho, p^*}^{-1}}^2\right] = O\left(C_{\mathrm{exp}} \left(\log n\right)^{1 + 1/\beta}\right).
        \end{align*}
    \end{itemize}
    Substitute $\alpha$ and $\lambda$ back, we have that:
    \begin{itemize}
        \item For $\beta$-finite spectrum, we have
        \begin{align*}
            V_{T^*, r}^{\pi} - V_{T^*, r}^{\hat{\pi}} \lesssim \frac{\omega \beta^{3/2}\log n}{(1-\gamma)^2} \sqrt{\frac{C C_{\pi}^* \log (|\mathcal{P}|/\delta)}{n}}
        \end{align*}
        \item For $\beta$-polynomial decay, we have
        \begin{align*}
            V_{T^*, r}^{\pi} - V_{T^*, r}^{\hat{\pi}}  \lesssim \frac{C_{\mathrm{poly}} \omega n^{\frac{1-\beta}{2(1+\beta)}}\log n\sqrt{C C_{\pi}^* \log (|\mathcal{P}|/\delta)}}{(1-\gamma)^2}
        \end{align*}
        \item For $\beta$-exponential decay, we have
        \begin{align*}
            V_{T^*, r}^{\pi} - V_{T^*, r}^{\hat{\pi}} \lesssim \frac{C_{\mathrm{exp}} \omega (\log n)^{\frac{3 + 2\beta}{2\beta}}}{(1-\gamma)^2} \sqrt{\frac{C C_{\pi}^* \log (|\mathcal{P}|/\delta)}{n}}
        \end{align*}
    \end{itemize}
    This finishes the proof.
\end{proof}

% \newpage
\section{Auxillary Lemmas}
We first state the MLE generalization bound from \citep{agarwal2020flambe}. Note that, when Assumption~\ref{assumption:function_class} holds, the MLE generalization bound only depends on the complexity of $\mathcal{P}$.
\begin{lemma}[MLE Generalization Bound \citep{agarwal2020flambe}]
    \label{lem:mle}
    For a fixed episode $n$, with probability at least $1-\delta$, we have that
    \begin{align*}
        \mathbb{E}_{s\sim 0.5\rho_n + 0.5\rho_n^\prime, a\sim\mathcal{U}(\mathcal{A})} \left[\left\|\hat{T}_n(s, a) - T^*(s, a)\right\|_1^2\right] \leq \frac{\log (|\mathcal{P}|/\delta)}{n}.
    \end{align*}
    With a union bound, with probability at least $1-\delta$, we have that
    \begin{align*}
        \forall n\in\mathbb{N}^+, \mathbb{E}_{s\sim 0.5\rho_n + 0.5\rho_n^\prime, a\sim\mathcal{U}(\mathcal{A})} \left[\left\|\hat{T}_n(s, a) - T^*(s, a)\right\|_1^2\right] \leq \frac{\log (n|\mathcal{P}|/\delta)}{n}
    \end{align*}
\end{lemma}

% \Tongzheng{@Bo: Can you polish these two lemmas?}
\begin{lemma}[Concentration of the Bonuses] 
    \label{lem:bonus_concentration}
    Let $\mu_i$ be the conditional distribution of $\phi$ given the sampled $\phi_1, \cdots, \phi_{i-1}$, define $\Sigma: L_2(\mu) \to L_2(\mu)$, $\Sigma_n := \frac{1}{n}\sum_{i\in[n]} \mathbb{E}_{\phi\sim\mu_i} \phi \phi^\top$. Assume $\|\phi\|_{\mathcal{H}_k} \leq 1$ for any realization of $\phi$. If $\lambda$ satisfies the following conditions for each eigendecay condition:
    \begin{itemize}
        \item $\beta$-finite spectrum: $\lambda = \Theta(\beta\log N + \log (N/\delta))$;
        \item $\beta$-polynomial decay: $\lambda = \Theta(C_{\mathrm{poly}}N^{1/(1 + \beta)} + \log (N/\delta))$; 
        \item $\beta$-exponential decay: $\lambda = \Theta(C_{\mathrm{exp}}(\log N)^{1/\beta} + \log(N/\delta))$, where $C_3$ is a constant depends on $C_1$ and $C_2$;
    \end{itemize}
    then there exists absolute constant $c_1$ and $c_2$, such that $\forall x \in \mathcal{H}_k$, the following event holds with probability at least $1-\delta$:
    \begin{align*}
        \forall n\in [N], \quad  c_1  \left\langle x, \left(n\Sigma_n + \lambda T_k^{-1}\right) x \right\rangle_{L_2(\mu)}\leq &  \left\langle x, \left(\sum_{i\in [n]} \phi_i \phi_i^\top + \lambda T_k^{-1} \right) x\right\rangle_{L_2(\mu)}, \\
        \mathrm{and} \quad \left\langle x, \left(\sum_{i\in [n]} \phi_i \phi_i^\top + \lambda T_k^{-1} \right) x\right\rangle_{L_2(\mu)}\leq & c_2\left\langle x, \left(n\Sigma_n + \lambda T_k^{-1}\right)x \right\rangle_{L_2(\mu)}.
    \end{align*}
    In the same event above, the following event must hold as well:
    \begin{align*}
        \forall n\in [N], \quad \frac{1}{c_2}  \left\langle x, \left(n\Sigma_n + \lambda T_k^{-1}\right)^{-1} x \right\rangle_{L_2(\mu) }\leq & \left\langle x, \left(\sum_{i\in [n]} \phi_i \phi_i^\top + \lambda T_k^{-1} \right)^{-1} x\right\rangle_{L_2(\mu)}\\
        \mathrm{and} \quad \left\langle x, \left(\sum_{i\in [n]} \phi_i \phi_i^\top + \lambda T_k^{-1} \right)^{-1} x\right\rangle_{L_2(\mu)} \leq & \frac{1}{c_1}\left\langle x, \left(n\Sigma_n + \lambda T_k^{-1}\right)^{-1}x \right\rangle_{L_2(\mu)}
    \end{align*}
\end{lemma}
\begin{proof}
    Note that, $\|T_k^{-1/2} \phi\|_{L_2(\mu)} = \|\phi\|_{\mathcal{H}_k} \leq 1$. Hence, the operator norm of operators $\widetilde{\Sigma}_n := T_k^{-1/2} \Sigma_n T_k^{-1/2}$ that maps from $L_2(\mu)$ to $L_2(\mu)$ are upper bounded by $1$. For notation simplicity, we define $\widetilde{\phi} := T_k^{-1/2} \phi$ and $\widetilde{\mu}_i$ denotes the conditional distribution of $\widetilde{\phi}$ given the sampled $\widetilde{\phi}_1, \cdots, \widetilde{\phi}_{i-1}$. Note that $\forall x \in \mathcal{H}_k$, $T_k^{1/2} x \in\mathcal{H}_{\widetilde{k}}$. We now prove the following equivalent form of the claim: $\forall x \in \mathcal{H}_{\widetilde{k}}$, $ \forall n\geq 1$, 
    \begin{align*}
         c_1  \left\langle x, \left(n \widetilde{\Sigma}_n + \lambda T_k^{-2}\right) x \right\rangle_{L_2(\mu) }\leq \left\langle x, \left(\sum_{i\in [n]} \widetilde{\phi}\widetilde{\phi}^\top + \lambda T_k^{-2} \right) x\right\rangle_{L_2(\mu)}\leq c_2\left\langle x, \left(n \widetilde{\Sigma}_n + \lambda T_k^{-2}\right)x \right\rangle_{L_2(\mu)}
    \end{align*}
    It is sufficient to consider $x$ with $\|x\|_{\mathcal{H}_{\widetilde{k}}} = 1$. Note that, we have
    \begin{align*}
        \left\langle x, \widetilde{\phi}\widetilde{\phi}^\top x\right\rangle_{L_2(\mu)} = \left\langle x, \widetilde{\phi}\right\rangle_{L_2(\mu)}^2  \leq \|x\|^2_{L_2(\mu)}.
    \end{align*}
    Denote $\widetilde{\Sigma}^i := \mathbb{E}_{\phi\sim \mu_i} \widetilde{\phi}\widetilde{\phi}^\top$. We have
    \begin{align*}
        \mathrm{Var}_{\phi\sim\mu_i}\left[\langle x, \widetilde{\phi}\rangle_{L_2(\mu)}^2\right] \leq \|x\|_{L_2(\mu)}^2 \mathbb{E}_{\widetilde{\phi}\sim\widetilde{\mu}_i}\left[\langle x, \widetilde{\phi}\rangle_{L_2(\mu)}^2\right] = \|x\|_{L_2(\mu)}^2\langle x, \tilde{\Sigma}^i x\rangle_{L_2(\mu)}
    \end{align*}
    we can invoke a Bernstein-style martingale concentration inequality \citep[Lemma 45,][]{zanette2021cautiously}, and obtain that with probability at least $1-\delta$
    \begin{align*}
    \left|\frac{1}{n}\sum_{i\in[n]}\left[ \langle x, \widetilde{\phi}_i \rangle_{L_2(\mu)}^2 \right]- \langle x, \widetilde{\Sigma} x\rangle_{L_2(\mu)}\right| \leq c\left(\sqrt{\frac{\|x\|_{L_2(\mu)}^2\langle x, \tilde{\Sigma}_n x\rangle_{L_2(\mu)} \log (2/\delta)}{n}} + \frac{\|x\|_{L_2(\mu)}^2\log (2/\delta)}{3n}\right),
    \end{align*}
    where $c$ is an absolute constant.
    % Note that, $\langle x, T_k x\rangle_{L_2(\mu)}\leq 1$ if $\|x\|_{\mathcal{H}_k} = 1$. 
    We then show that, if we have $\lambda = \Omega(\log (1/\delta))$, we have that
    \begin{align*}
        c\left(\sqrt{\frac{\|x\|_{L_2(\mu)}^2\langle x, \widetilde{\Sigma} x \rangle_{L_2(\mu)}\log (2/\delta)}{n}} + \frac{\|x\|_{L_2(\mu)^2}\log (2/\delta)}{3n}\right) \leq C\left(\langle x, \widetilde{\Sigma} x\rangle_{L_2(\mu)} + \frac{\lambda\|x\|_{L_2(\mu)}^2}{n}\right).
    \end{align*}
    where $C < 1$ is an absolute constant, following the similar reasoning in the proof of Lemma 39 in \citep{zanette2021cautiously}:
    \begin{itemize}
        \item $\langle x, \widetilde{\Sigma} x\rangle_{L_2(\mu)} \leq \frac{\lambda \|x\|_{L_2(\mu)}^2}{n}$: It's sufficient to show that $c\sqrt{\lambda \log (2/\delta)} \leq \frac{C\lambda}{2}$ and $\frac{c\log(2/\delta)}{3} \leq \frac{C\lambda}{2}$, which can be achieved by $\lambda = \Omega(\log 1/\delta)$.
        \item $\langle x, \widetilde{\Sigma} x\rangle_{L_2(\mu)} \geq \frac{\lambda \|x\|_{L_2(\mu)}^2}{n}$: It's sufficient to show that 
        \begin{align*}
            \lambda \geq \frac{c}{C} \log (2/\delta) \quad \mathrm{and} \quad  \frac{c}{C} \sqrt{\frac{\|x\|_{L_2(\mu)}^2\log (2/\delta)}{n}} \leq \sqrt{\langle x, \widetilde{\Sigma} x\rangle_{L_2(\mu)}}.
        \end{align*}
        As $\langle x, \widetilde{\Sigma }x\rangle_{L_2(\mu)} \geq \frac{\lambda \|x\|_{L_2(\mu)}^2}{n}$, when $\lambda \geq \max\left\{\frac{c}{C}, \frac{c^2}{C^2}\right\} \log (2/\delta)$, these two conditions hold simultaneously.
    \end{itemize}
    Hence, for any fixed $x$ with $\|x\|_{\mathcal{H}_{\widetilde{k}}} = 1$, we have
    \begin{align*}
        \left|\frac{1}{n}\sum_{i\in[n]}\left[ \langle x, \widetilde{\phi}_i \rangle_{L_2(\mu)}^2 \right]- \langle x, \widetilde{\Sigma} x\rangle_{L_2(\mu)}\right| \leq C\left \langle x, \left(\widetilde{\Sigma}  + \frac{\lambda}{n} I\right)x \right\rangle_{L_2(\mu)} \leq C\left \langle x, \left(\widetilde{\Sigma}  + \frac{\lambda}{n} T_k^{-2}\right)x \right\rangle_{L_2(\mu)}.
    \end{align*}
    Now, assume such condition holds for an $\varepsilon$-net $\mathcal{B}_{\varepsilon}$  of $\mathcal{S}_{\mathcal{H}_{\widetilde{k}}}$, the unit sphere of RKHS $\mathcal{H}_{\widetilde{k}}$ (i.e. $\{x:\|x\|_{\mathcal{H}_{\widetilde{k}}}=1\}$), under $\|\cdot\|_{L_2(\mu)}$. 
    Then $\forall x$ satisfies $\|x\|_{\mathcal{H}_{\widetilde{k}}} = 1$, let $x^\prime$ be the closest point of $x$ in $\mathcal{B}_\varepsilon$ under $\|\cdot\|_{L_2(\mu)}$ (note that $x^\prime \in \mathcal{H}_{\widetilde{k}}$ by the definition of $\varepsilon$-net). We have that
    \begin{align*}
        \left|\langle x, \widetilde{\Sigma} x\rangle - \langle x^\prime, \widetilde{\Sigma} x^\prime\rangle\right| \leq & 2\varepsilon\\
        \left|\left\langle x, \left(\frac{1}{n}\sum_{i\in[n]} \widetilde{\phi}_i \widetilde{\phi}_i^\top\right) x\right\rangle - \left\langle x^\prime,  \left(\frac{1}{n}\sum_{i\in[n]} \widetilde{\phi}_i \widetilde{\phi}_i^\top \right)x^\prime\right\rangle\right| \leq & 2\varepsilon\\
    \end{align*}
        % Take $A = \frac{1}{n}\sum_{i\in [n]} \phi_i \phi_i^\top$ and $A = \Sigma$, we have that
        % \begin{align*}
        %     \left|\frac{1}{n} \sum_{i\in [n]} \langle x, \phi_i\rangle_{L_2(\mu)}^2 - \frac{1}{n} \sum_{i\in [n]} \langle x^\prime, \phi_i\rangle_{L_2(\mu)}^2 \right| \leq & 2\varepsilon, \\
        %     \left|\langle x, \Sigma x\rangle_{L_2(\mu)} - \langle x^\prime, \Sigma x^\prime\rangle_{L_2(\mu)}\right| \leq 2\varepsilon.
        % \end{align*}
    With a triangle inequality, $\forall n$, $\forall \|x\|_{\mathcal{H}_k} \leq 1$, we have
    \begin{align*}
        & \left|\left\langle x, \left(\frac{1}{n}\sum_{i\in [n]} \widetilde{\phi}_i \widetilde{\phi}_i^\top + \frac{\lambda}{n} T_k^{-2}\right) x \right\rangle_{L_2(\mu)} - \left\langle x, \left(\widetilde{\Sigma} + \frac{\lambda}{n} T_k^{-2} \right)x  \right\rangle_{L_2(\mu)}\right| \\
        \leq & C\left\langle x, \left(\widetilde{\Sigma}  + \frac{\lambda}{n} T_k^{-2} \right) x\right\rangle_{L_2(\mu)} + (4 + 2C)\varepsilon
    \end{align*}

    Hence, we can choose $\varepsilon = O\left(\frac{\lambda}{n}\right)$, to guarantee that
    \begin{align*}
        C\left\langle x, \left(\widetilde{\Sigma}  + \frac{\lambda}{n} T_k^{-2} \right) x\right\rangle_{L_2(\mu)} + (4 + 2C)\varepsilon \leq C^\prime\left\langle x, \left(\widetilde{\Sigma}  + \frac{\lambda}{n} T_k^{-2} \right) x\right\rangle_{L_2(\mu)},
    \end{align*}
    where $C^\prime < 1$ is an absolute constant.

    Now we consider the covering number $\mathcal{N}(\mathbb{S}_{\mathcal{H}_{\widetilde{k}}}, \|\cdot\|_{L_2(\mu)}, \varepsilon)$. We start from the entropy number $e_i(\mathcal{S}_{\mathcal{H}_{\widetilde{k}}}, \|\cdot\|_{L_2(\mu)})$. From (A.36) in \citet{steinwart2008support}, we know        \begin{align*}
        e_i(\mathcal{B}_{\mathcal{H}_{\widetilde{k}}}, \|\cdot\|_{L_2(\mu)}) \leq e_i(\mathcal{S}_{\mathcal{H}_{\widetilde{k}}}, \|\cdot\|_{L_2(\mu)}) \leq 2e_i(\mathcal{B}_{\mathcal{H}_{\widetilde{k}}}, \|\cdot\|_{L_2(\mu)}),
    \end{align*}
    % \Tongzheng{Talking about the approximation number?}
    where $\mathcal{B}_{\mathcal{H}_{\widetilde{k}}}$ is the unit ball in RKHS $\mathcal{H}_{\widetilde{k}}$ (i.e. $\{x: \|x\|_{\mathcal{H}_{\widetilde{k}}} \leq 1\}$). With Carl's inequality\footnote{A more formal claim is on the approximation number of the bounded linear operator, which, as shown in \citet{steinwart2008support}, is identical to the eigenvalue of the bounded linear operator if the bounded linear operator is compact, self-adjoint and positive.} \citep{carl_stephani_1990} (also see \citep{steinwart2009optimal}), $\forall p > 0, m \in \mathbb{N}^{+}$, we have
    \begin{align*}
        & \sup_{i\in [m]} i^{1/p} e_i(\mathrm{id}:\mathcal{H}_{\widetilde{k}} \to L_2(\mu)) \\
        \leq & c_p \sup_{i\in [m]} i^{1/p}\mu_i^{1/2}\left(T_k^2:L_2(\mu) \to L_2(\mu)\right)\\
        = & c_p \sup_{i\in [m]} i^{1/p}\mu_i\left(T_k:L_2(\mu) \to L_2(\mu)\right)
    \end{align*}
    where $c_p = 128(32 + 16/p)^{1/p} $ denotes a constant only depending on $p$. 
    We then consider the entropy number under different eigendecay conditions:
    \begin{itemize}
        \item For $\beta$-finite spectrum, as we have $\sum_{i\in I} \mu_i \leq 1$ from Assumption~\ref{assump:trace}, and $\forall i > \beta, \mu_i(T_k: L_2(\mu) \to L_2(\mu)) = 0$, we know for a fixed $p$, 
        \begin{align*}
            e_i(\mathrm{id}:\mathcal{H}_{\widetilde{k}} \to L_2(\mu)) \leq 128\left(\left(32 + 16/p\right)\right)^{1/p}(\beta/i)^{1/p}.
        \end{align*}
        Take $p=\beta/i$, we know that
        \begin{align*}
            e_i(\mathrm{id}:\mathcal{H}_{\widetilde{k}} \to L_2(\mu)) \leq 128(32 \beta + 16)^{-i/\beta}.
        \end{align*}
        \item For $\beta$-polynomial decay, take $p = 2/\beta$ and obtain that 
        \begin{align*}
            e_i(\mathrm{id}:\mathcal{H}_{\widetilde{k}} \to L_2(\mu)) \leq 128C_0(32 + 8\beta)^{\beta/2} i^{-\beta}. 
        \end{align*}
        \item For $\beta$-exponential decay, note that, for a fixed $p$, 
        % the $i$ with the maximum $i^{1/p} \exp(-C_2 i^{\beta})$ should also achieve the maximum $(\log i )/ p - C_2 i ^\beta$.
        direct computation shows the maximum of $i^{1/p} \exp(-C_2 i^{\beta})$ is achieved when $i^\beta = \frac{1}{C_2\beta p}$. Furthermore, $i^{1/p} \exp(-C_2 i^{\beta})$ is monotonically increasing with respect to $i$ when $i^\beta < \frac{1}{C_2 \beta p}$, while monotonically decreasing with respect to $i$ when $i^\beta > \frac{1}{C_2 \beta p}$. Hence, for a given $i$, we can choose $p$ such that $i^\beta > \frac{1}{C_2 \beta p}$, and obtain that
        \begin{align*}
            e_i(\mathrm{id}:\mathcal{H}_{\widetilde{k}} \to L_2(\mu)) \leq 128(32 + 16/p)^{1/p} C_1 \exp(- C_2 i^\beta).
        \end{align*}
        As we can take $p \to \infty$, we have that
        \begin{align*}
            e_i(\mathrm{id}:\mathcal{H}_{\widetilde{k}} \to L_2(\mu)) \leq 128 C_1 \exp(-C_2 i^\beta).
        \end{align*}
    \end{itemize}
    We now convert the entropy number bound for different eigendecay conditions to the covering number bound accordingly.
    \begin{itemize}
        \item For $\beta$-finite spectrum, we fix a $\delta \in (0, 1)$ and an $\varepsilon\in (0, 128]$, and assume the integer $i\geq 1$ satisfies the condition:
        \begin{align*}
            128(1 + \delta)(32 \beta + 16)^{-(i+1)/\beta} \leq \varepsilon \leq 128(1 + \delta)(32 \beta + 16)^{-i/\beta}.
        \end{align*}
        By the definition of the entropy number and covering number, we know
        \begin{align*}
            & \log \mathcal{N}(\mathcal{B}_{\mathcal{H}_{\widetilde{k}}}, \|\cdot\|_{L_2(\mu)}, \varepsilon)\\
            \leq & \log \mathcal{N}(\mathcal{B}_{\mathcal{H}_{\widetilde{k}}}, \|\cdot\|_{L_2(\mu)}, 128(1 + \delta)(32\beta + 16)^{-(i+1)/\beta})\\
            \leq & i \log (2) \\
            \leq & \beta \log(2) \log\left(\frac{128(1 + \delta)}{\varepsilon}\right) \\
            \leq & \beta \log(2) \log\left(\frac{256}{\varepsilon}\right) = O\left(\beta \log (1/\varepsilon)\right)
        \end{align*}
        \item For $\beta$-polynomial decay, with Lemma 6.21 in \citep{steinwart2008support}, we have that
        \begin{align*}
            \log \mathcal{N}(\mathcal{B}_{\mathcal{H}_{\widetilde{k}}}, \|\cdot\|_{L_2(\mu)}, \varepsilon) \leq \log (4) \left(\frac{128 C_0(32 + 8\beta)^{\frac{\beta}{2}}}{\varepsilon}\right)^{1/\beta} = O\left(C_{\mathrm{poly}} \varepsilon^{-1/\beta}\right).
        \end{align*}
        \item For $\beta$-exponential decay, we fix a $\delta \in (0, 1)$ and an $\varepsilon \in (0, 128C_1]$, and assume the integer $i \geq 1$ satisfies the condition
        \begin{align*}
            128C_1 (1 + \delta) \exp(-C_2 (i+1)^{\beta}) \leq \varepsilon \leq 128C_1 (1 + \delta) \exp(-C_2 i^{\beta}) .
        \end{align*}
        By the definition of the entropy number and covering number, we know
        \begin{align*}
            & \log \mathcal{N}(\mathcal{B}_{\mathcal{H}_{\widetilde{k}}}, \|\cdot\|_{L_2(\mu)}, \varepsilon)\\
            \leq & \log \mathcal{N}(\mathcal{B}_{\mathcal{H}_{\widetilde{k}}}, \|\cdot\|_{L_2(\mu)}, 128C_1 (1 + \delta) \exp(-C_2 (i+1)^{\beta})) \\
            \leq & i \log (2) \\
            \leq & \log(2)\left(\frac{\log\left(\frac{128 C_1(1 + \delta)}{\varepsilon}\right)}{C_2}\right)^{1/\beta} \\
            \leq & \log(2)\left(\frac{\log\left(\frac{256 C_1}{\varepsilon}\right)}{C_2}\right)^{1/\beta} = O\left(C_{\mathrm{exp}}\log(1/\varepsilon)^{1/\beta}\right),
        \end{align*}
        where $C_3$ is a constant depends on $C_1$ and $C_2$.
    \end{itemize}
    Note that $n\leq N$. Hence, we can choose $\varepsilon$ for different eigendecay conditions and lead to the first claim as follows:
    \begin{itemize}
        \item For $\beta$-finite spectrum: we choose $\varepsilon = \Theta(n^{-1})$, and obtain the first claim with $\lambda = \Theta\left(\beta \log N + \log (N/\delta)\right)$ using a union bound over $\mathcal{B}_{\varepsilon}$ and $[N]$.
        \item For $\beta$-polynomial decay: we choose $\varepsilon = \Theta(n^{-\beta/(1+\beta)})$, and obtain the first claim with $\lambda = \Theta\left(C_{\mathrm{poly} }N^{1/(1+\beta)} + \log (N/\delta)\right)$ using a union bound over $\mathcal{B}_{\varepsilon}$ and $[N]$.
        \item For $\beta$-exponential decay: we choose $\varepsilon = \Theta(n^{-1})$, and obtain the first claim with $\lambda = \Theta\left(C_{\mathrm{exp}} (\log N)^{1/\beta} + \log (N/\delta)\right)$ using a union bound over $\mathcal{B}_{\varepsilon}$ and $[N]$.
    \end{itemize}

    For the second claim, note that, 
    \begin{align*}
        & \left\langle x, \left(n \widetilde{\Sigma}_n + \lambda T_k^{-1}\right) x \right\rangle = \left\langle T_k^{-1/2}x, T_k^{1/2}\left(n \widetilde{\Sigma}_n T_k^{1/2} + \lambda T_k^{-1}\right)T_k^{1/2} T_k^{-1/2} x \right\rangle,\\
        & \left\langle x, \left(\sum_{i=1}^n \phi_i \phi_i^\top + \lambda T_k^{-1}\right) x \right\rangle = \left\langle T_k^{-1/2} x, T_k^{1/2}\left(n \sum_{i=1}^n \phi_i \phi_i^\top + \lambda T_k^{-1}\right)T_k^{1/2} T_k^{-1/2} x \right\rangle.\\
    \end{align*}
    Note that, $\{T_k^{-1/2}x, x\in \mathcal{H}_k\}$ spans the $L_2(\mu)$, when the first claim holds, we have that, $\forall x^\prime \in L_2(\mu)$, $\forall n\in [N]$
    \begin{align*}
        \frac{1}{c_2}  \left\langle x^\prime, T_k^{-1/2}\left(n\Sigma  + \lambda T_k\right)^{-1} T_k^{-1/2} x^\prime \right\rangle_{L_2(\mu) }\leq \left\langle  x^\prime, T_k^{-1/2}\left(\sum_{i\in [n]}\phi_i \phi_i^\top + \lambda T_k^{-1} \right)^{-1} T_k^{-1/2} x^\prime\right\rangle_{L_2(\mu)},
    \end{align*}
    and
    \begin{align*}
        \left\langle x^\prime, T_k^{-1/2}\left(\sum_{i\in [n]}\phi_i \phi_i^\top + \lambda T_k^{-1} \right)^{-1} T_k^{-1/2} x^\prime\right\rangle_{L_2(\mu)} \leq \frac{1}{c_1}\left\langle x^\prime, \left(n \Sigma + \lambda T_k^{-1}\right)^{-1} T_k^{-1/2}x^\prime \right\rangle_{L_2(\mu)}.
    \end{align*}
    As $\forall x\in \mathcal{H}_k$, $T_k x \in L_2(\mu)$ and we can choose $x^\prime = T_k x$, which shows the second claim holds when the first claim holds.
\end{proof}

\begin{remark}
    Here we follow the idea of~\citet[][Lemma 45]{zanette2021cautiously} and present a less involved proof. However, it is also possible to use the Bernstein inequality for matrix martingale with intrinsic dimension~\citep[e.g.][]{minsker2017some} to prove the similar results. 
\end{remark}

\begin{lemma}[Simulation Lemma]
\label{lem:simulation}
    Suppose we have two MDP instances $\mathcal{M} = (\mathcal{S}, \mathcal{A}, P, r, d_0, \gamma)$ and $\mathcal{M}^\prime = (\mathcal{S}, \mathcal{A}, P^\prime, r + b, d_0, \gamma)$. Then for any policy $\pi$, we have that
    \begin{align*}
    V_{P^\prime, r + b}^\pi - V_{T, r}^\pi = & \frac{1}{1-\gamma}\mathbb{E}_{(s, a) \sim d_{P}^\pi}\left[b(s, a) + \gamma\left[\mathbb{E}_{P^\prime(s^\prime|s, a)}[V_{P^\prime, r + b}^{\pi}(s^\prime)] - \mathbb{E}_{P(s^\prime|s, a)}[V_{P^\prime, r + b}^\pi(s^\prime)]\right]\right],\\
    V_{P^\prime, r + b}^\pi - V_{T, r}^\pi = & \frac{1}{1-\gamma}\mathbb{E}_{(s, a) \sim d_{P^\prime}^\pi}\left[b(s, a) + \gamma\left[\mathbb{E}_{P^\prime(s^\prime|s, a)}[V_{T, r}^{\pi}(s^\prime)] - \mathbb{E}_{P(s^\prime|s, a)}[V_{T, r}^\pi(s^\prime)]\right]\right].
    \end{align*}
\end{lemma}
\begin{proof}
    See \citet[Lemma 20]{uehara2021representation}.
\end{proof}
\begin{lemma}[Potential Function Lemma for RKHS]
\label{lem:potential_function_RKHS}
    If $\alpha = \Omega(1)$, then for any distribution $\nu$ supported on the unit ball of $\mathcal{H}_{\widetilde{k}}$, we have that,
    \begin{itemize}
        \item For $\beta$-finite spectrum:
            \begin{align*}
                \log\mathrm{det}\left(\alpha \mathbb{E}_{\nu} [\phi \phi^\top] + I\right) = O\left(\beta \log\left(1 + \frac{\alpha}{\beta}\right)\right).
            \end{align*}
        \item For $\beta$-polynomial decay:
            \begin{align*}
                \log\mathrm{det}\left(\alpha \mathbb{E}_{\nu} [\phi \phi^\top] + I\right) = O\left(C_{\mathrm{poly}}\alpha^{1/(2\beta)}\log \alpha\right).
            \end{align*}
        \item For $\beta$-exponential decay:
            \begin{align*}
                \log\mathrm{det}\left(\alpha \mathbb{E}_{\nu} [\phi \phi^\top] + I\right) = O\left(C_{\mathrm{exp}}(\log \alpha)^{1 + 1/\beta} \right).
            \end{align*}
    \end{itemize}
    where operators are in the space of $L_2(\mu) \to L_2(\mu)$. 
    
    Meanwhile, when $\alpha = O(1)$, for any eigendecay conditions, we have that
    \begin{align*}
        \log\mathrm{det}\left(\alpha \mathbb{E}_{\nu} [\phi \phi^\top] + I\right) = O(1).
    \end{align*}
\end{lemma}
\begin{proof}
    We consider the optimization problem:
    \begin{align*}
        \sup_{\nu} \log \mathrm{det}\left(I + \alpha \mathbb{E}_{\phi \sim \nu} \left[\phi\phi^\top\right]\right).
    \end{align*}
    We first consider the optimality condition of $\nu$. Note that, $\log\mathrm{det}(X)$ is concave with respect to positive definite $X$ and $\mathbb{E}_{\phi \sim \nu} \left[\phi\phi^\top\right]$ is linear with respect to $\nu$. Direct computation shows that
    \begin{align*}
        \frac{d \log \mathrm{det}\left(I + \alpha \mathbb{E}_{\phi \sim \nu} \left[\phi\phi^\top\right]\right)}{d \nu(\phi^\prime)} = & \mathrm{Tr}\left(\alpha \left(I + \alpha \mathbb{E}_{\phi \sim \nu} \left[\phi\phi^\top\right]\right)^{-1} \phi\phi^\top\right) \\
        = & \left\langle \phi^\prime, \left(\alpha^{-1} I + \mathbb{E}_{\phi \sim \nu} \left[\phi\phi^\top\right]\right)^{-1}\phi^\prime\right\rangle_{L_2(\mu)}.
    \end{align*}
    Note that, $\left(\alpha^{-1} I + \mathbb{E}_{\phi \sim \nu} \left[\phi\phi^\top\right]\right)^{-1} \preceq \alpha I$. Hence, the inner product is well-defined.
    As $\nu$ is a probability measure over the $\mathcal{B}_{\mathcal{H}_{\widetilde{k}}}$, and if $c\geq 1$, 
    \begin{align*}
        & \left\langle c\phi^\prime, \left(\alpha^{-1} I + \mathbb{E}_{\phi \sim \nu} \left[\phi\phi^\top\right]\right)^{-1}c\phi^\prime\right\rangle_{L_2(\mu)} \\
        = & c^2 \left\langle \phi^\prime, \left(\alpha^{-1} I + \mathbb{E}_{\phi \sim \nu} \left[\phi\phi^\top\right]\right)^{-1}\phi^\prime\right\rangle_{L_2(\mu)}\\
        \geq & \left\langle \phi^\prime, \left(\alpha^{-1} I + \mathbb{E}_{\phi \sim \nu} \left[\phi\phi^\top\right]\right)^{-1}\phi^\prime\right\rangle_{L_2(\mu)}.
    \end{align*}
    Hence, we can focus on the $\nu$ supported on $\mathcal{S}_{\mathcal{H}_{\widetilde{k}}}$. Furthermore, with the optimality condition of the probability measure, we know the optimal $\nu$ should satisfy that $\forall\phi^\prime\in\mathrm{supp}(\nu)$, 
    \begin{align*}
        \left\langle \phi^\prime, \left(\alpha^{-1} I + \mathbb{E}_{\phi \sim \nu} \left[\phi\phi^\top\right]\right)^{-1}\phi^\prime\right\rangle_{L_2(\mu)} = C,
    \end{align*}
    and $\forall \phi^\prime\in\mathcal{S}_{\mathcal{H}_{\widetilde{k}}} $, we have
    \begin{align*}
        \left\langle \phi^\prime, \left(\alpha^{-1} I + \mathbb{E}_{\phi \sim \nu} \left[\phi\phi^\top\right]\right)^{-1}\phi^\prime\right\rangle_{L_2(\mu)} \leq C,
    \end{align*}
    where $C$ is some constant. 
    
    We first show that, $C \geq \frac{\alpha \mu_1(T_k)^2}{\alpha \mu_1^2(T_k) + 1}$, which can be shown by consider the following constraint optimization problem
    \begin{align*}
        \inf_{\nu} \left\langle \phi^\prime, \left(\frac{\lambda}{n} I + \mathbb{E}_{\phi \sim \nu} \left[\phi\phi^\top\right]\right)^{-1}\phi^\prime\right\rangle_{L_2(\mu)},
    \end{align*}
    where $\nu$ is from the space of probability measure supported on $\mathcal{S}_{\mathcal{H}_{\widetilde{k}}}$. As $\langle x, A^{-1} x\rangle$ is convex with respect to positive definite $A$ and $\mathbb{E}_{\phi\sim\nu} [\phi\phi^\top]$ is linear with respect to $\nu$, straightforward computation shows that
    \begin{align*}
        \frac{d \left\langle \phi^\prime, \left(\frac{\lambda}{n} I + \mathbb{E}_{\phi \sim \nu} \left[\phi\phi^\top\right]\right)^{-1}\phi^\prime\right\rangle_{L_2(\mu)}}{d\nu\left(\widetilde{\phi}\right)} = -\left\langle \phi^\prime, \left(\frac{\lambda}{n} I + \mathbb{E}_{\phi \sim \nu} \left[\phi\phi^\top\right]\right)^{-1}\widetilde{\phi}\right\rangle_{L_2(\mu)}^2.
    \end{align*}
    With the optimality condition, the optimal $\nu$ should satisfy that $\forall\phi\in\mathrm{supp}(\nu)$,
    \begin{align*}
        \left\langle \phi^\prime, \left(\alpha^{-1} I + \mathbb{E}_{\phi \sim \nu} \left[\phi\phi^\top\right]\right)^{-1}\widetilde{\phi}\right\rangle_{L_2(\mu)}^2 = C^\prime.
    \end{align*}
    and $\forall \widetilde{\phi} \in \mathcal{S}_{\mathcal{H}_{\widetilde{k}}}$, we have
    \begin{align*}
        \left\langle \phi^\prime, \left(\alpha^{-1} I + \mathbb{E}_{\phi \sim \nu} \left[\phi\phi^\top\right]\right)^{-1}\widetilde{\phi}\right\rangle_{L_2(\mu)}^2 \leq C^\prime,
    \end{align*}
    where $C^\prime$ is an absolute constant. With Cauchy-Schwartz inequality, we have
    \begin{align*}
        & \left\langle \phi^\prime, \left(\alpha^{-1} I + \mathbb{E}_{\phi \sim \nu} \left[\phi\phi^\top\right]\right)^{-1}\widetilde{\phi}\right\rangle_{L_2(\mu)}^2 \\
        \leq & \left\|T_k\left(\alpha^{-1} I + \mathbb{E}_{\phi \sim \nu} \left[\phi\phi^\top\right]\right)^{-1}\phi^\prime\right\|_{L_2(\mu)} \|T_k^{-1} \widetilde{\phi}\|_{L_2(\mu)}\\
        = & \left\|T_k\left(\alpha^{-1} I + \mathbb{E}_{\phi \sim \nu} \left[\phi\phi^\top\right]\right)^{-1}\phi^\prime\right\|_{L_2(\mu)},
    \end{align*}
    where the maximum only achieves when
    \begin{align*}
        \phi^\prime = c^\prime\left(\alpha^{-1} I + \mathbb{E}_{\phi\sim\nu}[\phi\phi^\top] \right)T_k^{-2} \widetilde{\phi}.
    \end{align*}
    where $c^\prime$ is an absolute constant to make sure $\|\widetilde{\phi}\|_{\mathcal{H}_{\widetilde{k}}} = 1$. Hence, the optimal $\nu$ is a point measure supported on $\widetilde{\phi}$, which further leads to
    \begin{align*}
        \left(\alpha^{-1} I + \widetilde{\phi}\widetilde{\phi}^\top \right)T_k^{-2} \widetilde{\phi} = \left(\alpha^{-1} T_k^{-2} + I \right) \widetilde{\phi},
    \end{align*}
    Take $\phi^\prime = \mu_1(T_k) e_1$, as $e_i$ is the eigenfunction of $T_k^{-2}$, we know $\nu$ should only support on $\mu_1(T_k) e_1$, and
    \begin{align*}
        \inf_{\nu} \left\langle \phi^\prime, \left(\alpha^{-1} I + \mathbb{E}_{\phi \sim \nu} \left[\phi\phi^\top\right]\right)^{-1}\phi^\prime\right\rangle_{L_2(\mu)} = \frac{\alpha \mu_1^2(T_k)}{\alpha\mu_1^2(T_k) + 1},
    \end{align*}
    which means $C \geq \frac{\alpha \mu_1(T_k)^2}{\alpha \mu_1^2(T_k) + 1}$.
    
    Now we consider the constraint optimization problem
    \begin{align*}
        \max_{\phi^\prime} \left\langle \phi^\prime, \left( \alpha^{-1} I + \mathbb{E}_{\phi\sim\nu} \left[\phi\phi^\top\right]\right)^{-1}\phi^\prime \right\rangle, \quad \mathrm{s.t.}\quad \left\|\phi^\prime\right\|_{\mathcal{H}_{\widetilde{k}}} \leq 1.
    \end{align*}
    With the method of Lagrange multiplier, we know that $C T_k^{-2} - \left( \alpha^{-1} I - \mathbb{E}_{\phi\sim\nu} \left[\phi\phi^\top\right]\right)^{-1} \succeq 0$, and for all $\phi$ in the support of $\nu$, we have $\left(C T_k^{-2} - \left( \alpha^{-1} I + \mathbb{E}_{\phi\sim\nu} \left[\phi\phi^\top\right]\right)^{-1}\right)\phi = 0$. Note that $\left\|\left(\alpha^{-1} I + \mathbb{E}_{\phi\sim\nu} \left[\phi\phi^\top\right]\right)^{-1}\right\|_{\mathrm{op}} \leq \alpha$. With Weyl's inequality, we know that, 
    \begin{align*}
        \mu_i\left(C T_k^{-2} - \left( \alpha^{-1} I - \mathbb{E}_{\phi\sim\nu} \left[\phi\phi^\top\right]\right)^{-1}\right) \geq C \mu_i(T_k)^{-2} - \alpha \geq \alpha\left(\frac{\mu_1^2(T_k)\mu_i^{-2}(T_k)}{\alpha \mu_1^2(T_k) + 1} - 1\right),
    \end{align*}
    which means the support of $\nu$ is at most $i_0$ dimension, where $i_0$ is the largest integer that $\mu_1^2 (T_k) \mu_i^{-2}(T_k) \leq \alpha \mu_1^2(T_k) + 1$.
    
    When $\alpha = O(1)$, with Assumption~\ref{assump:trace}, we know $\mu_i(T_k) \leq 1$ and $i_0 = O(1)$. Combined with Jensen's inequality, we finish the proof of the second claim.
    
    We then consider the case when $\alpha = \Omega(1)$ under different eigendecay conditions:
    \begin{itemize}
        \item $\beta$-finite spectrum: we know $i_0 \leq \beta$. As $\|\phi\|_{L_2(\mu)} \leq \|\phi\|_{\mathcal{H}_{\widetilde{k}}} = 1$, we have $\left\|\mathbb{E}_{\phi\sim\nu}\left[\phi\phi^\top\right]\right\|_{\mathrm{op}} \leq 1$. With Jensen's inequality, we have
        \begin{align*}
            \log \mathrm{det}\left(I + \alpha\mathbb{E}_{\phi\sim\nu}\left[\phi\phi^\top\right]\right) = O\left(\beta \log\left(1 + \frac{\alpha}{\beta}\right)\right).
        \end{align*}
        \item $\beta$-polynomial decay: we know $i_0 = O\left(C_{\mathrm{poly}}\alpha^{1/(2\beta)}\right)$ dimension. As $\|\phi\|_{L_2(\mu)} \leq \|\phi\|_{\mathcal{H}_{\widetilde{k}}} = 1$, we have $\left\|\mathbb{E}_{\phi\sim\nu}\left[\phi\phi^\top\right]\right\|_{\mathrm{op}} \leq 1$. With Jensen's inequality, we have
        \begin{align*}
            \log \mathrm{det}\left(I + \alpha\mathbb{E}_{\phi\sim\nu}\left[\phi\phi^\top\right]\right) = O\left(C_{\mathrm{poly}}\alpha^{1/(2\beta)} \log \alpha\right).
        \end{align*}
        \item $\beta$-exponential decay: we know the support of $\nu$ is at most $O\left(C_{\mathrm{exp}}(\log \alpha)^{1/\beta}\right)$ dimension. 
        As $\|\phi\|_{L_2(\mu)} \leq \|\phi\|_{\mathcal{H}_{\widetilde{k}}} = 1$, we have $\left\|\mathbb{E}_{\phi\sim\nu}\left[\phi\phi^\top\right]\right\|_{\mathrm{op}} \leq 1$, With Jensen's inequality, we have
        \begin{align*}
            \log \mathrm{det}\left(I + \alpha\mathbb{E}_{\phi\sim\nu}\left[\phi\phi^\top\right]\right) = O\left(C_{\mathrm{exp}}(\log \alpha)^{1 + 1/\beta} \right).
        \end{align*}
    \end{itemize}
    Hence, we obtain the desired results.
\end{proof}
\revise{
\section{Additional Experiment Result}
\label{sec:appendix-new-results}

\subsection{Training with 1M Steps}

This section provides the learning curves with 1M training steps compared to SAC in four Mujoco control problems. 
We only tune the \emph{feature-updates-per-step} parameter from $\{1,3,5\}$ and report the best result to save computations and running time. 
The results clearly demonstrate that \algabb also achieves significantly better performance in the long run.

% \vspace{2mm}
\begin{figure*}[thb]
%\centering
%\subfigure[cheetah_run]
{\includegraphics[width=3.3cm]{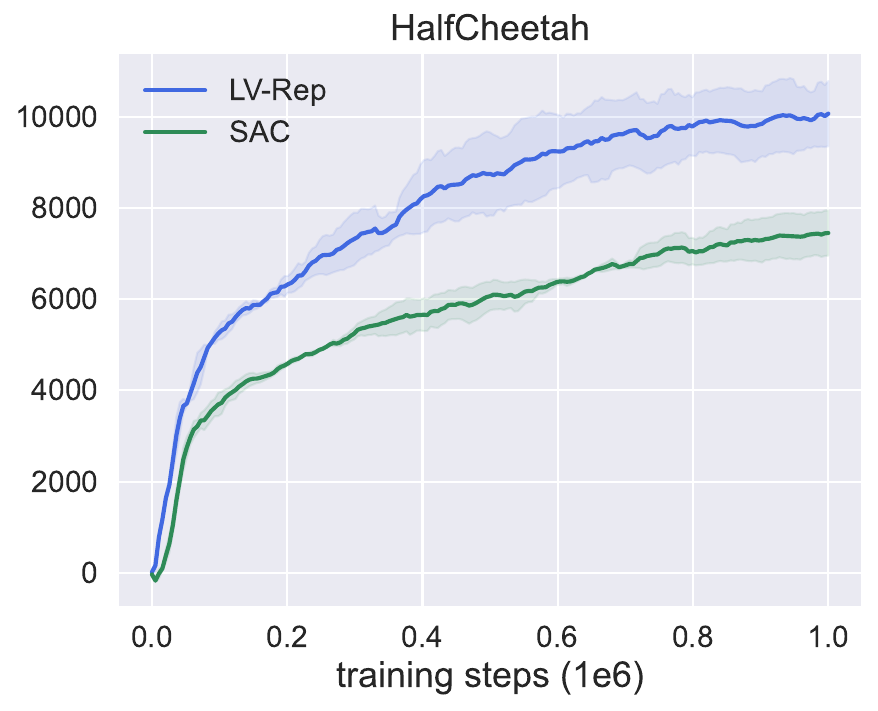}
\label{fig:gym-cheetah}
}
%\subfigure [walker_run]
{\includegraphics[width=3.3cm]{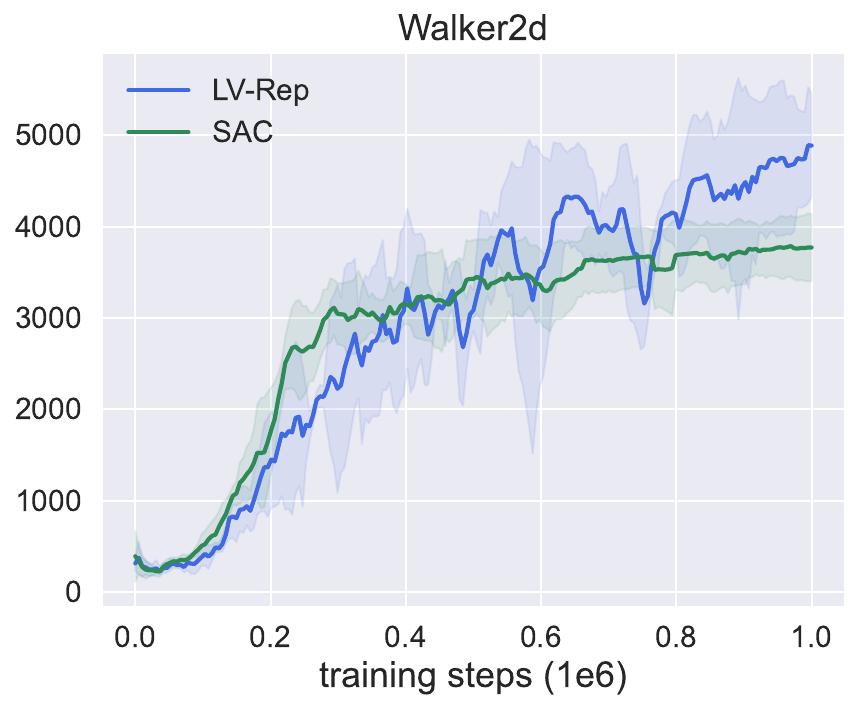}
\label{fig:gym-walker}
}
%\subfigure [humanoid_run]
{\includegraphics[width=3.3cm]{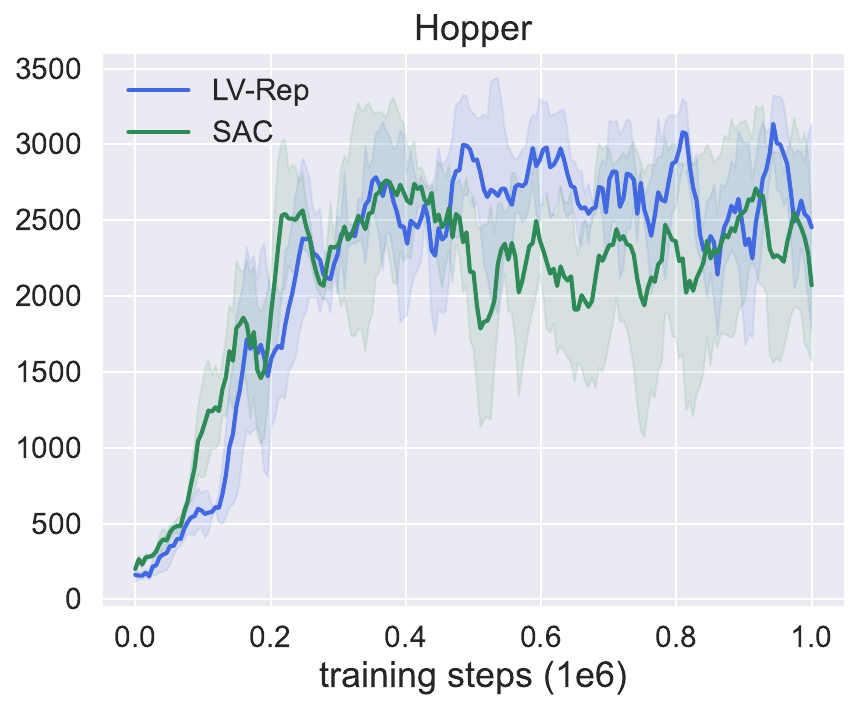}
\label{fig:gym-hopper}
}
%\subfigure [hopper_hop]
{\includegraphics[width=3.3cm]{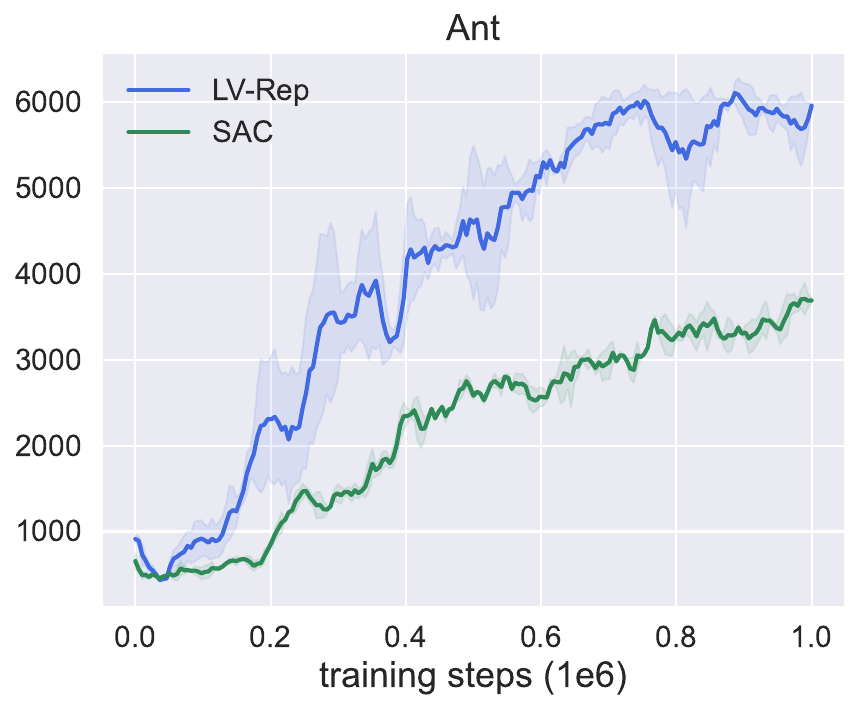}
\label{fig:gym-ant}
}
%\vspace{-2mm}
\caption{
\revise{We show the learning curves in Mujoco control compared to SAC. 
The $x$-axis shows the training iterations and $y$-axis shows the performance. 
All plots are averaged over 4 random seeds. The shaded area shows the standard error.}
%We only compare to SAC as it has the best overall performance in all baseline methods.
}
%\vspace{-3mm}
\label{fig:mujoco-1m}
\end{figure*}

\subsection{Ablation Study on Latent Representation Size}

In this section, we provide an ablation study on the latent representation dimension to show this parameter affects the performance of \algabb.  
In all our experiments the latent feature dimension is set to 256. 
We compare to latent feature dimension 64 and 128 in HalfCheetah. 
The results are reported in Figure~\ref{fig:feature-ablation}.

%\vspace{-2mm}
\begin{figure*}[thb]
\centering
%\subfigure[cheetah_run]
{\includegraphics[width=5cm]{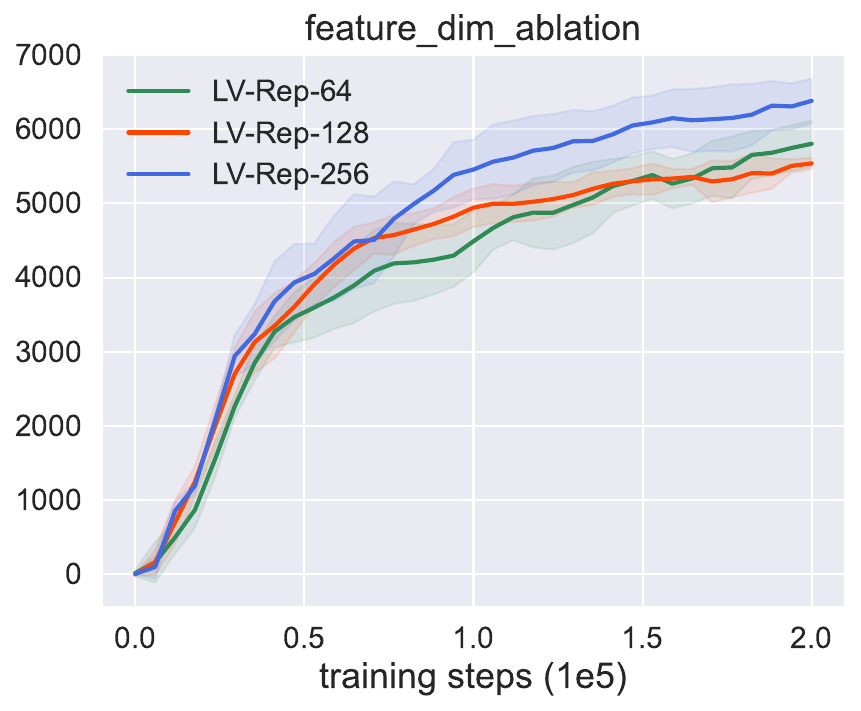}
%\label{fig:gym-cheetah}
}
\caption{
\revise{Ablation study on the dimension of latent representations.}
%The $x$-axis shows the training iterations and $y$-axis shows the performance. 
%All plots are averaged over 4 random seeds. The shaded area shows the standard error.  
%We only compare to SAC as it has the best overall performance in all baseline methods.
}
%\vspace{-3mm}
\label{fig:feature-ablation}
\end{figure*}}

\end{document}